\documentclass{article}

    \PassOptionsToPackage{authoryear,round,sort}{natbib}
\usepackage[abbrvbib,nohyperref]{jmlr2e}

\usepackage{lastpage}

\ShortHeadings{Bayesian Multi-Domain Causal Representation Learning}%
{Garg, Stettler, Schein, von K\"ugelgen}
\firstpageno{1}

\usepackage[utf8]{inputenc} %
\usepackage[T1]{fontenc}    %
\usepackage{url}            %
\usepackage{booktabs}       %
\usepackage{amsfonts}       %
\usepackage{nicefrac}       %
\usepackage{microtype}      %
\usepackage[dvipsnames]{xcolor}         %

\newcommand{\jvk}[1]{\textcolor{Red}{[JvK: #1]}}
\newcommand{\ag}[1]{\textcolor{OliveGreen}{[AG: #1]}}

\renewcommand{\jvk}[1]{}
\renewcommand{\ag}[1]{}

\input{neurips_header.tex}

\title{Discrete Causal Representations from Heterogeneous~Domains:\\ A Bayesian Approach with Social Survey Applications}

\author{%
\name Ankur Garg \email garga@uchicago.edu \\ \addr Department of Statistics, University of Chicago
\AND 
\name Michael Stettler \email michael.stettler@cin.uni-tuebgingen.de \\ \addr University of T\"ubingen
\AND
\name Aaron Schein \email schein@uchicago.edu \\ \addr Department of Statistics \& Data Science Institute, University of Chicago
\AND 
Julius von Kügelgen \email julius.vonkuegelgen@stat.math.ethz.ch \\ \addr Seminar for Statistics, ETH Zürich
}
\editor{}

\begin{document}
\doparttoc
\faketableofcontents%

\maketitle

\begin{abstract} %
\looseness-1
Causal representation learning aims to infer the high-level latent causal concepts that give rise to observed low-level measurements.
This is particularly relevant for heterogeneous data from different 
environments or domains since distribution shifts often arise through sparse, localized changes in some of the underlying causal mechanisms, while other parts of the generative process remain unchanged.
Whereas identifiability of causal representations has been studied extensively, practical uncertainty-aware methods and real-world use cases remain less explored. 
In this work, we propose a Bayesian approach to learning causal representations from multi-environment data, focusing on the case of discrete causal concepts and unknown multi-node soft interventions. 
To this end, we translate causal assumptions and interpretability desiderata into suitable priors and parametric choices within a hierarchical model.
We then devise an inference scheme based on sequential Monte Carlo sampling to approximate the resulting multimodal posterior.
We showcase our approach through case studies on social survey data, where latent causal concepts correspond to cultural values or political opinions, measurements to survey responses, and environments to different countries or states.
Our model infers meaningful high-level concepts and plausible causal relations among them, demonstrating its utility for learning causal representations of complex real-world data.
\end{abstract}

\vspace{1em}

\everypar{\looseness=-1}
\linepenalty=1000
\numberwithin{equation}{section}
\numberwithin{theorem}{section}

\section{Introduction}
\label{sec:introduction}
Questions in the natural and social sciences are often causal in nature: 
\textit{What will be the effects of knocking out certain genes or implementing a new policy? What are the main causes of a given disease? What drives political polarization?}
Causal inference provides a principled framework for answering such  questions from data~\citep{robins1986new,rubin1974estimating,neyman1923applications,wright1934method}. 
Causal models can be viewed as a generalization of statistical models and as a coarsening of mechanistic, differential-equation-based models~\citep{peters2017elements}. 
They describe not only the unperturbed, observational state of a system but %
a whole family of related distributions that capture how the involved variables would respond to manipulations or interventions~\citep{Pearl2009}. 
To go beyond associations and support reasoning about interventions, causal models rely on additional, typically graphical structure and express the data-generating process as a collection of independent modules, which determine how each variable depends on its direct causes.

Sometimes, the qualitative causal relations among variables---usually in the form of a directed, acyclic causal graph---are known or can be derived from the temporal ordering of events (since causes generally precede their effects).
\textit{Causal reasoning} then aims to quantify the strength of causal relations such as the effect of a medical treatment or policy intervention on an outcome of interest, often in the presence of unobserved confounding.
Methods for causal reasoning, mostly from observational (i.e., passively collected) data, have been extensively studied and successfully applied in fields such as epidemiology, econometrics, or social science~\citep{hernan2020causal,imbens2015causal,angrist2009mostly,morgan2014counterfactuals}.

\looseness-1 When the qualitative causal relations are unknown, they need to be inferred from data. This \textit{causal discovery} task is notoriously difficult if only observational data is available and generally only possible up to an equivalence class of graphs~\citep{spirtes2001causation}.
To overcome this challenge, a common approach is to leverage heterogeneous data from multiple experiments, environments, or domains~\citep{peters2016causal,bareinboim2016causal,jaber2020causal,tian2001causal,Brouillard2020,perry2022causal,hauser2012characterization,hauser2015jointly,yang2018characterizing,wu2025bayesian}.

A central assumption for %
causal reasoning and causal discovery is that the causal variables of interest are directly observed. 
However, %
we often only observe high-dimensional proxy measurements of the underlying causal variables.
Examples include computer vision, where we observe pixels %
but care about the underlying entities and their relations; 
single-cell biology, where gene expression or microscopy images are used to infer latent gene programs~\citep{uhler2025causal};
and, as explored in this work, political and social science, where questionnaire responses provide indirect evidence about respondents' beliefs on the involved topics.
Due to the mismatch between what is observed and the variables of interest, directly applying standard techniques for classical causal inference is challenging if not impossible in such settings.

\textit{Causal representation learning} (CRL) aims to bridge this gap by extending causal modeling to settings in which the causal variables of interest are latent and only indirectly observed~\citep{scholkopf2021toward,moran2026towards,varici2026causal}. 
In CRL, the observed data is modeled as arising from a measurement process whose inputs are high-level latent variables that are causally related~\citep{silva2006learning}.
This allows shifting causal assumptions and modeling from the level of the raw data to an embedding or representation space inferred by machine learning models, thereby broadening the scope of causal inference to more complex settings.

Motivated by impossibility results in causal discovery~\citep{eberhardt2005number} and representation learning~\citep{hyvarinen1999nonlinear,locatello2019challenging}, prior studies of CRL have mostly been theoretical in nature and focused on the question of identifiability: under what conditions and up to what ambiguities can the generative process be recovered at the population level (i.e., assuming access to infinite data).
Following breakthroughs in nonlinear independent component analysis~\citep{hyvarinen2016unsupervised,hyvarinen2017nonlinear,hyvarinen2018nonlinear,khemakhem2020variational}, there has been a flurry of identifiability results for learning causal representations from observational data~\citep{adams2021identification,xu2024sparsity,xie2020generalized,xie2022identification,cai2019triad}, time series~\citep{lippe2022citris,ahuja2021properties,lippe2022icitris,yao2021learning,yao2022temporally}, or simultaneously observed views or modalities~\citep{von2021self,brehmer2022weakly,daunhawer2023identifiability,yao2024multi}.

\looseness-1 In the present work, we consider the multi-domain setting of learning from interventional datasets collected from different environments or experimental conditions, which has also been studied extensively from an identifiability perspective~\citep{ahuja2023interventional,buchholz2023learning,wendong2023causal,von2023nonparametric,squires2023linear,zhang2023identifiability,varici2025score,xi2022indeterminacy}.
However, much less work has been dedicated to robust probabilistic inference of causal representations from finite real-world data. Most existing results assume idealized conditions such as perfect single-node interventions, infinite data, and synthetic benchmarks. 
In contrast, imperfect multi-node interventions~\citep{ahuja2024multi,bing2024identifying,varici2024linear}, partial identifiability~\citep{lachapelle2024nonparametric,zhang2024causal}, and limited sample size are the norm in practice. 

\looseness-1
These challenges motivate probabilistic approaches that can quantify uncertainty in the inferred structures and parameters~\citep{wu2026multi}.
Targeting a posterior rather than a point estimate, allows for incorporating different types of uncertainty: \text{epistemic} uncertainty due to a lack of data (finite sample issues as encountered in practice) and \text{aleatoric} uncertainty due to inherent ambiguities in the generative process (partial identifiability due to a lack of sufficiently informative interventions or domains). 

\looseness-1 The Bayesian modeling literature has developed an extensive toolkit for inference in these types of structured latent variable models~\citep{blei2014build,gelman2013bayesian,airoldi2014handbook}. Originally developed for a variety of classical latent variable modeling tasks, these methods provide a natural fit for CRL: structured priors offer a flexible and modular way to impose suitable assumptions on both the latent causal variables and the measurement model. 

\subsection{Structure and Overview}
\label{sec:contributions}
In this work, we develop a Bayesian approach for the unsupervised learning of discrete causal representations from heterogeneous environments and empirically validate it on social survey data. 

In~\cref{sec:problem_setting}, we begin by describing the general problem setting.
We formulate the assumed data generating process as a causal latent variable model, in which the unknown measurement process that produces the observed data from the underlying latent variables remains invariant and differences across domains arise through unknown, soft, multi-node interventions in the latent causal model~(\cref{fig:multi_env}). 
To facilitate interpretability, we focus on the case of discrete latent causal variables, which gives the generative process an interpretation as a finite mixture model with domain-specific mixture weights~(\cref{fig:ex_multi_domain_data}). 

In~\cref{sec:model}, we then propose a hierarchical Bayesian model~(\cref{fig:graphical_model}) that follows the structure of the assumed generative process. 
For the latent causal model, this involves specifying suitable priors over the causal mechanisms and unknown interventions that take into account fundamental assumptions from the causal inference literature such as modularity and sparse shifts~(\cref{tab:mapping}). 
In particular, we refine the sparse mechanism hypothesis~\citep{scholkopf2021toward} for finite-samples with our definition of \(\varepsilon\)-intervention detectability~(\cref{assump:epsilon_intervention_detectability}). We then operationalize this assumption by parameterizing interventions as additive shifts in logit space and enforcing detectability through a novel KL-thresholded shift prior. This prior uses an approximation to the KL divergence between the base and intervened mechanisms~(\cref{fig:kl_normal}) to rule out shifts that induce only negligible distributional changes. %
For the measurement model, we adopt a Gaussian mixture model, in which each latent is treated as an ordinal factor whose level contributes additively to the mean. 
This parametrizes how strongly each latent causal concept influences each observed variable, thereby aiding with interpretability. \looseness-1 %

In~\cref{sec:inference}, we present our approach to approximate posterior inference over the various model components. Due to the highly multimodal structure of the posterior, we employ sequential Monte Carlo sampling (SMCS), which embeds likelihood-tempered Markov chain Monte Carlo (MCMC) updates within a particle-based framework. Most model components admit closed-form conditional updates, and for the non-conjugate parameters we leverage data augmentation to obtain efficient Gibbs steps throughout.

 In~\cref{sec:synthetic_experiments}, we use synthetic experiments and ablations to interpret and evaluate our model. We analyze the posterior landscape, showing that permutations of the latent structure induce distinct asymmetric modes~(\cref{fig:bayes_factors}), and identifying the model components that create and distinguish these modes~(\cref{fig:bayes_factors_decomp}). We also evaluate our inference procedure by comparing against alternative inference schemes and model ablations~(\cref{tab:inference_comparison}).

In~\cref{sec:case_study_wvs,sec:case_study_us}, we then demonstrate our method through case studies on political survey data, where we further develop a procedure for ascribing semantic meaning to the inferred latent variables and thus interpreting the learned causal representations.
For data from the World Values Survey between 2017 and 2022~\citep{haerpfer2022world} in~\cref{sec:case_study_wvs}, where domains correspond to countries, we recover latent factors corresponding to economic hardship, demographics and cultural conservatism; these align with the widely accepted Inglehart–Welzel map~\citep{inglehart2005world} of cross-cultural variation. Our model suggests a causal link \{economic hardship, demographics\} $\to$ cultural conservatism, and we further investigate country-specific patterns of interventions and causal effects.
For data from the US in~\cref{sec:case_study_us}, where domains correspond to urbanicity and census regions, we find a largely one-dimensional latent structure reflecting well-known partisan divisions. We further consider a large language model (LLM)-generated ``semi-synthetic'' version of the same dataset, where we introduce controlled interventions and show that our approach is able to recover the structure of partisanship and four issue opinions (Trade, Regulation, Immigration, and Poverty).

In~\cref{sec:discussion}, we discuss our approach and findings in the context of related work and highlight limitations and avenues for future research. 
\Cref{sec:conclusion} concludes with a brief summary and outlook. 

\subsection{Contributions}
Our main contributions can be summarized as follows.
\begin{enumerate}[leftmargin=*]
    \item \textbf{Conceptual:} We formulate the causal generative process for interventional multi-domain data~(\cref{sec:problem_setting}) as a hierarchical Bayesian model~(\cref{fig:graphical_model,sec:model}). This 
    involves translating assumptions from the causality literature into suitable modeling choices and priors~(\cref{tab:mapping}), including a refinement of sparse mechanism shifts to consider detectability from finite samples (\cref{assump:epsilon_intervention_detectability}).
    \item \textbf{Technical:} We develop a novel prior based on KL-divergence truncation~(\cref{sec:shift_priors}), as well as a Sequential Monte Carlo sampling–based inference procedure~(\cref{sec:inference}) which leverages conjugate updates and likelihood tempering to efficiently explore a highly multimodal posterior over latent causal structures~(\cref{tab:inference_comparison,fig:bayes_factors,fig:bayes_factors_decomp}).
    \item \textbf{Empirical:} We showcase our approach through case studies on the World Values Survey~(\cref{sec:case_study_wvs}), as well as on real and LLM--generated US political survey data~(\cref{sec:case_study_us}), highlighting how to interpret and derive causal insights from %
    the estimated posterior over different model components.
\end{enumerate}
The present work thus constitutes one of the first successful demonstrations of applying a complete unsupervised CRL pipeline---from model formulation through estimation to interpretation of the inferred latent variables and structures---to complex real-world data.

\begin{table}[t]
\centering
\caption{Overview of the high-level assumptions and desiderata on the causal generative process and their implementation within our Bayesian hierarchical latent variable model.} 
\label{tab:mapping}
\setlength{\tabcolsep}{10pt} %
\renewcommand{\arraystretch}{2} %
\resizebox{\textwidth}{!}{
\begin{tabular}{p{0.3\linewidth} p{0.4\linewidth} p{0.21\linewidth}}
\toprule
\textbf{Assumption/Desideratum} & \textbf{Implementation}  & \textbf{Reference} \\
\midrule
discrete causal concepts $\Zb$
 &  each $p^e(\Xb)$ is a finite mixture model
 & \cref{ass:cardinality}, \cref{sec:discrete}
 \\
 invariant measurement process 
& mixture components $p(\Xb \mid \Zb=\zb)$ shared across domains 
& \cref{ass:shared_obs_model}, \cref{sec:model_measurement}
\\
modularity/independent causal mechanisms
& parameters $\Thetab$ of unintervened mechanisms $p(Z_\ell \mid \PA_\ell)$ shared across domains 
&
\cref{assumption:shared_mechs}, \cref{sec:model_base_mechanisms}
\\
sparse mechanism shifts 
& sparsity-inducing Beta-Bernoulli prior for intervention target indicators $\Ical^e$
& \cref{sec:model_interventions}
\\
post-intervention mechanisms are sufficiently distinct
& prior based on KL-divergence truncation for domain-specific shifts $\Deltab^e$
& \cref{assump:epsilon_intervention_detectability}, \cref{sec:model_intervened_mechanisms}
\\
{ease of interpretability\newline of causal concepts}
&
measurement model as additive effects for simple latent-question pairing
& \cref{sec:model_measurement} \\
\bottomrule
\end{tabular}
}
\end{table}

\subsection{Notation}
We use the terms domain and environment interchangeably.
Scalar quantities are written in non-bold (e.g., $c$ or $L$), column vectors in bold lower case (e.g., $\xb$) and matrices in bold uppercase (e.g., $\Mb$ with $\mb_i$ representing the $i$\textsuperscript{th} column). We use uppercase for random variables (e.g., $\Zb$ and~$Z_1$) and lowercase for their realizations (e.g., $\zb$ and~$z_1$).
For $n\in\NN_{\geq 1}$, we use the shorthand $[n]:=\{1, ..., n\}$. %
Superscripts index domains (e.g., $\Ical^e$ with environment index $e$) and data points (e.g., $\xb^{e,j}$ with sample index~$j$), subscripts index dimensions (e.g., $Z_\ell$ with latent variable index~$\ell$), and parentheses index parent configurations (e.g.,~$ \thetab_\ell(\pa_\ell)$). 
We denote the vector of all ones by $\bm1=(1,1,...,1)$ and the identity matrix by $\Ib$.
In a slight abuse of notation, we use $p$ to refer to both probability distributions and their associated probability mass or density functions.

\section{Problem Setting: CRL from Heterogeneous Domains}
\label{sec:problem_setting}
\looseness-1 
We begin by formalizing the problem of learning causal representations from high-dimensional measurements across multiple related domains.
In this section, we focus on a general nonparametric description of the setting. Specific modeling choices and priors are deferred to~\cref{sec:model}.
Specifically, we formulate the data-generating process as a latent variable model, %
in which the latent variables are related through an unknown causal model~(\cref{sec:causal_model}). Differences among domains are then modeled through unknown soft, or imperfect, interventions in the shared latent causal model, subject to the same invariant measurement process~(\cref{sec:multi_domain}). We focus on the case of discrete latent variables;  the data generating process then takes the form of a decomposable finite mixture model with fixed mixture distributions but variable mixture weights across domains,  resulting from interventions on subsets of latents~(\cref{sec:discrete}).

\subsection{Causal Latent Variable Model}
We formalize the assumed data generating process for the observed random vector $\Xb$ as a \textit{latent variable model} with latent variables $\Zb$.
A latent variable model is a generative model that describes a hierarchical data-generating process.
Rather than specifying the distribution $p(\Xb)$ directly, we specify a (marginal) distribution $p(\Zb)$ over the latent variables and a conditional  $p(\Xb\mid\Zb)$ over the observed variables given the latent variables, which we will refer to as a \textit{measurement model}.
Together, this defines the joint distribution
   $p(\Zb,\Xb)=p(\Zb)p(\Xb\mid\Zb)$
which in turn induces the marginal data distribution $p(\Xb)$ via $p(\Xb)=\int p(\Xb\mid\Zb) p(\Zb) \mathrm{d}\Zb\,.$

\label{sec:causal_model}
Rather than choosing a single fixed latent distribution $p(\Zb)$, we assume that the latent variables $\Zb=(Z_1, ..., Z_L)$ are related through a causal model. 
The causal relations among $\Zb$ are encoded by a \textit{causal graph} $G$, %
i.e., a directed acyclic graph (DAG) with vertices $\Zb$ and directed edges $Z_j\to Z_\ell$ indicating that $Z_j$ is a direct cause or causal parent of $Z_\ell$. The set of all such causal parents is denoted $\PA_\ell \subseteq\Zb\setminus\{Z_\ell \}$.
The joint distribution $p(\Zb)$ of a causal model with causal graph $G$ is Markovian w.r.t.\ $G$, i.e., it obeys the following factorization,
\begin{equation}
\label{eq:causal_factorisation}
p(\Zb)=p(Z_1,\dots,Z_L) = \prod_{\ell=1}^L  p\left(Z_\ell \mid \PA_\ell\right)\,.
\end{equation}
where the conditionals \(p\left(Z_\ell \mid \PA_\ell\right)\) in~\eqref{eq:causal_factorisation} are called the %
\textit{causal mechanisms}.

What makes the model causal is that it also describes a family of \textit{interventional} distributions, i.e., distributions that arise from replacing a subset of the causal mechanisms with new mechanisms $\tilde p(Z_\ell\mid\PA_\ell)$.
The principle of independent causal mechanisms~\citep[ICM;][]{peters2017elements} states that causal mechanisms are modular in the sense that changing some of them does not affect the other mechanisms. 
Hence, the distribution for an intervention that changes the mechanisms of $Z_\ell$ for all $\ell$ in some subset $\Ical\subseteq [L]$ is given by
\begin{equation}
\label{eq:interventional_distribution}
\tilde p_\Ical(\Zb)=\prod_{\ell\in\Ical}\tilde p\left(Z_\ell\mid\PA_\ell\right) \prod_{\ell\not\in\Ical}p\left(Z_\ell\mid\PA_\ell\right)\,. 
\end{equation}

\subsection{Interventional Multi-Domain Data}
\label{sec:multi_domain}
Causal representation learning from i.i.d.\ data is known to be impossible without strong additional assumptions and richer, heterogeneous data %
is therefore typically required. In the present work, we consider data from multiple environments or domains, which has been shown to provide useful causal learning signals when combined with suitable causal assumptions~\citep{peters2016causal,heinze2018invariant,Rojas-Carulla2018,krueger2021out,eastwood2022probable,perry2022causal}. 
In the multi-domain setting, we observe datasets $\Dcal=\{\Dcal^e\}_{e\in\Ecal}$ where $e\in\Ecal$ indexes  different domains or environments.
Each such dataset $\Dcal^e$ is assumed to be an i.i.d.\ sample of observations of~$\Xb$ of size $m_e$ from a domain-specific distribution~$p^e(\Xb)$, i.e.,
\begin{equation}
\label{eq:multi_env_data}
    \Dcal^e=\{\xb^{e,1},...,\xb^{e,m_e}\}
    \overset{\mathrm{i.i.d.}}{\sim} 
    p^e(\Xb)
    \qquad
    \mathrm{for}
    \qquad
    e \in \Ecal%
\end{equation}

\begin{figure}
    \newcommand{\xshift}{3.5em}
    \newcommand{\yshift}{2.75em}
    \newcommand{\zshift}{0.75em}
    \centering
    \begin{tikzpicture}
        \centering
        \node (V1) [latent] {$Z_1$};
        \node (V2) [latent, xshift=\xshift] {$Z_2$};
        \node (V3) [latent, yshift=-\yshift, draw=\ivcolor,thick] {\textcolor{\ivcolor}{$Z_3$}};
        \node (V4) [latent, xshift=\xshift, yshift=-\yshift] {$Z_4$};
        \edge[-stealth, thick]{V1,V2,V3}{V4};
        \edge[-stealth, color=\ivcolor,thick]{V1}{V3};
        \plate[inner sep=0.1em, yshift=0.2em, dashed] {plate1}{(V1) (V2) (V3) (V4)}{};

        \node (Y1) [latent,yshift=-2*\yshift] {$Z_1$};
        \node (Y2) [latent, xshift=\xshift, yshift=-2*\yshift] {$Z_2$};
        \node (Y3) [latent, yshift=-3*\yshift] {$Z_3$};
        \node (Y4) [latent, xshift=\xshift, yshift=-3*\yshift, draw=\ivcolor,thick] {\textcolor{\ivcolor}{$Z_4$}};
        \edge[-stealth, color=\ivcolor,thick]{Y1,Y2,Y3}{Y4};
        \edge[-stealth, thick]{Y1}{Y3};
        \plate[inner sep=0.1em, yshift=0.2em, dashed] {plate2}{(Y1) (Y2) (Y3) (Y4)}{};

        \node (W1) [latent, yshift=-4*\yshift, draw=\ivcolor,thick] {\textcolor{\ivcolor}{$Z_1$}};
        \node (W2) [latent, xshift=\xshift, yshift=-4*\yshift, draw=\ivcolor,thick] {\textcolor{\ivcolor}{$Z_2$}};
        \node (W3) [latent, yshift=-5*\yshift] {$Z_3$};
        \node (W4) [latent, xshift=\xshift, yshift=-5*\yshift] {$Z_4$};
        \edge[-stealth, thick]{W1}{W3};
        \edge[-stealth, thick]{W1,W2,W3}{W4};
        \plate[inner sep=0.1em, yshift=0.2em, dashed] {plate3}{(W1) (W2) (W3) (W4)}{};

        \node (e1) [const, xshift=-1*\xshift, yshift=-0.5*\yshift]{$e=1$:};
        \node (e2) [const, xshift=-1*\xshift, yshift=-2.5*\yshift]{$e=2$:};
        \node (e3) [const, xshift=-1*\xshift, yshift=-4.5*\yshift]{$e=3$:};
        \node (p1) [const, xshift=2.5*\xshift, yshift=-0.5*\yshift]{$\zb\sim p^1(\Zb)$};
        \node (p2) [const, xshift=2.5*\xshift, yshift=-2.5*\yshift]{$\zb\sim p^2(\Zb)$};
        \node (p3) [const, xshift=2.5*\xshift, yshift=-4.5*\yshift]{$\zb\sim p^3(\Zb)$};

        \coordinate (A) at (4*\xshift,-0.75*\yshift);
        \coordinate (B) at (5.5*\xshift,-0.*\yshift);
        \coordinate (C) at (5.5*\xshift,-5*\yshift);
        \coordinate (D) at (4*\xshift,-4.25*\yshift);
        \draw[fill=\decodercolor!20] (A) -- (B) -- (C) -- (D) -- cycle;
        \coordinate (input_1) at (4*\xshift,-1.25*\yshift);
        \coordinate (input_i) at (4*\xshift,-2.5*\yshift);
        \coordinate (input_n) at (4*\xshift,-3.75*\yshift);
        \edge[-stealth, thick, color=\decodercolor]{p1.south east}{input_1};
        \edge[-stealth, thick,color=\decodercolor]{p2.east}{input_i};
        \edge[-stealth, thick,color=\decodercolor]{p3.north east}{input_n};
        \node[const,xshift=4.75*\xshift, yshift=-2.5*\yshift] {$p(\Xb~|~\Zb)$};
        \coordinate (output_1) at (5.5*\xshift,-0.5*\yshift);
        \coordinate (output_i) at (5.5*\xshift,-2.5*\yshift);
        \coordinate (output_n) at (5.5*\xshift,-4.5*\yshift);
        \node (D_1) [const, xshift=7*\xshift, yshift=-0.5*\yshift] {\, $\Dcal^1\sim p^1(\Xb)
        $};
        \node (D_i) [const, xshift=7*\xshift, yshift=-2.5*\yshift] {\, $\Dcal^2\sim p^2(\Xb)
        $};
        \node (D_n) [const, xshift=7*\xshift, yshift=-4.5*\yshift] {\, $\Dcal^3\sim  p^3(\Xb)
        $};
        \edge[-stealth, thick,color=\decodercolor]{output_1}{D_1};
        \edge[-stealth, thick,color=\decodercolor]{output_i}{D_i};
        \edge[-stealth, thick,color=\decodercolor]{output_n}{D_n};
    \end{tikzpicture}
    \caption{\textbf{Multi-domain CRL with soft, multi-node interventions.} Illustration of our  multi-domain setup with $L=4$ causal variables $\Zb=(Z_1,Z_2,Z_3,Z_4)$ and $|\Ecal|=3$ domains.
    The intervention targets $\Ical^e$ are given by $\Ical^1=\{3\}$, $\Ical^2=\{4\}$, and $\Ical^3=\{1,2\}$, and the corresponding environment-specific (changed) mechanisms by $p^1(Z_3\mid Z_1)$, $p^2(Z_4\mid Z_1,Z_2,Z_3)$, and $\{p^3(Z_1), p^3(Z_2)\}$, respectively.}
    \label{fig:multi_env}
\end{figure}
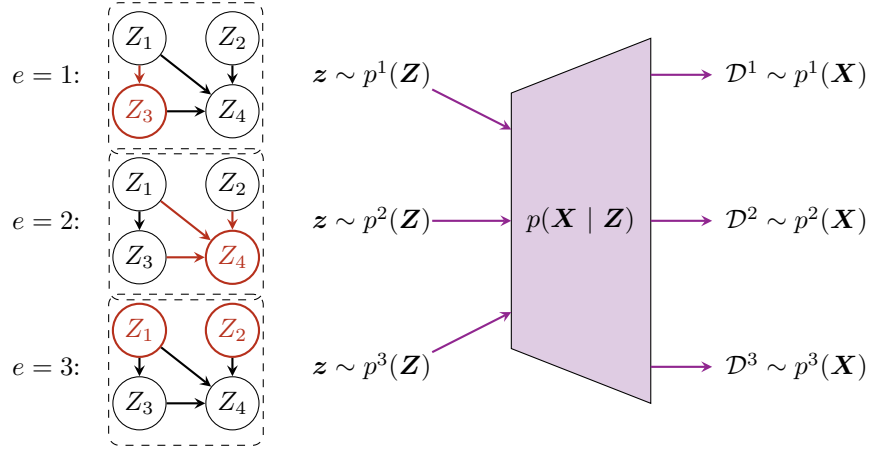

On its own, such multi-domain data is not necessarily useful, since the domain-specific distributions~$p^e$ need not be related in any meaningful way. 
However, we will assume that certain parts of the data generating process are shared across domains.
Concretely, we assume that all domains share the same invariant measurement model $p(\Xb\mid\Zb)$ and arise from (unknown) interventions to some causal mechanisms in a shared causal model, as illustrated in~\cref{fig:multi_env}.
\begin{assumption}[Invariant measurement model]
\label{ass:shared_obs_model}
The observation model $p(\Xb\mid\Zb)$ is invariant across domains. That is, observations $\xb^e$ in~\eqref{eq:multi_env_data} are generated as
\begin{equation}
    \zb^e\sim p^e(\Zb), \qquad \xb^e\sim p(\Xb\mid\Zb=\zb^e).
\end{equation}
\end{assumption}
\begin{assumption}[Shared mechanisms]
\label{assumption:shared_mechs}
Each domain $e$ independently results from a shared latent causal model
by intervening on an (unknown) subset $\Ical^e\subseteq[L]$ of mechanisms. That is, for all $e\in\Ecal$, the interventional joint latent distribution $p^e(\Zb)$ can be written as
\begin{equation}
\label{eq:shared_mechanisms}
    p^e(\Zb)=p^e(Z_1, ..., Z_L) = 
    \prod_{\ell\in \Ical^e} p^e\left(Z_\ell\mid\PA_\ell\right)%
    \prod_{\ell\in[L]\setminus \Ical^e}p\left(Z_\ell\mid\PA_\ell\right),
\end{equation}
where $p(Z_\ell\mid\PA_\ell)$ are the (base) causal mechanisms from~\eqref{eq:causal_factorisation} and $p^e(Z_\ell\mid\PA_\ell)$ are changed domain-specific mechanisms.%
\end{assumption}
\Cref{assumption:shared_mechs} is a common assumption for causal analyses of multi-domain data~\citep{peters2016causal,Rojas-Carulla2018,perry2022causal} 
that is also typical for multi-domain causal representation learning~\citep{ahuja2023interventional,buchholz2023learning,wendong2023causal,von2023nonparametric,squires2023linear,zhang2023identifiability,varici2025score}.
However, the aforementioned identifiability-focused CRL studies mostly consider perfect or hard interventions, which remove any influence from the causal parents such that $p^e(Z_\ell\mid\PA_\ell)=p^e(Z_\ell)$.
We instead consider more general imperfect or soft interventions, which modify the dependence on causal parents, $p^e(Z_\ell\mid\PA_\ell)\neq p(Z_\ell\mid\PA_\ell)$.
\begin{remark}
    \Cref{assumption:shared_mechs} alone imposes no restriction on the generating process if\ $\Ical^e = [L]$. However, when combined with the \textit{sparse mechanism shift} hypothesis~\citep{scholkopf2021toward,perry2022causal}---which we will formalize in~\Cref{sec:model_interventions}---this does impose a meaningful structural assumption.
\end{remark}

\subsection{Discrete Causal Representations and Finite Mixture Models}
\label{sec:discrete}
\begin{figure}[tbp]
\centering
\includegraphics[width=\textwidth]{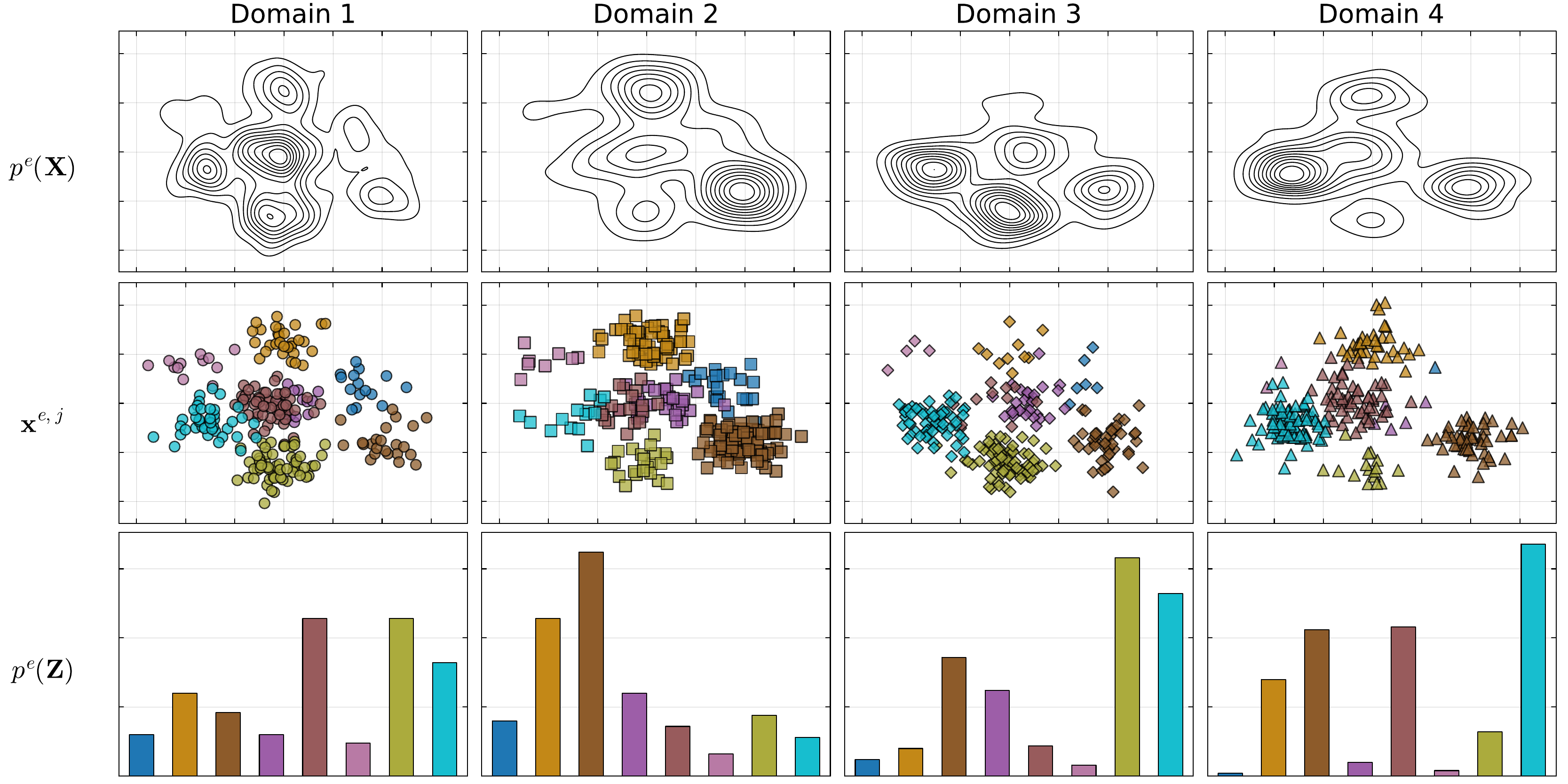} 
\caption{
\textbf{Multi-domain data for CRL with discrete causal latents.} We illustrate an example data generating process with $L=3$ binary causal latents $\Zb=(Z_1, Z_2, Z_3)$, a Gaussian measurement model $p(\Xb\mid\Zb)$ giving rise to $D=2$ continuous observations $\Xb=(X_1,X_2)$, and $|\Ecal|=4$ domains (columns; different markers) corresponding to an observational reference setting and interventions on $Z_1$, $Z_2$, and $Z_3$ (from left to right).
The distributions $p^e(\Xb)$ over observed variables (top row; kernel density estimates based on $N_e = 250$ observations) are mixtures of Gaussians with fixed mixture components corresponding to the $2^3=8$ distinct joint latent states $\zb$ (different colours). The mixture weights are given by $p^e(\Zb=\zb)$ and change across domains (bottom row). 
The center row shows the underlying data annotated with the true but unobserved mixture indices; the cluster locations are shared, but the proportions of observations assigned to each cluster vary by domain.
}
\label{fig:ex_multi_domain_data}
\end{figure}

The case with continuous $\Zb$ can be quite challenging and typically requires strong assumptions such as linear causal relations and perfect single-node interventions~\citep[see, e.g.,][]{von2023nonparametric,squires2023linear,zhang2023identifiability,varici2025score,buchholz2023learning,ahuja2023interventional}.
Here, we focus on the case in which the latent causal variables are discrete, i.e., each  $Z_i$ takes values in a finite domain $\Zcal_i$. 
In this case, the marginal domain-specific distributions $p^e(\Xb)$ are given by
\begin{equation}
\label{eq:mixture_model}
\begin{aligned}
    p^e(\Xb)=\sum_{\zb\in\Zcal}p^e(\Zb=\zb) p(\Xb\mid\Zb=\zb)
\end{aligned}
\end{equation}
where $\Zcal=\Zcal_1\times...\times \Zcal_L$.
In other words, each $p^e(\Xb)$ is a \textit{finite mixture model} with $|\Zcal|$~components indexed by the latent joint states~$\zb$. 
The mixture weights are domain-specific and given by $p^e(\Zb=\zb)$, whereas the mixture distributions are shared (i.e., domain-invariant) and given by $p(\Xb\mid\Zb=\zb)$.
An example of this setting with two binary latents $Z_1,Z_2$, bivariate Gaussians for $p(\Xb\mid\Zb)$, and $|\Ecal|=3$ domains is illustrated in~\cref{fig:ex_multi_domain_data}.

Finite mixture models have been studied extensively, and various situations are known in which the mixture model is identifiable, e.g., when $p(\Xb \mid \Zb = \zb)$ is from the exponential family~\citep{yakowitz1968identifiability}.
In this case, the number of mixture components, as well as the corresponding mixture weights and distributions, can be recovered from $p^e(\Xb)$. 
However, the mixture model in~\eqref{eq:mixture_model} can also arise from a single categorical latent with $\Zcal$ states. 
Hence, mixture identifiability on its own is insufficient to reveal the underlying latent causal structure. 

\section{Hierarchical Bayesian Model for Multi-Domain Discrete CRL}
\label{sec:model}
\looseness-1
In this section, we present our Bayesian approach to learning discrete causal representations from heterogeneous environments by explicitly incorporating the causal structure and assumptions from~\cref{sec:problem_setting} into a hierarchical Bayesian generative model. 
Given the observations of $\Xb$ across multiple domains in~\eqref{eq:multi_env_data}, our goal is to infer the unobserved parts of the data-generating process: the latent causal variables and their relations, the unintervened mechanisms, the domain-specific intervention targets and intervened mechanisms, and the measurement model. 
To this end, we need to first  parametrize the various model components and place suitable priors on all unknowns.
A representation of the resulting graphical model is shown in~\Cref{fig:graphical_model}.

\subsection{Causal Variables and Causal Graph}
\looseness-1 
The causal graph $G$ determines the factorization  of $p(\Zb)$ into causal mechanisms in~\eqref{eq:causal_factorisation}.
In particular, it determines the set of parents $\PA_\ell$ of each $Z_\ell$ and thus also affects the number of parameters~$\Thetab_\ell$ for the corresponding mechanism. %
As a result, models based on different graphs can have a different number of parameters, which poses implementation challenges for comparing models and inferring a posterior over graphs. 
Therefore, we focus on the case in which the cardinalities $|\Zcal_\ell|$ of all discrete latent variables are assumed to be the same.\footnote{In principle, different joint latent states $\zb\neq\zb'$ can share the same $p(\Xb\mid\Zb=\zb)=p(\Xb\mid\Zb=\zb')$, i.e., the mixture model in~\eqref{eq:mixture_model} can be degenerate. Thus, if the cardinalities were different to begin with, we can choose the largest one for each $Z_\ell$.
Any excess capacity should then result in redundancies in $p(\Zb=\zb)$ (i.e., in the mechanism parameters $\Thetab$ and $\Thetab^e$) and $p(\Xb\mid\Zb=\zb)$.
This assumption thus comes w.l.o.g.}
\begin{assumption}[Same cardinality for all latents]
\label{ass:cardinality}
$|\Zcal_\ell|=K$ for all $\ell\in[L]$, where $K$ is known.
\end{assumption}
In this case, the latent variables are interchangeable---recall that the intervention targets $\Ical^e$ are also unknown---and we can fix an arbitrary partial causal ordering w.l.o.g.~\citep{squires2023linear,von2023nonparametric}. 
We choose the natural ordering $Z_1\preceq Z_2 \preceq ... \preceq Z_L$ for convenience.
Further, we assume that the graph is \textit{complete} w.r.t.\ this ordering, i.e.,  $\PA_\ell=\{Z_1, ..., Z_{\ell-1}\}$ for all $\ell\in[L]$.
\begin{assumption}[Fixed causal graph]
\label{ass:fixed_graph}
The latent causal graph $G$ is the complete directed acyclic graph consistent with the ordering $Z_1\preceq Z_2 \preceq ... \preceq Z_L$.
\end{assumption}
While $p(\Zb)$ may not be faithful~\citep{spirtes2001causation} to the complete graph (i.e., it may contain extra edges that are not in the true $G$), faithfulness is not strictly required~\citep{wendong2023causal} since $p^e(\Zb)$ is still Markov w.r.t.~the complete graph, i.e., the factorization~\eqref{eq:causal_factorisation} still holds.
If the $\Thetab_\ell$ are (approximately) correctly recovered, unnecessary edges can be pruned post-hoc based on the dependencies encoded in the learnt mechanisms.

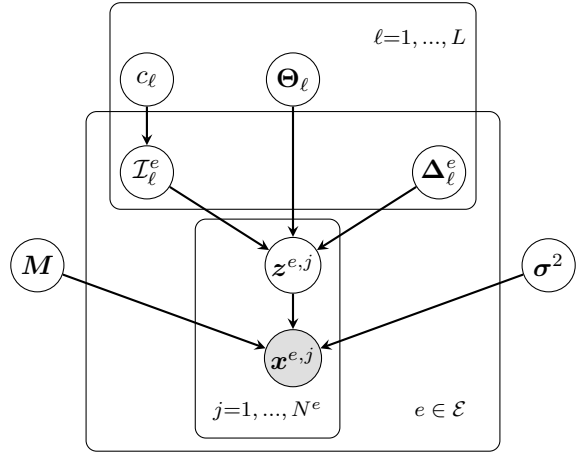
\begin{figure}[tbp]
\newcommand{\xshift}{11.em}
\newcommand{\yshift}{3.5em}
    \centering
    \begin{tikzpicture}
        \node(c)[latent, yshift=-\yshift,xshift=-0.5*\xshift]{$c_\ell$};
        \node(theta)[latent, yshift=-\yshift,xshift=0.*\xshift]{$\Thetab_\ell$};
        \node(I^e)[latent, yshift=-2*\yshift,xshift=-0.5*\xshift]{$\Ical_\ell^e$};
         \node(theta^e)[latent, yshift=-2*\yshift,xshift=0.5*\xshift]{$\Deltab^e_\ell$};
        \node(v)[latent, yshift=-3*\yshift]{$\zb^{e,j}$};
        \node(x)[obs, yshift=-4*\yshift]{$\xb^{e,j}$};
        \node(M)[latent, yshift=-3*\yshift, xshift=-.875*\xshift]{$\Mb$};
        \node(sigma)[latent, yshift=-3*\yshift, xshift=.875*\xshift]{$\sigmab^2$};        
        \edge[-stealth, thick]{c}{I^e};
        \edge[-stealth, thick]{I^e,theta,theta^e}{v};
        \edge[-stealth, thick]{M,v,sigma}{x};
        \plate[inner sep=1.25em] {environments}{(theta^e) (I^e) (v) (x)}{$e\in\Ecal$};
        \plate[inner sep=.65em] {samples}{(v) (x)}{$j{=}1,...,N^e$};
        \tikzset{plate caption/.style={caption, node distance=0, inner sep=0pt,
        below left=-15pt and 0pt of #1.north east,text height=1.2em,text depth=0.3em}}
        \plate[inner sep=.35em] {variables}{(c) (theta) (theta^e) (I^e)}{$\ell{=}1,...,L$};
    \end{tikzpicture}
    \caption{\looseness-1 \textbf{Graphical model representation for Bayesian multi-domain CRL with categorical causal latents.} 
    Shaded and white circles denote observed and unobserved random variables, respectively. 
    Plates indicate that the contained subgraph and incoming edges are repeated across
    the corresponding plate indices $j, \ell, e$.
    For each domain $e\in\Ecal$, we have a sample of $N_e$ observations $\xb^{e,j}$ which result from local latent causal variables $\zb^{e,j}$ via a measurement model $p(\Xb\mid\Zb)$ parametrized by $\Mb$ and $\sigmab^2$.
    The base mechanisms of each $Z_\ell$ are parametrized by $\Thetab_\ell$, the binary unknown intervention targets $\Ical^e_\ell$ are parametrized by $c_\ell$, and the corresponding domain-specific shifts are parametrized by~$\Deltab^e_\ell$.
    }
    \label{fig:graphical_model}
\end{figure}

\subsection{Base Causal Mechanisms} 
\label{sec:model_base_mechanisms}
Since each $Z_\ell$ is discrete, all marginals and conditionals in the factorization of $p(\Zb)$ in~\eqref{eq:causal_factorisation} (i.e., the causal mechanisms) are given by (conditional) Categorical distributions. 
For all $\ell\in[L]$ and all values $\pa_\ell\in\mathcal{PA}_\ell:=\bigtimes_{Z_j\in\PA_\ell}\Zcal_j$ of the causal parents, we parametrize these categoricals as follows:
\begin{equation}
\label{eq:Categorical_mechanism}
    p\left(Z_\ell=k~\middle|~\PA_\ell=\pa_\ell\right) = 
    \softmax\left(\bm\theta_{\ell}\left(\pa_\ell\right)\right)_k
\end{equation}
for parameter vectors $\{\bm\theta_{\ell}(\pa_\ell)\in\RR^{K}\}_{\pa_\ell\in\mathcal{PA}_\ell}$ where $\softmax(\xb):\RR^K\to\Delta^{K-1}\subset(0,1)^K$ denotes the softmax function
\begin{equation}
    \softmax(\xb)_k=\frac{\exp(x_k)}{\sum_{j=1}^K \exp(x_j)}
\end{equation}
and 
  $\Delta^{K-1}=\{(\pi_1, ..., \pi_{K}):\pi_k\geq 0, \sum_{k=1}^{K}\pi_k=1\}$
denotes the $(K-1)$-dimensional probability simplex. 
We collect the parameters $\bm\theta_{\ell}(\pa_\ell)$ for all values of $\pa_\ell$ in a  matrix $\Thetab_\ell\in \RR^{K\times|\PAcal_\ell|}$.

\subsubsection{Mechanism Priors}
\label{subsec:mechanism_priors}
We place a multivariate Gaussian prior on the mechanism parameters, i.e.,
\begin{equation}
    \label{eq:mechanism_prior}
    \left(\thetab_\ell(\pa_\ell)\right)_{\ell\in[L],\pa_\ell\in\PAcal_\ell} \iidsim 
        \Ncal_K\Big(\mathbf{0}, \, \Sigmab_\Theta \Big)
        ,
\end{equation}
with covariance matrix given by
\begin{equation}
    \label{eq:covariance_theta}
    \Sigmab_\Theta= \text{diag}\left(\sigmab_\Theta^2\right) - \frac{1}{K} \sigmab_\Theta \sigmab_\Theta^\top
\end{equation}
for some hyperparameter vector $\sigmab_\Theta \in \RR_+^K$, which controls the variance of each logit dimension. 
Larger values of $\sigma_i$ increase the prior uncertainty on the corresponding logit~$\theta_i$. Conversely, smaller $\sigma_i$ values shrink the logits toward zero, concentrating the prior over nearly uniform categorical probabilities.

The specific covariance structure in~\eqref{eq:covariance_theta} is chosen to project out the constant vector $\bm1_K:=(1,1,...,1)$.
This is useful because the distribution induced by a vector of logits  $\thetab\in\RR^K$ is invariant to shifting all entries by the same constant, i.e., $\softmax(\thetab)=\softmax(\thetab+a\bm1)$ for all $a\in\RR$. 
The choice from~\eqref{eq:covariance_theta} removes the redundant degree of freedom due to the simplex constraint.

\subsection{Intervention Targets}
\label{sec:model_interventions}
Whether the base causal mechanism $p(Z_\ell\mid\PA_\ell)$ of $Z_\ell$ or an intervened version $p^e(Z_\ell\mid\PA_\ell)$ thereof is active in domain $e$ is determined by the (unknown) intervention targets $\Ical^e$. 
We model these as binary vectors $\Ical^e\in\{0,1\}^L$ where $\Ical^e_\ell=1$ indicates an intervention on $Z_\ell|\PA_\ell$ in domain $e$, and $\Ical^e_\ell=0$ indicates no intervention.
According to the ICM principle~\citep{peters2017elements}, the causal mechanisms ``do not inform or influence each other''. Hence, $\Ical^e_1, ..., \Ical^e_L$ should be independent within any domain $e$. Moreover, we assume that the intervention targets are chosen independently for each domain, i.e., that $\{\Ical^e\}_{e\in\Ecal}$ are jointly independent. 
However, we do allow the probability of an intervention to differ across $\ell\in[L]$, encoding the notion that some causal mechanisms may be more stable than others. 
The intervention targets can thus be viewed as the result of independent (but biased) coin flips, i.e., for all $\ell\in[L]$: %
\begin{align}
\label{eq:intervention_targets}
    \left(\Ical^e_\ell\right)_{e\in\Ecal}\iidsim \mathrm{Bernoulli}(c_\ell)\,,
\end{align}
where the mechanism change probabilities $c_\ell \in[0,1]$ have the interpretation of \textit{instability parameters}, with $c_\ell=0$ indicating perfect stability (i.e., the mechanism for $Z_\ell|\PA_\ell$ never changes) and  $c_\ell=1$ indicating complete instability (i.e., the mechanism changes in each new domain).

\subsubsection{Instability Priors}
\label{sec:instability_priors}
Motivated by the \textit{sparse mechanism shift} hypothesis~\citep{scholkopf2021toward,perry2022causal}, which states that distribution changes tend to only affect a few factors in~\eqref{eq:causal_factorisation},
we choose an asymmetric prior for the instability parameters $c_\ell\in[0,1]$, which puts more probability mass on small shift probabilities, thus inducing sparsity. 
Specifically, we choose
\begin{equation}
    \label{eq:shift_prior}(c_\ell)_{\ell\in[L]}\iidsim\mathrm{Beta}(\alpha_c,\beta_c)
\end{equation}
with fixed $\alpha_c,\beta_c$ such that the expected shift probability for each $Z_\ell|\PA_\ell$ in any new domain, given by $\EE[c_\ell]=\frac{\alpha_c}{\alpha_c+\beta_c}$, is small. In all experiments, we set $\alpha_c = 1,\, \beta_c = 9$,  for a prior expected sparsity of $90\%$.

\subsection{Intervened Mechanisms}  
\label{sec:model_intervened_mechanisms}
In domains where an intervention occurs on a mechanism $Z_\ell|\PA_\ell$ (i.e., if $\Ical^e_\ell=1$), we model the resulting \emph{intervened mechanism} $p^e(Z_\ell\mid\PA_\ell)$ as an exponential tilting of the base mechanism. Specifically, for any domain~$e$ and parent configuration~$\pa_\ell \in \mathcal{PA}_\ell$, we define
\begin{equation}
\label{eq:intervened_categorical}
    p^e\left(Z_\ell=k \mid \PA_\ell=\pa_\ell\right) = 
    \softmax\Big(\bm\theta_{\ell}(\pa_\ell) + \Ical^e_\ell  \bm\delta^e_\ell(\pa_\ell)\Big)_k,
\end{equation}
where $\bm\delta^e_\ell(\pa_\ell)$ encodes the intervention-specific perturbation of the logits relative to the base mechanism $\bm\theta_\ell(\pa_\ell)$, which determines both the direction of the change in probability mass across categories and the magnitude of that shift. 
Similar to the base mechanisms $\Thetab_\ell$, we collect the parameters $\bm\delta^e_{\ell}(\pa_\ell)$ for all values of $\pa_\ell$ in a  matrix $\Deltab^e_\ell\in \RR^{K\times|\PAcal_\ell|}$.

Given the intervention targets and base/intervened causal mechanisms, we can now define the distributions for the different domain-specific realizations of the latent causal variables. 
According to Assumption~\ref{assumption:shared_mechs} and~\eqref{eq:Categorical_mechanism}, for each domain $e\in\Ecal$ we have:
\begin{align}
\label{eq:domain_specific_latent_distribution}
    \left(\zb^{e,i}\right)_{i\in[N^e]}\iidsim p^e(\Zb)&
     =\prod_{\ell=1}^L 
     p^e\left(Z_\ell\mid\PA_\ell\right)
     =
     \prod_{\ell\,:\,\Ical^e_\ell=1} p^e\left(Z_\ell\mid\PA_\ell\right)%
    \prod_{\ell\,:\,\Ical^e_\ell=0}p\left(Z_\ell\mid\PA_\ell\right)
\end{align}

\subsubsection{Shift Priors}
\label{sec:shift_priors}
With finite data, arbitrarily small interventions are not statistically distinguishable from no intervention. Thus, sparsity of interventions alone, as induced by the Beta-Bernoulli prior over intervention targets from~\cref{sec:instability_priors}, is not sufficient for finite-sample recovery. A shift may be technically nonzero while inducing only an arbitrarily small, undetectable change in the conditional mechanism \(p^e(\zb_\ell \mid \pa_\ell)\). 
As formalized in Lemma~\ref{lem:kl_min_detectability}, %
interventions that induce a sufficiently small change in KL cannot be reliably distinguished from the null (i.e., unintervened) mechanism. 
This results in %
an unstable sparse
recovery problem: small environment-specific fluctuations can be explained by
many weak interventions, even when these fluctuations are due to finite-sample
noise or mild model misspecification. 

This suggests a refinement or operationalization of the sparse mechanism hypothesis~\citep{scholkopf2021toward} for finite-sample scenarios:
``interventions \textit{with non-negligible effects} are sparse''.
We formalize this by imposing the following detectability condition requiring active interventions to induce a sufficiently large distributional change.
\begin{assumption}[\(\varepsilon\)-intervention detectability]
\label{assump:epsilon_intervention_detectability}
For every candidate intervention \((\ell,e)\), the intervention is either absent or induces a KL shift of at least \(\varepsilon\) for all parent configurations: 
\[
p(\Zb_\ell \mid \pa_\ell) = p^e(\Zb_\ell \mid \pa_\ell)
\quad \text{for all } \pa_\ell \in \mathcal{PA}_\ell,
\]
or
\[
\operatorname{KL}\left(
p(\Zb_\ell \mid \pa_\ell)
\,\middle\|\,
p^e(\Zb_\ell \mid \pa_\ell)
\right) \geq \varepsilon
\quad \text{for all } \pa_\ell \in \mathcal{PA}_\ell.
\]
\end{assumption}

 Thresholded detectability conditions of this kind appear frequently in the finite-sample causal discovery literature, where causal dependencies or intervention effects must be bounded away from zero to be distinguishable from statistical noise~\citep{uhler2013geometry,zhang2012faithfulness,chevalley2025deriving}. The closest analogue to our condition is \(\varepsilon\)-interventional faithfulness~\citep{chevalley2025deriving} which imposes a very similar divergence-based lower bound on intervention-induced distributional changes. More generally, such minimum signal strength conditions are common in various settings such as sparse support recovery, where beta-min assumptions require nonzero coefficients to be sufficiently large~\citep{wainwright2009sharp}, and conditional independence testing, where alternatives are typically required to be separated in total variation distance~\citep{canonne2018testing,neykov2021minimax}.

To specify this shift prior, we begin with a similar structure to the mechanism parameters. The environment-specific shifts  $\deltab^e_\ell$ in logit space are also drawn from a multivariate normal distribution $\Ncal(\bm0,\Sigmab_\Delta)$
with covariance given by
\begin{equation}
\label{eq:covariance_delta}
    \Sigmab_\Delta=\text{diag}\left(\sigmab_\Delta^2\right) - \frac{1}{K} \sigmab_\Delta \sigmab_\Delta^\top
\end{equation}
for some hyperparameter $\sigmab_\Delta\in\RR^K_+$. Similar to $\Sigmab_\Theta$ from~\eqref{eq:covariance_theta}, the specific form in~\eqref{eq:covariance_delta} is chosen to project out logit shifts of the form $a\cdot\bm1$ which have no effect, see~\cref{subsec:mechanism_priors} for details.

For tractability, we do not enforce the KL condition in \cref{assump:epsilon_intervention_detectability} exactly. Instead, we approximate the KL-shift magnitude,
\[
\operatorname{KL}\left(
\softmax(\thetab)
\,\middle\|\,
\softmax(\thetab+\deltab)
\right),
\]
by evaluating it at \(\thetab=\bm0\) and taking its second-order expansion with
respect to \(\deltab\). This gives the approximate constraint,
\begin{equation}
    \label{eq:hessian_approximation}
    \frac{1}{2}\deltab^\top \Hb  \deltab \geq \varepsilon, \qquad \text{where} \qquad 
    \Hb = \frac{1}{K}\left( \Ib - \frac{1}{K} \bm{1} \bm{1}^\top\right),
\end{equation}
and yields a $\thetab$-independent truncation region, which keeps the shift prior computationally atractable while ruling out nearly trivial shifts. We find this approximation sufficient in the small-KL ($\varepsilon=0.1$) regime used in our experiments; see~\cref{fig:kl_normal} for an illustration and \Cref{app:KL} for details.\looseness-1

The (unnormalized) truncated shift prior is then given by
\begin{equation}
 \label{eq:truncated_normal}
\left(\deltab^e_\ell(\pa_\ell)\right)_{e\in\Ecal,\ell\in[L],\pa_\ell\in\PAcal_\ell} 
\iidsim 
p\left(\deltab \mid \Sigmab_\Delta, \varepsilon \right) 
\propto 
\mathcal{N}_K\left(\bm0, \Sigmab_\Delta\right) \cdot\mathbbm{1}\left\{\frac{1}{2}\deltab^\top \Hb  \deltab \geq \varepsilon\right\}.
\end{equation}
Empirically, we found that softer sparsity-inducing priors, such as the Beta-Binomial priors on intervention indicators alone, were insufficient to consistently recover sparse interventions. These priors penalize the number of interventions but do not rule out many weak shifts, and as such often hallucinated extra interventions when fit to data. Hard KL-thresholding was significantly more robust, as shown in~\cref{sec:evaluation}. This robustness is critical, since sparse intervention recovery is what enables recovery of the latent causal order though the \textit{sparse mechanism shift} hypothesis.

\begin{figure}[tbp]
\centering
\includegraphics[height=0.32\textwidth]{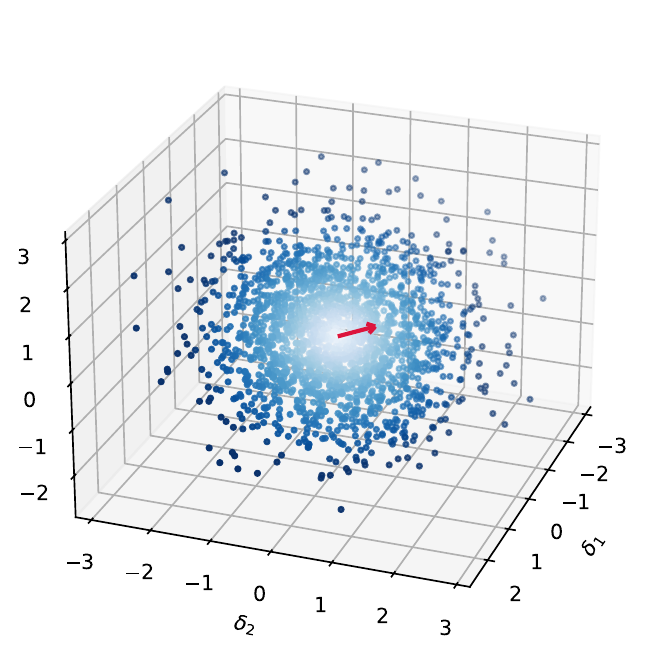}
\includegraphics[height=0.32\textwidth]{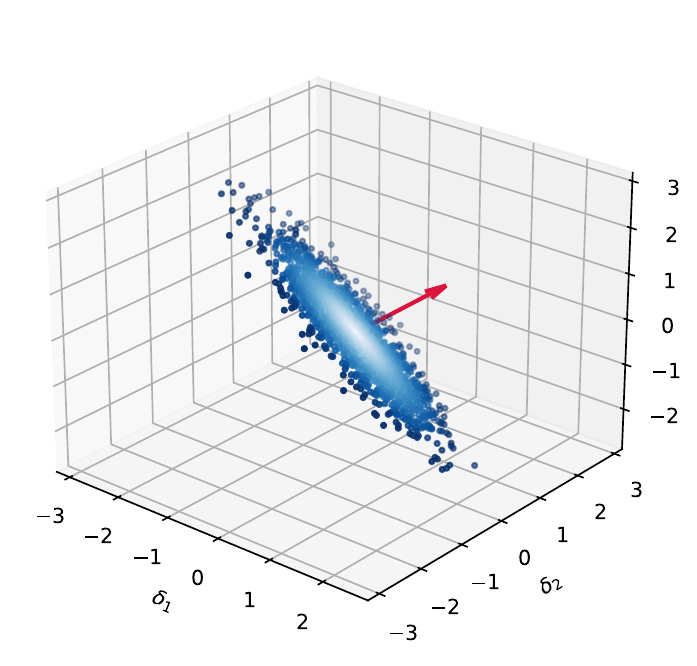}
\includegraphics[height=0.3\textwidth]{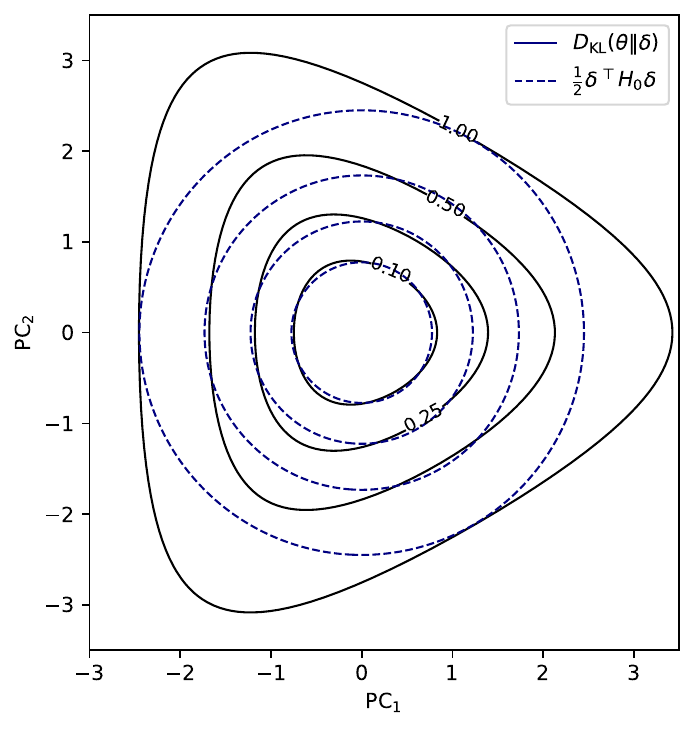}
\caption{
\textbf{Illustration of the quadratic KL constraint on $\deltab$.}
(Left, Middle) Samples of $\deltab$, colored by the quadratic approximation $\tfrac{1}{2}\deltab^\top \Hb \deltab$. 
The red arrow indicates the constant vector $\bm 1$ which is orthogonal to the subspace of shift vectors.
(Right) Contours of the exact KL divergence and its quadratic approximation, shown on a two-dimensional projection of $\deltab$ onto its principal components.}
\label{fig:kl_normal}
\end{figure}

\subsection{Measurement Model and Observations}
\label{sec:model_measurement}
Each latent vector $\zb^{e,j}$ generates a corresponding observation $\xb^{e,j}$ through a domain-invariant measurement model $p(\xb \mid \zb)$. Here, we adopt the following additive Gaussian measurement model,
\begin{align}
\label{eq:Gaussian_measurement_model}
    \xb^{e,j} &\sim \mathcal{N}_D\left(\mub\left(\zb^{e,j}\right), \mathrm{diag}\left(\sigmab^2\right)\right).
\end{align}
The mean function $\mub:\RR^{L}\to\RR^D$ is defined as
\begin{align}
\label{eq:Gaussian_measurement_model_mean}
    \mub(\zb) = \mb_0 + \sum_{\ell \in [L]} \mb_\ell z_\ell.
\end{align}
where each $\mb_\ell$ represents the additive contribution of latent coordinate $Z_\ell$ to the mean of $\xb$.
We collect these mean parameters together with an offset $\mb_0$ as columns in a measurement model matrix $\Mb \in \RR^{D \times (L+1)} = \begin{bmatrix} \mb_0 & \hdots & \mb_L \end{bmatrix}$. 

Here, each $Z_\ell$ is centered around $0$ with unit spacing:
\[
z_\ell \in \left\{1,2,\cdots, K\right\} -  \frac{K+1}{2},
    \qquad \text{e.g.,} \quad \{-2,-1,0,1,2\}\text{ or } \left\{-\frac{3}{2},-\frac{1}{2},\frac{1}{2},\frac{3}{2}\right\}.
\]
This parameterization is chosen with interpretability in mind. In particular, $\mb_\ell$ can be interpreted as an effect-size vector where large absolute entries indicate strong associations between latent variable~$z_\ell$ and observed dimensions of $\xb$.

\subsubsection{Measurement Model Priors}
We assume that the offset $\mb_0$ and the additive influences $\mb_\ell\in\RR^D$ of each $Z_\ell$ on the mean of $\Xb$ (i.e., the columns of~$\Mb$) are i.i.d.\ draws from a standard isotropic Gaussian, i.e.,
\begin{equation}
\label{eq:prior_M}
    (\mb_\ell)_{\ell\in[L]\cup\{0\}}\iidsim\Ncal_D(\bm 0, \Ib)\,.
\end{equation}
Further, we assume that the component-wise variances are sampled i.i.d.\ from an inverse Gamma prior which puts higher weight on small variances,
\begin{equation}
    \label{eq:prior_sigma}
(\sigma^2_d)_{d\in[D]}\iidsim\mathrm{Inv-}\Gamma(\alpha_\sigma=3,\beta_\sigma=1)\,.
\end{equation}

\subsection{Permutation Symmetry of the Measurement Model}
\label{sec:model_permutation}
A central aspect of the measurement model is its invariance under relabeling of latent coordinates, which is crucial to recover the causal relationships among the inferred latent variables $Z_1, \dots, Z_L$. Formally, for any permutation $\pi \in S_L$ of the $L$ latent coordinates,
\begin{equation}
\label{eq:permutation_symmetry}
p(\xb \mid \zb, \Mb, \sigmab) = p(\xb \mid \zb_\pi, \Mb_\pi, \sigmab),
\end{equation}
where 
\[
\zb_\pi = (z_{\pi(1)}, \dots, z_{\pi(L)}), \qquad
\Mb_\pi = \begin{bmatrix} \mb_0 & \mb_{\pi(1)} & \dots & \mb_{\pi(L)} \end{bmatrix}.
\]
 Because we fix the causal graph to the complete DAG with ordering $Z_1 \preceq \dots \preceq Z_L$, inference over causal relations does not proceed by modifying the graph%
 , but rather by what information is encoded in latents that come earlier (e.g., $Z_1$) versus later (e.g., $Z_L$) in the fixed causal hierarchy. The measurement model %
 determines this encoding %
 by relating latent concepts to observed responses. 

In our additive Gaussian model specifically, this interpretation is encoded in the columns of the measurement matrix $\Mb$: the column $\mb_\ell$ determines how $Z_\ell$ affects the mean of $\xb$.  Therefore, permutation of the columns of $\Mb$ acts as a permutation of the causal order of latent concepts. 
As a result, the measurement model parametrization is not arbitrary: it must admit permutations of its parameters so that any reordering of the latent coordinates can be correspondingly represented.

Importantly, while the measurement model is symmetric, the complete model is not. As we %
show in \cref{sec:multimodality}, the priors over $\Zb$ introduce asymmetries that allow distinguishing %
causal orders.

\begin{remark}[More General Measurement Models]
The property that permutations of latent coordinates can be paired with corresponding permutations of measurement model parameters to induce the same distribution is not specific to our choice of an additive Gaussian model. 
Any measurement model whose parameters can be symmetrically permuted to reflect a relabeling of the latent coordinates will exhibit the same permutation invariance, and could be used in place of the additive Gaussian model.
\end{remark}

\subsection{Complete Model} 
\label{sec:complete_model}

The full generative process described in the previous subsections is given by:
\begin{align*}
\forall \ell\in[L], \pa_\ell\in\mathcal{PA}_\ell:& & \thetab_\ell(\pa_\ell) &\sim \mathcal{N}_K\left(\mathbf{0}, \Sigmab_\Theta\right) 
\\ 
\forall \ell\in[L]:& & c_\ell &\sim \operatorname{Beta}(\alpha_c , \beta_c)
\\
\forall e\in\Ecal, \ell\in[L]:& & \Ical_\ell^e &\sim \operatorname{Bernoulli}(c_\ell) 
\\ 
\forall e\in\Ecal, \ell\in[L], \pa_\ell\in\mathcal{PA}_\ell:& & \qquad \deltab_\ell^e(\pa_\ell) &\sim p\left(\deltab \mid \Sigmab_\Delta, \varepsilon \right) 
\propto 
\mathcal{N}_K\left(\bm0, \Sigmab_\Delta\right) \cdot\mathbbm{1}\left\{\tfrac{1}{2}\deltab^\top \Hb  \deltab \geq \varepsilon\right\}
\\
\forall e\in\Ecal,\ell\in[L], j\in[N_e]:& & z_\ell^{e,j} &\sim \operatorname{Categorical}\left(\softmax\left(\thetab_\ell\big(\pa_\ell^{e,j}\big) + \Ical_\ell^e \, \deltab_\ell^e\big(\pa_\ell^{e,j}\big)\right)\right) 
\\ 
\forall \ell\in[L]\cup\{0\}:& & \mb_\ell &\sim\Ncal_D(\bm 0, \Ib)\ 
\\
\forall d\in[D]:& & \sigma^2_d & \sim\mathrm{Inv-}\Gamma(\alpha_\sigma,\beta_\sigma)
\\
\forall e\in\Ecal,j\in[N_e]:& & \xb^{e,j} &\sim \mathcal{N}_D\left(\mb_0 +\textstyle\sum_{\ell \in [L]} \mb_\ell z_\ell^{e,j},
\mathrm{diag}\big(\sigmab^2\big)\right)
\end{align*}

For notational convenience, we introduce the shorthands $\Thetab=(\Thetab_\ell)_{\ell\in[L]}$, $\cbb=(c_1, ..., c_L)$, ${\Ical^\Ecal=(\Ical^e_\ell)_{\ell\in[L],e\in\Ecal}}$, $\Deltab^\Ecal=(\Deltab^e_\ell)_{\ell\in[L],e\in\Ecal}$, $\Zb^\Ecal=(\zb^{e,j})_{e\in\Ecal,j\in[N^e]}$, and $\Xb^\Ecal=(\xb^{e,j})_{e\in\Ecal,j\in[N^e]}$,  and separately collect the parameters of the latent distributions and measurement model respectively into 
\begin{align*}
\Phib&\coloneqq\left(\Thetab, \cbb, \Ical^\Ecal, \Deltab^\Ecal\right), \\
\Psib&\coloneqq\left(\Mb, \sigmab\right).
\end{align*}
The joint distribution of all observed and unobserved variables then factorizes as follows:
\begin{equation*}
p\left(\Phib,\Psib, \Zb^\Ecal,\Xb^\Ecal\right)=
p\left(\Phib\right)\, p\left(\Psib\right) \,p\left(\Zb^\Ecal\mid\Phib\right)\,p\left(\Xb^\Ecal\mid \Zb^\Ecal,\Psib\right)
\end{equation*}
Unless otherwise specified, hyper-parameters are fixed across all experiments. 
In particular, we use $\alpha_c = 1$ and $\beta_c = 9$ as a sparsity-inducing Beta prior on the instability parameters $c_\ell$, and $\alpha_\sigma = 3$ and $\beta_\sigma = 1$ for the inverse-gamma prior on the noise variances $\sigma_d^2$. 
For the mechanism priors, we use $\varepsilon = 0.1$ as the lower threshold for the shift prior, and $\sigmab_\Theta = \sigmab_\Delta = \bm{1}$ to construct the covariances $\Sigmab_\Theta$ and $\Sigmab_\Delta$.

\section{Posterior Inference with Sequential Monte Carlo Sampling}
\label{sec:inference}
Given a set of observations $\Xb^\Ecal$ collected across multiple environments, our inference target is the joint posterior distribution over the latent causal variables $\Zb^\Ecal$, the parameters governing the latent distributions~$\Phib$, and the parameters of the measurement model~$\Psib$ conditioned on the observed $\Xb^\Ecal$. Formally, we seek to approximate
\[
p\left(\Phib,\Psib, \Zb^\Ecal \mid \Xb^\Ecal\right).
\]
As with many latent-variable models, this posterior distribution is analytically intractable. We therefore rely on Monte Carlo methods for posterior approximation.

\subsection{Conditional Conjugate Updates}
The model is constructed such that most parameters admit efficient closed-form conjugate updates. In these cases, forming the complete conditional distribution (i.e., the conditional distribution given all other variables) amounts to incrementing the natural parameters of the prior.  For example, for the intervention instability parameters ($\cbb$) and intervention target indicators ($\Ical$), these follow from Beta-Binomial conjugacy and Bernoulli updates, respectively,
\begin{align*}
(c_\ell \mid -) &\sim \operatorname{Beta}\left(\alpha_c + \sum_{e\in \Ecal }\Ical_\ell^e,\beta_c+|\Ecal|-\sum_{e\in \Ecal }\Ical_\ell^e\right), 
\\
 p(\Ical_\ell^e = \Ical \mid -) &\propto 
c_\ell^{\Ical}(1 - c_\ell)^{1-\Ical}\prod_{j\in [N_e]}
\softmax\Big(\bm\theta_{\ell}\big(\pa_\ell^{e,j}\big) + \Ical %
\deltab^e_\ell\big(\pa_\ell^{e,j}\big) \Big)_{z^{e,j}_\ell}.
\end{align*}

The notable exception to closed-form conjugacy are the logit parameters $\Thetab$ and $\Deltab$, which do not admit standard conjugate updates. We address this using Pólya--Gamma augmentation, a data-augmentation technique originally developed for Bayesian logistic regression with Gaussian priors~\citep{polson2013bayesian} and extended to multinomial models~\citep{chen2013polya}. This approach introduces Pólya-Gamma-distributed auxiliary variables that render the logistic likelihood into a conditionally Gaussian form, allowing $\Thetab$ and $\Deltab$ to be updated via standard Gaussian updates. Full conditional derivations and implementation details are provided in \Cref{app:inference_conj}.

A standard way to exploit these conditional updates to sample from the posterior is through a Gibbs sampler that iteratively samples each variable from its complete conditional distribution $p(\cdot \mid -)$.

\subsection{Sequential Monte Carlo Sampling (SMCS)}
\label{sec:SMCS}
In our case, despite the availability of efficient conditional updates, straightforward Gibbs sampling performs poorly due to its inability to mix over a multimodal posterior. As discussed in \Cref{sec:model_permutation} and illustrated in~\Cref{sec:multimodality}, permutations of the measurement model correspond to permutations of the latent causal ordering. Each such permutation defines a local posterior mode where the measurement likelihood remains the same across permutations, but the latent terms differ. Consequently, accurately recovering the causal structure requires exploring all of these modes.

We employ Sequential Monte Carlo Sampling (SMCS), a class of population-based Monte Carlo methods that are particularly effective for exploring complex, multimodal posteriors \citep{del2006sequential,dai2022invitation}. Unlike MCMC, which iteratively updates a single posterior state, SMCS uses importance sampling to evolve a population of candidate posterior states, referred to as particles, through a sequence of intermediate distributions that gradually approach the true posterior. 

\subsubsection{Tempered Posterior Targets}
In our implementation, this sequence is constructed via complete-data-likelihood tempering. Let
\[
\infty =T_0 > T_1 > \cdots > T_K = 1
\]
denote a decreasing temperature schedule. For each temperature $T_k$, define the tempered posterior
\begin{equation}
p_{T_k}(\Phib,\Psib,\Zb^\Ecal \mid \Xb^\Ecal)
\;\propto\;
p(\Phib)p(\Psib)
\left[
p\left(\Zb^\Ecal, \Xb^\Ecal \mid \Psib, \Phib\right)
\right]^\frac{1}{T_k}.
\label{eq:tempered-posterior}
\end{equation}
When $T_0=\infty$, the contribution of terms that scale with the size of the data is removed, and the target distribution reduces to the prior over global
parameters, $(\Phib,\Psib)$, together with a flat distribution over
$\Zb^\Ecal$.
As the temperature is lowered, the complete-data-likelihood contribution is gradually introduced, until the final temperature, $T_K=1$, where the target distribution is exactly the posterior distribution of interest.
Tempering has two useful interpretations. First, at high temperatures the target distribution is flattened, reducing energy barriers between posterior modes and allowing particles to explore the parameter space more freely. Second, by initially weakening the $(\Zb^\Ecal, \Xb^\Ecal)$ terms that scale with the size of the data, tempering allows the structured global priors to meaningfully guide inference before the likelihood becomes dominant. These effects make SMCS particularly effective for exploring our multi-modal posterior, where preserving the influence of our model priors is critical for valid inference over latent causal structure.

\subsubsection{SMCS Steps}
When transitioning from one intermediate distribution to the next, SMCS uses importance sampling to account for the change in the target distribution. Each particle carries an importance weight, and as the temperature decreases from $T_{k-1}$ to $T_k$, particles are \textbf{re-weighted} to adjust for this shift in the target density $p_{T_k}$. However, importance sampling alone can quickly lead to \emph{weight degeneracy}: after several temperature updates, most of the total weight may be concentrated on only a small number of particles. To mitigate this, SMCS augments importance sampling with \textbf{resampling}, which discards particles with negligible weight and duplicates higher weight particles. This introduces its own issue of \emph{path degeneracy}, where most of the particles are identical copies. This motivates the third component \textbf{mutation}, which applies Markov transition kernels to each particle. improving diversity while maintaining the current target density. This preserves the current target distribution, while increasing particle diversity and enabling local exploration of the current tempered distribution. The algorithm therefore proceeds by iterating through the temperature schedule and repeatedly applying these three operations: \textbf{reweighting}, \textbf{resampling}, and \textbf{mutation}.

\begin{algorithm}[tbp]
\caption{Sequential Monte Carlo with Likelihood Tempering}
\label{alg:smc}
\begin{algorithmic}[1]
\Require Temperatures $\{T_0,\ldots,T_K\}$, number of particles $S$, unnormalized tempered posterior from Eq.~\ref{eq:tempered-posterior} ($p_{T_k}(\Phib,\Psib,\Zb \mid \Xb)$), and tempered $p_{T_k}$-invariant MCMC kernels (\textsc{GibbsUpdate}).\looseness-1
\State %
$(\Phib,\Psib,\Zb)_0^{(s)} \iidsim p_{T_0}(\Phib,\Psib,\Zb)$ for all $s \in [S]$ \Comment{Initialize particles}
\State %
$w_0^{(s)} \gets 1/S$ for all $s \in [S]$\Comment{Initialize weights}
\For{$k = 1$ {\bf to} $K$}
    		\State Set weights $w_k^{(s)} \propto w_{k-1}^{(s)}%
		 \left[p(\Zb , \Xb  \mid \Phib, \Mb)\right]^{1/T_k - 1/T_{k-1}}$  for all $s \in [S]$\Comment{Reweight}
        	\State $\{(\Phib,\Psib,\Zb)_k^{(s)}, w_{k}^{(s)}\}_{s=1}^S 
		 	  \gets \textsc{Resample}\!\left(\left\{(\Phib,\Psib,\Zb)_{k}^{(s)}, w_{k}^{(s)}\right\}_{s=1}^S\right)$
		\Comment{Resample}

     	  	 \State $(\Phib,\Psib,\Zb)_k^{(s)} \leftarrow \textsc{GibbsUpdate}\left((\Phib,\Psib,\Zb)_k^{(s)}, T_k\right)$ for all $s \in [S]$\Comment{Mutate}
\EndFor  
\State {\bf return} $\left\{\left((\Phib,\Psib,\Zb)_K^{(s)},\, w_K^{(s)}\right)\right\}_{s=1}^S$ as approximate posterior samples
\end{algorithmic}
\end{algorithm}

\begin{enumerate}[leftmargin=*]
    \item 
\textbf{Reweighting:}
Given particles $\{(\Phib,\Psib,\Zb)_{k-1}^{(s)}\}_{s=1}^S$ with normalized weights $w_{k-1}^{(s)}$, we compute un-normalized importance weights induced by the change from $T_{k-1}$ to $T_k$:
\[
    \tilde{w}_k^{(s)}
    \;=\;
      w_{k-1}^{(s)}
     \frac{p_{T_k}(\Phib^{(s)},\Psib^{(s)},\Zb^{(s)} \mid \Xb)}{p_{T_{k-1}}(\Phib^{(s)},\Psib^{(s)},\Zb^{(s)} \mid \Xb)}
         \;=\;
    w_{k-1}^{(s)}
      \left[p(\Zb^{(s)} \mid \Phib^{(s)})  p(\Xb \mid \Zb^{(s)} , \Psib^{(s)})  \right]^{\frac{1}{T_k} - \frac{1}{T_{k-1}}}.
\]
These weights are then normalized as
\[
w_k^{(s)} = \frac{\tilde{w}_k^{(s)}}{\sum_{r=1}^S \tilde{w}_k^{(r)}},
\quad \text{so that } \sum_{s=1}^S w_k^{(s)} = 1.
\]
\item 
\textbf{Resampling:}
As tempering progresses, weight degeneracy is monitored using effective sample size,\looseness-1
\[
\mathrm{ESS}_k = \frac{1}{\sum_{s=1}^S \left(w_k^{(s)}\right)^2}.
\]
When $\mathrm{ESS}_k < \tau S$ (with $\tau = 0.5$ in our implementation), we perform a resampling step by re-drawing $S$ particles with replacement from the current set, where the expected number of times each particle is selected is proportional to its importance weight, specifically
\[
\mathbb{E}[N_s] = S w_k^{(s)},
\]
where $N_s$ denotes the number of copies of particle $s$ after resampling.  See appendix~\cref{app:inference_resampling} for exact details on the \emph{stratified resampling} \citep{Kitagawa1996Filter} scheme used.

\item
\textbf{Mutation via Gibbs updates:}
After resampling, the particles are mutated using an MCMC step that leaves the tempered posterior $p_{T_k}(\Phib,\Psib,\Zb \mid \Xb)$ invariant. In our implementation, this mutation is performed by Gibbs updates. For each particle, the sampler cycles through the latent variables and resamples each from its complete conditional distribution under temperature $T_k$.

\end{enumerate}
These three steps are repeated until the final temperature $T=1$ is reached, yielding a weighted particle approximation to the posterior:
\[
p(\Phib,\Psib,\Zb^\Ecal \mid \Xb^\Ecal)
\approx
\sum_{s=1}^S w_K^{(s)} \, \delta_{(\Phib,\Psib,\Zb^\Ecal)_K^{(s)}}
\]
A summary of the full algorithm is given in Algorithm~\ref{alg:smc}.
Code for our model and data for all experiments is available at: 
\href{https://github.com/agarg7/discrete-bayesian-crl}{\texttt{github.com/agarg7/discrete-bayesian-crl}}.

\section{Synthetic Experiments}
\label{sec:synthetic_experiments}
We first evaluate our method on synthetic data to study the  structure of the posterior and assess recovery of the latent causal structure in a controlled setting. We first examine the geometry of the posterior in \Cref{sec:multimodality}, demonstrating how the measurement model and latent mechanism priors interact to produce distinct, asymmetric posterior modes. We then compare our approach against alternative inference methods and model ablations in \cref{sec:evaluation}, evaluating both predictive fit and recovery of the underlying latent structure. These comparisons demonstrate the efficacy of both our model structure and our inference method.
\subsection{Posterior Multimodality}
\label{sec:multimodality}

\begin{figure}[tbp]
\centering
\includegraphics[width = \textwidth]{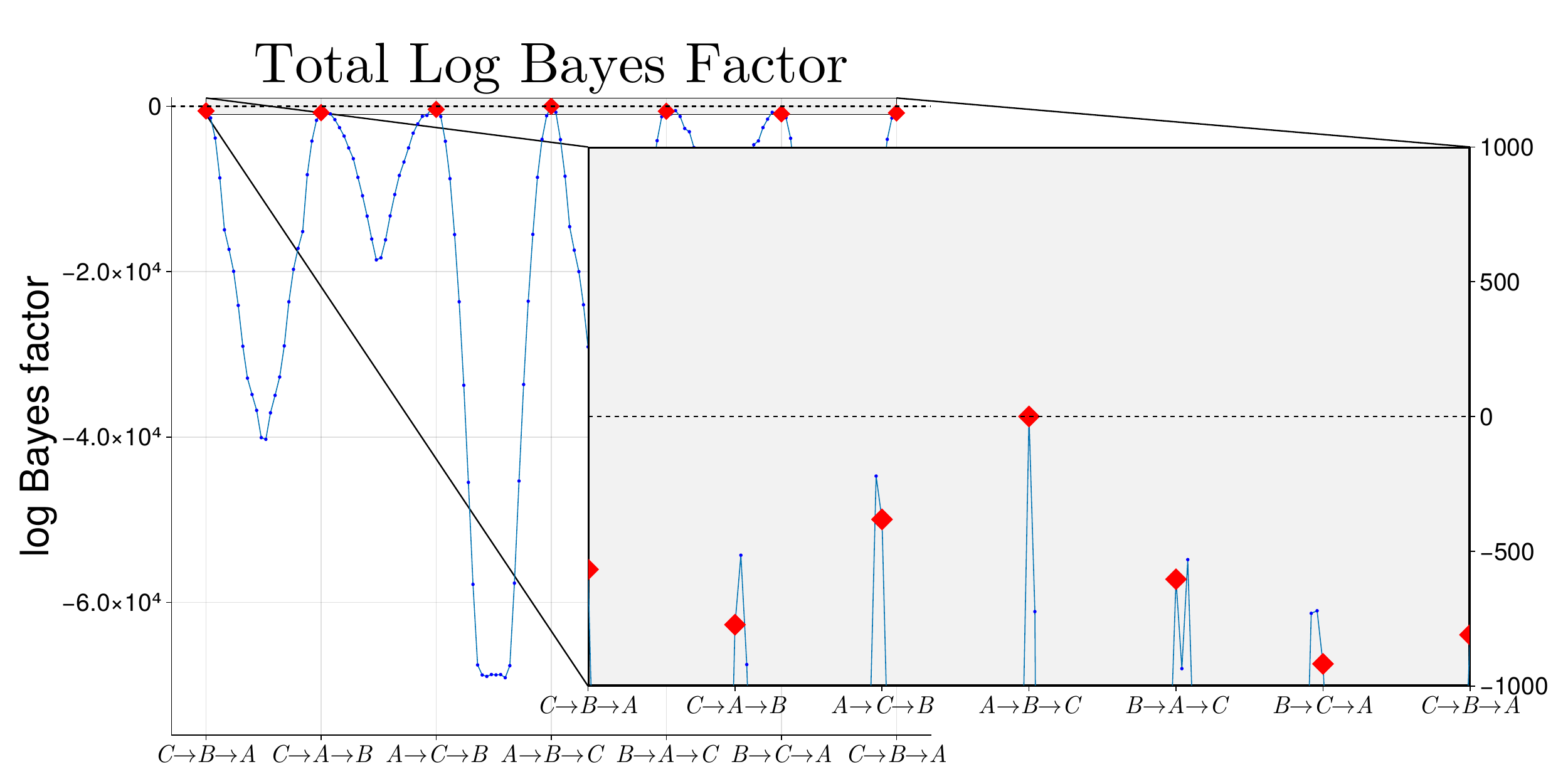} 
\caption{\textbf{Posterior multimodality across permutations of the measurement model.}\\ Each red point corresponds to a permutation of the columns of $\Mb$, with intermediate points obtained by linear interpolations between permutations. The vertical axis shows the log Bayes factor relative to the true ordering. The landscape exhibits multiple local modes, each associated with a distinct causal ordering, while the true ordering ($A\to B\to C$) attains the global optimum. The measurement model is the same as in \cref{fig:ex_multi_domain_data}, and marginal likelihoods are estimated using the method of~\citet{chib1995marginal}.\looseness-1}
\label{fig:bayes_factors}
\end{figure}

We begin by characterizing the posterior landscape induced by permutations of the latent ordering. To make this precise, we define the log Bayes factor of a permutation $\pi$ relative to the true ordering $\pi^\star$ as
\begin{equation}
\label{eq:logBF}
    \mathrm{logBF}(\pi) = \log p(\Xb, \Psib_\pi) - \log p(\Xb, \Psib_{\pi^\star})
\end{equation}
where \(p(\Xb, \Psib_\pi)\) is the marginal likelihood of the data under the measurement model corresponding to permutation \(\pi\).

Figure~\ref{fig:bayes_factors} illustrates these Bayes factors in a setting with $L=3$ binary latent variables and $|E|=16$ environments, each containing $N_e=500$ observations of dimension $D=2$. The environments consist of 4 observational environments with no interventions and 12 single-intervention environments, evenly divided across the $L=3$ latent variables. We denote the three latent factors by $(A,B,C)$ and their corresponding measurement vectors are $(\mb_1,\mb_2,\mb_3)$:
\[
\mb_0 =
\begin{bmatrix}
0\\
0
\end{bmatrix},\quad
\mb_1 =
\begin{bmatrix}
1\\
5
\end{bmatrix},\quad
\mb_2 =
\begin{bmatrix}
3\\
-1
\end{bmatrix},\quad
\mb_3 =
\begin{bmatrix}
-2\\
4
\end{bmatrix}.
\]
This is the same measurement model as in~\cref{fig:ex_multi_domain_data}; all remaining components are drawn from the priors specified in~\Cref{sec:model}.\looseness-1

These Bayes factors show the posterior landscape to contain multiple well separated, non-identical modes, with the mode corresponding to the true ordering ($A\to B\to C$) attaining the global optimum. In many latent variable models, such multimodality arises from nuisance symmetries, where multiple posterior modes correspond to equivalent representations under relabelings~\citep{jasra2005markov}. Here, however, the multimodality is neither symmetric nor a nuisance: each mode corresponds to a distinct ordering of the latent concepts. This distinction is crucial and motivates our use of SMCS~(\cref{sec:SMCS}), as recovering the latent ordering requires exploring the full multimodal posterior rather than concentrating on a single mode.

To further investigate the components of the model that create and distinguish these modes, we decompose the Bayes factors into contributions from the latent mechanism and the measurement model:\looseness-1
\begin{equation}
\label{eq:logBF_decomp}
\mathrm{logBF}(\pi) 
= \underbrace{\log p(\Zb_\pi , \Phib_\pi) - \log p(\Zb_{\pi^\star} , \Phib_{\pi^\star})}_{\text{latent mechanism contribution}}
+ \underbrace{\log p(\Xb,  \Psib_\pi \mid \Zb_\pi) - \log p(\Xb, \Psib_{\pi^\star} \mid \Zb_{\pi^\star} )}_{\text{measurement contribution}}.
\end{equation}

Figure~\ref{fig:bayes_factors_decomp} illustrates these two components and shows how the multimodal posterior structure arises from their interaction: the measurement model induces the modes, while the latent mechanism prior differentiates among them.

The measurement contribution (left panel of Figure~\ref{fig:bayes_factors_decomp}) is sharply and uniformly peaked at configurations corresponding to permutations of the true measurement model. It assigns high probability to these configurations and rapidly decays away from them, but does not distinguish between different permutations within this set. This symmetry is expected and necessary, as discussed in \cref{sec:model_permutation}, and consequently the measurement model induces a collection of symmetric modes without differentiation between them.

In contrast, the latent mechanism contribution (right panel of Figure~\ref{fig:bayes_factors_decomp}) does not induce such sharp peaks, but instead differentiates among them. Its contribution varies across permutations and assigns the highest probability to the true causal ordering ($A\to B\to C$), thereby selecting among the modes created by the measurement model. This behavior reflects the role of this component in encoding our priors over latent causal structure, as described in \cref{sec:model_base_mechanisms,sec:model_interventions,sec:model_intervened_mechanisms}. In regions between the permutation modes, this prior takes relatively favorable values even when the measurement model fit is poor. In these regions, the latent variables and parameters are effectively free to adjust in ways that optimize the prior structure without being constrained by the observed data.

Taken together, these effects create a posterior landscape characterized by sharp, well-separated peaks induced by the measurement model, with their relative evidence determined by the latent mechanism prior. This interplay explains both the multimodality of the posterior and the necessity of exploring all modes.

\begin{figure}[tbp]
\centering
\includegraphics[width = 0.5\textwidth]{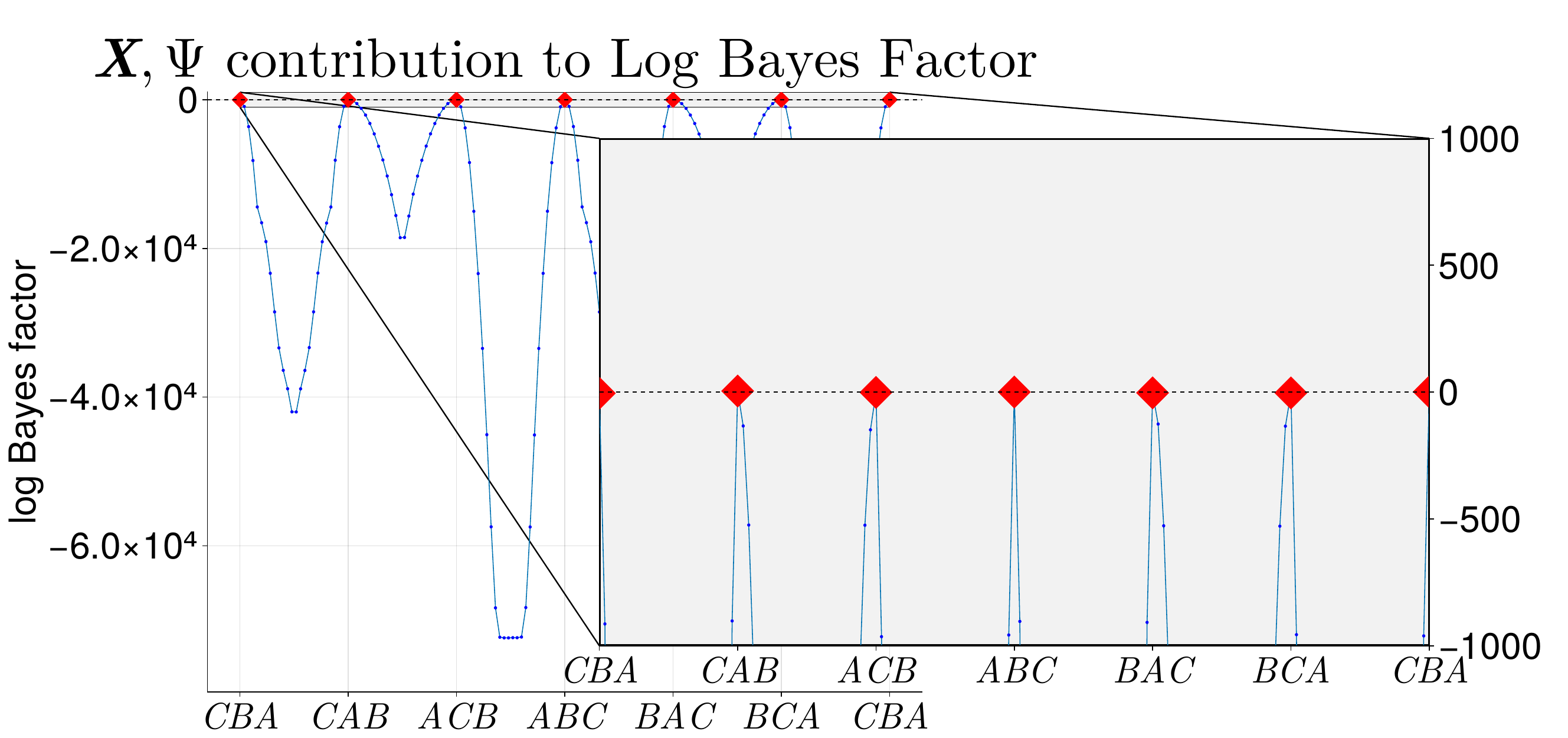}%
\includegraphics[width = 0.5\textwidth]{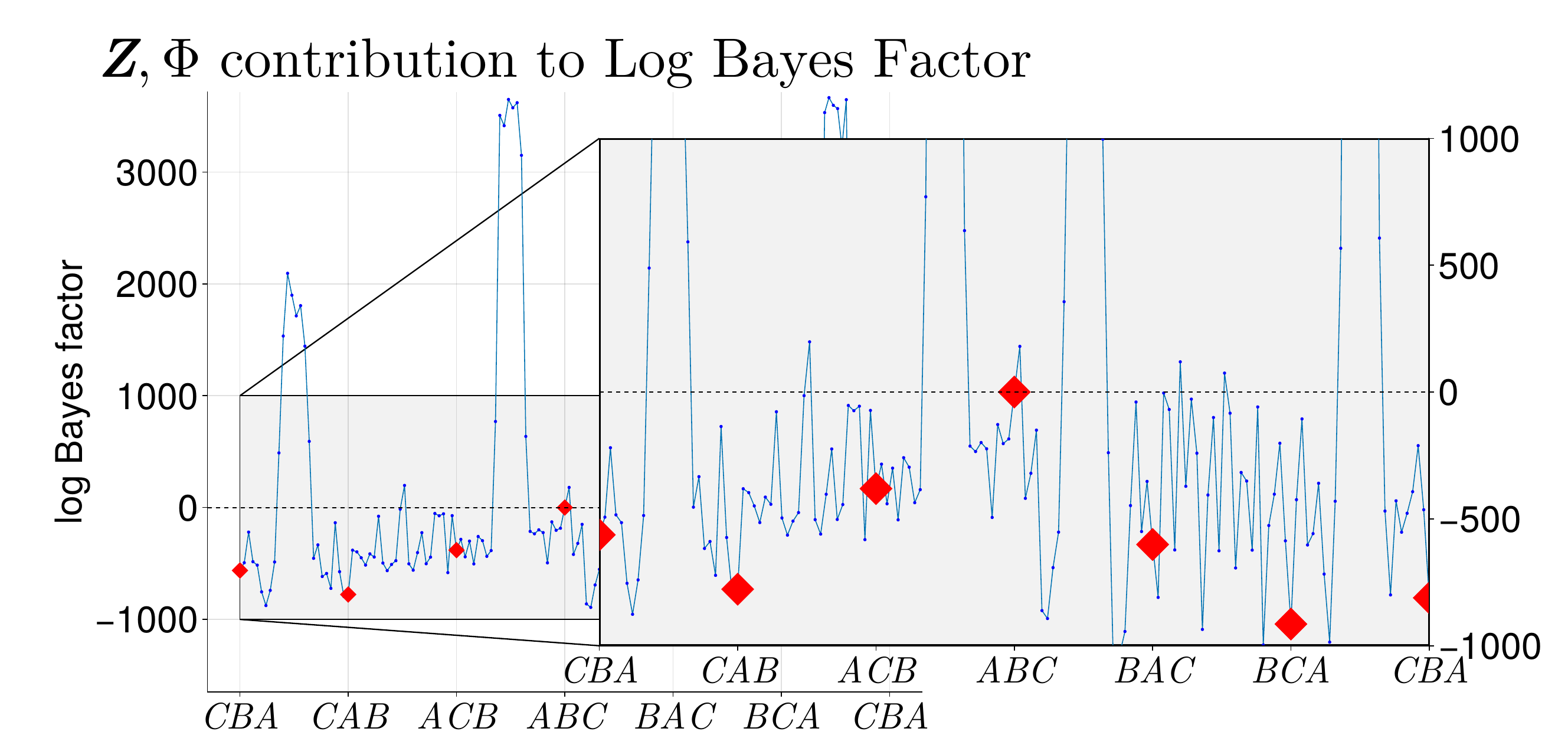} 
\caption{\textbf{Decomposition of the log Bayes factors.}
Each panel shows the contribution of a model component to the evidence difference relative to the true ordering.
The left panel shows the measurement model, 
$\log p(\Xb,\Psib_\pi \mid \Zb_\pi)$, 
which produces multiple posterior modes across permutations. 
The right panel shows the latent mechanisms, 
$\log p(\Zb_\pi,\Phib_\pi)$),
which distinguishes among these modes, favoring the true ordering.
The sum of these two terms yields the total log Bayes factor shown in Figure~\ref{fig:bayes_factors}. \looseness-1}
\label{fig:bayes_factors_decomp}
\end{figure}

\subsection{Empirical Comparison with Alternative Inference Approaches}
\label{sec:evaluation}
We compare our method against alternative inference schemes and model ablations.
For inference methods, we compare our Sequential Monte Carlo Sampling (SMCS) approach against several alternative inference strategies: a standard Gibbs sampler; Parallel Tempering (PT), which augments Gibbs sampling with multiple chains at different temperatures; and Variational Inference (VI), which approximates the posterior using a tractable family optimized via the evidence lower bound (ELBO).

\looseness-1 For model ablations, we investigate the contribution of two key components of our approach: the KL-based shift prior and the sparse intervention prior. Specifically, we consider a variant that removes the KL truncation (SMCS (no KL)) and a variant that fixes all intervention target indicators to $1$ (SMCS ($\mathcal{I} = 1$))
such that every environment is fitted using a separate joint distribution over~$\Zb$%
. These ablations highlight how key modeling choices affect both predictive performance and recovery of the underlying causal structure.

\subsubsection{Experimental Setup} 

We generate $|\Ecal| = 12$ environments with $L=3$ binary (i.e., $K=2$) latent variables, each with $N^e = 500$ observations. These environments comprise three observational environments with no interventions and nine interventional environments with a single intervention each, evenly divided among the $L=3$ latent causal variables.
The measurements $\xb$ have dimension $D=9$. The measurement model is generated with a block structure designed to mimic the structure observed in our real data. Specifically, the observed variables are partitioned into three equal blocks, with each block primarily associated with one latent variable. All entries of $\Mb$ are drawn independently from centered Gaussians, the off-block measurements with $\sigma = 5$ and the targeted blocks with $\sigma \approx 30$.
The remaining components of the dataset are from the priors as specified at the end of~\cref{sec:complete_model}.
Each method is evaluated on the same set of ten independent draws of multi-domain data.

\subsubsection{Evaluation Metrics}
We assess recovery of the intervention structure using the average difference between the learned intervention indicators 
$\mathcal{I}^\Ecal$ and the ground truth $\mathcal{I}^*$,

\[
\text{Intervention Error} = 
\frac{1}{L|\Ecal|}\sum_{e,\ell}\abs{\Ical^e_\ell-(\Ical^*)^e_\ell},
\]
which represents the fraction of intervention indicators that are incorrectly inferred.

To evaluate the measurement model, we report two complementary metrics:
\begin{align*}
\text{Measurement Error} &= 
\frac{1}{(L+1)\,D}\sum_{\ell,d}\abs{ (\Mb- \Mb^*)_{d,\ell}},
\\
\text{Aligned Measurement Error} &= 
\frac{1}{(L+1)\,D} \sum_{\ell,d} \abs{ (\Mb_{\pi^*}- \Mb^*)_{d,\ell}} ,
\end{align*}
where $\pi^*$ is the optimal label alignment relative to ground truth,
\(\pi^* \in \arg\min_\pi \norm{\Mb_\pi - \Mb^*}_{1,1}.\)

The unaligned Measurement Error reflects both recovery of the measurement model and identification of the correct latent ordering. In contrast, the Aligned Measurement Error isolates recovery of the measurement model by minimizing over label permutations. In symmetric mixture models, it is common to align samples to a consistent permutation~\citep{stephens2000labels}. Although label permutations in our model are not symmetric, the Aligned Measurement Error is useful as a diagnostic: it separates whether a method has learned the correct measurement structure from whether it has also recovered the correct latent ordering.\footnote{The model exhibits a sign-flip symmetry in the latent representation and measurement model parameters. This is truly a nuisance symmetry, and we align posterior samples to a consistent sign when computing all reported metrics in this paper (see \Cref{app:metrics} for details).}

In addition, we report the minimum and maximum runtime over the 10 tests and the log posterior predictive density~\citep[LPPD;][]{gelman2013understandingpredictiveinformationcriteria}
\[
\mathrm{LPPD}
=
\frac{1}{\sum_{e\in\Ecal}N^e}
\sum_{e \in \mathcal{E}}
\sum_{j \in [N_e]}
\log \,
\mathbb{E}_{( \boldsymbol{\Phi}, \boldsymbol{\Psi}) \sim p_{\mathrm{post}}}
\!\left[
p\!\left(
\mathbf{x}^{e,j}\mid \boldsymbol{\Phi},\boldsymbol{\Psi}
\right)
\right].
\]
Together, these metrics separate latent structure recovery from predictive fit. Intervention Error and (unaligned) Measurement Error evaluate recovery of the latent causal structure, while Aligned Measurement Error and LPPD evaluate measurement recovery and predictive performance. For sampling-based methods, we use posterior means for all estimates, see \Cref{app:metrics} for exact details.

\begin{table}[tbp]
\centering
\caption{
\looseness-1 \textbf{Comparison of inference methods.} Results are averaged over ten draws from the generative process. Lower values indicate better performance for Intervention (Int.) and Measurement (Meas.) Errors; higher values indicate better predictive fit for log posterior predictive density (LPPD).
SMCS more reliably recovers the latent causal structure, achieving lower intervention and measurement error. The ablation variants SMCS (no KL) and SMCS ($\Ical = 1$) remove the KL-based truncation prior~(\cref{sec:shift_priors}) and the intervention sparsity~(\cref{sec:model_interventions}), respectively. These variants improve predictive fit but lead to substantially worse recovery of the latent structure. %
{\footnotesize$^\dagger$50K total MCMC steps. $^\ddagger$2K iterations.} 
}
\label{tab:inference_comparison}
\resizebox{\textwidth}{!}{
\begin{tabular}{lcccrc}
\toprule
\textbf{Method} & \textbf{Int.\ Err.} ($\downarrow$) & \textbf{Meas.\ Err.}($\downarrow$) & \textbf{Aligned Meas.\ Err.} ($\downarrow$) & \textbf{LPPD} ($\uparrow$) & \textbf{Runtime} ($\downarrow$) \\
\midrule
SMCS$^\dagger$  & 
\bf{0.260} & \bf{8.45}
&4.95 & -13.3 &   10--20min \\ 
Gibbs$^\dagger$ 
& 0.341 &9.77
&\bf{4.41} & -19.1 &  20--30min \\
PT$^\dagger$    
& 0.344 & 9.61
&4.90 & -19.5 & 15--45min \\
VI$^\ddagger$ 
& 0.606 & 9.03 
&9.00 & -1078.0 & \bf{40--50sec} \\ 
\midrule
SMCS (no KL)$^\dagger$  
& 0.547 & 10.51
& 8.64 & -12.0 &   10--20min \\ 
SMCS ($\Ical = 1$)$^\dagger$  
& - & 10.89
& 7.19 & \bf{-10.5}&   10--20min \\ 
\bottomrule
\end{tabular}
}
\end{table}

\subsubsection{Results}
\looseness-1 
The results are summarized in~\cref{tab:inference_comparison}.
SMCS most reliably recovers the latent causal structure. It achieves the lowest intervention error and %
unaligned measurement error, indicating improved recovery of both the intervention structure and the latent ordering. The aligned measurement errors are relatively similar across Gibbs, PT, and SMCS. This suggests that these methods can often find a good mode of the measurement model. However, SMCS achieves substantially lower unaligned measurement error, indicating that it more consistently identifies the correct latent ordering. Consequently, we use SMCS for all subsequent experiment%
s.

The ablation results reveal a tradeoff between predictive fit and causal structure recovery. Removing the KL truncation or intervention sparsity improves LPPD, but substantially worsens recovery of the latent structure, as reflected in the Intervention and Measurement Errors. This suggests that, without the KL and sparsity constraints, the model can explain the observed data by introducing additional interventions, but this comes at the cost of not learning the true latent causal structure.

\section{Case Study I: World Values Survey}
\label{sec:case_study_wvs}
\looseness-1 We apply our model to CRL tasks motivated by survey analysis. In this domain, many important factors such as partisanship, political beliefs, or cultural values are inherently latent and cannot be observed directly. Instead, they are measured indirectly through proxies, such as responses to survey questions. A key question is how these latent views influence one another:
for example, what are the major dimensions of cross cultural variation and what causal relationships exist between different cultural values? 
Because direct interventions on these latent values are not experimentally feasible, methods that leverage %
data across multiple environments may be an appealing way to study these relationships.

We begin with the World Values Survey (WVS), a large-scale cross-national survey designed to measure people’s values, beliefs, and norms in a comparative cross-national perspective. The WVS covers topics such as democracy, religion, gender equality, and societal trust across nearly $100$ countries. We focus on the most recent complete round conducted from $2017$ to $2022$ \citep{haerpfer2022world}. 

\subsection{Modeling}
For this analysis, we select a subset of $D = 37$ questions that were answered by at least $99\%$ of respondents, and treat each country as a separate environment. Respondents with missing answers to these questions were excluded from the analysis, yielding between $382$ and $4,018$ responses per environment, for a total of $\sum_{e \in \mathcal{E}} N^e = 85,851$ responses.

We fit our model with $L = 3$ latent variables, each taking $K = 3$ discrete states. Posterior inference is performed using SMCS with $1,000$ particles and a temperature schedule consisting of 100 geometrically spaced temperatures ranging from 1 to the total number of responses ($\sum_{e \in \mathcal{E}} N^e$). The full model fitting procedure required approximately 9 hours on a multi-core compute cluster.

Since we do not have ground truth for the real world data, in contrast to~\cref{sec:synthetic_experiments}, we focus on qualitative analysis and interpretation of the results. To that end, we examine three aspects of the fitted model: 1) the learned measurement model to ascribe semantic meaning to the causal latent variables, 2) the causal influence among latents, to understand how these learned concepts causally relate to one another, and 3) the cross-country variation in latent concepts as well as the environment-specific interventions underlying these differences. 

\begin{figure}[tbp]
\centering
\includegraphics[width=\textwidth]{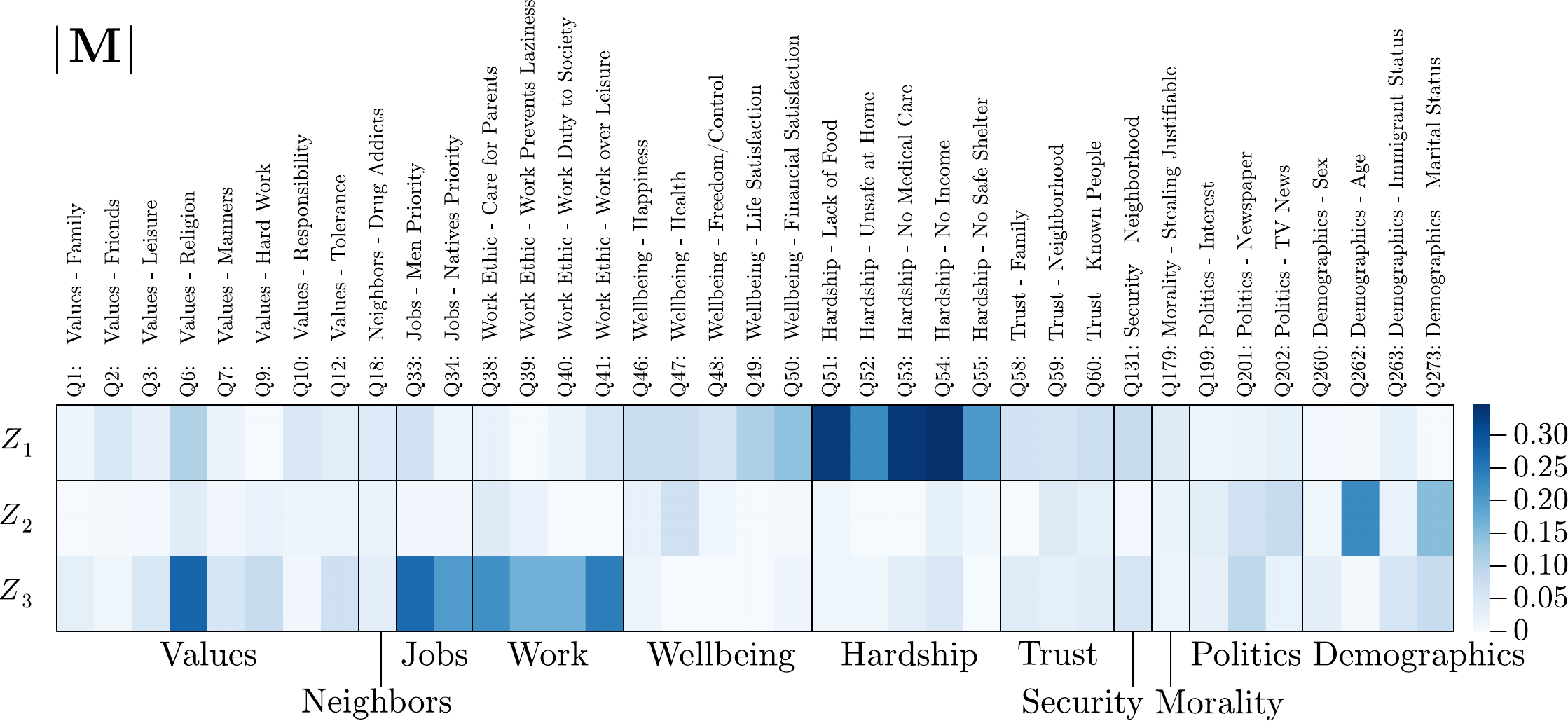}
\caption{
\textbf{Estimated measurement mapping from latent variables to WVS survey responses.}
Heatmap of the normalized posterior mean absolute measurement matrix, $|\Mb^\mathsf{T}|$.
The learned block structure is interpretable: $Z_1$ aligns with \emph{economic hardship}, $Z_2$ with \emph{demographics}, specifically age and marital status, and $Z_3$ with \emph{cultural conservatism}.
}
\label{fig:wvs_M}
\end{figure}

\subsection{Interpretation of Inferred Causal Concepts}
We organize our interpretation in two steps: first, we use the measurement model to assign semantic meaning to the inferred latent variables, and second, we examine how these latent concepts vary across countries to confirm that they are consistent with expected patterns.
\subsubsection{Interpreting the Measurement Model}
In this case study, our model associates three latent variables $(Z_1, Z_2, Z_3)$ with each survey respondent. To interpret their possible correspondence to known domain-specfic concepts, we first examine the posterior mean of the measurement matrix $\Mb$, which we visualize as a heatmap in \Cref{fig:wvs_M}. Recall that the $\ell$\textsuperscript{th} column $\textbf{m}_\ell \in \mathbb{R}^D$ of the measurement matrix specifies how the $\ell$\textsuperscript{th} latent variable $Z_\ell$ %
affects the mean of a given respondent's survey responses $\mathbf{X} \in \mathbb{R}^D$.\footnote{Details on how we deal with the ``label switching'' problem~\citep{jasra2005markov} can be found in \Cref{app:metrics}.}

 Since the World Value Survey groups questions into distinct topics, 
 we can check whether the inferred measurement matrix aligns with the block structure corresponding to the underlying topics (which is unknown to the model). 
 As shown in~\Cref{fig:wvs_M}, a clear block structure, which is consistent with the known topic structure, emerges. The first latent variable $Z_1$ loads strongly onto items in the \emph{hardship} topic, while $Z_2$ loads onto two items in the \emph{demographics} topic, specifically \textsc{Q262 Age} and \textsc{Q263 Marital Status}, and $Z_3$ loads onto the \emph{work ethic} and \emph{jobs} topics, as well as onto one single item in the \emph{values} topic, namely \textsc{Q6 Religion}. In grouping items relating to work ethic, the role of men in the workplace, and religion, $Z_3$ appears to be capturing a broader concept of religious, social, and/or cultural \textit{conservatism}. For clarity, we adopt the following descriptive labels below and from hereon refer to $(Z_1,Z_2,Z_3)$ as $(Z_{\mathrm{hard}}, Z_{\mathrm{demo}}, Z_{\mathrm{cons}})$:
\begin{align*}
Z_1 \mapsto \text{economic hardship}, \qquad 
Z_2  \mapsto  \text{demographics} , \qquad 
Z_3 \mapsto  \text{cultural conservatism}
\end{align*}
Interestingly, the inferred concepts $Z_{\mathrm{cons}}$ and $Z_{\mathrm{hard}}$ %
align well with the two axes of the Inglehart–Welzel map~\citep{inglehart2005world}: %
\emph{traditional versus secular values}, measuring the importance of religion, family values, and nationalism; and 
\emph{survival versus self-expression values}, measuring the relative desire for economic and physical well-being against the desire for diversity and pluralism. The Inglehart-Welzel map has been widely validated within political science for explaining cross-cultural variation in self-reported attitudes and values, so it is encouraging that the model infers causal concepts aligned with these axes. In additional fits with \(L=2\), we found a similar structure, but with the demographic and conservatism dimensions merged into a single shared latent factor. We emphasize that the model inferred these concepts entirely unsupervised from the data without access to the survey's pre-specified topics or any pre-imposed block structure. 
\subsubsection{Distribution of Latent Concepts across Environments}
We further explore this interpretation by examining the distribution of $(Z_{\mathrm{hard}}, Z_{\mathrm{demo}}, Z_{\mathrm{cons}})$ across countries. Specifically, for each country (i.e., environment) $e \in \Ecal$, we compute the marginal mean of each latent concept:
\begin{equation}
\mathbb{E}_{\Zb  \sim p^e(\Zb)}[Z_\ell]
\end{equation}

Throughout this analysis, we refer to the environment with no interventions as the baseline environment. This baseline does not correspond to any observed country. However,
it is still well defined in our model: the baseline distribution of $\Zb$ is obtained
by setting all interventional shifts to zero and depends only on
the shared parameters $\Thetab$.

\looseness-1 In \Cref{fig:wvs_Z}, we visualize %
the raw country-level means and their
differences from this baseline. The patterns %
align with prior expectations based both on global socioeconomic and cultural differences in general and on the Inglehart–Welzel map in particular. For example, Western countries tend to cluster in the low economic hardship, low conservatism corner, while Islamic-African countries exhibit higher levels of cultural conservatism. Latin American countries span a range of economic hardship values but generally remain near moderate conservatism levels. Notably, Venezuela appears at the extreme end of economic hardship in the 2021 survey wave, aligning with its economic crisis during this period.

The results up to here are driven by the measurement model, and the same structure was found without imposing interventional sparsity (i.e., by setting all interventions to be active as in the ablation from \cref{sec:evaluation}). While these priors do not affect the learned concepts, they do allow us to interpret the causal graph, which would otherwise be completely arbitrary.

\begin{figure}[p]
\centering
\includegraphics[width=0.99\textwidth]{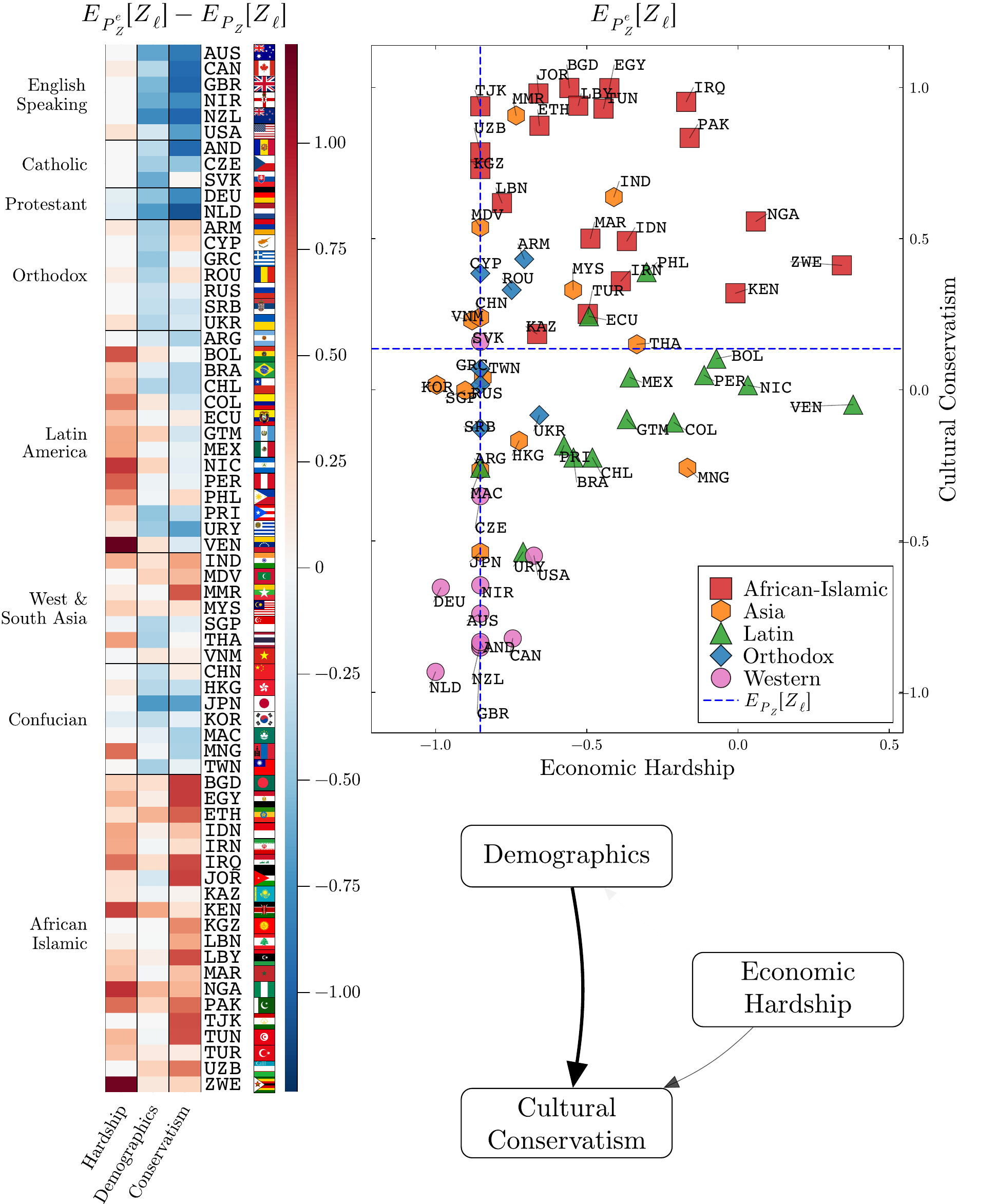}
\caption{
\textbf{Marginal distributions and baseline causal influence of latent concepts.}\\
\textbf{Left:} Country-specific deviations from baseline for each latent concept: $\mathbb{E}_{\Zb  \sim p^e(\Zb)}[Z_\ell] - \mathbb{E}_{\Zb  \sim p(\Zb)}[Z_\ell].$\looseness-1
\\
\textbf{Top Right:} Scatterplot of $Z_{\mathrm{hard}}$ versus $Z_{\mathrm{cons}}$ colored by region. 
We adopt the regional divisions of the Inglehart--Welzel cultural map~\citep{inglehart2005world}, but combine \{English-speaking, Protestant Europe, Catholic Europe\} into ``Western'', and \{Confucian, West \& South Asia\} into ``Asia''.\\
\textbf{Lower Right:} Estimated causal influence among latent concepts. Edges are proportional to causal \\influence, showing strong influences of $\{Z_{\mathrm{hard}}, Z_{\mathrm{demo}}\} \to  Z_{\mathrm{cons}}$ but negligible influence $Z_{\mathrm{hard}} \to  Z_{\mathrm{demo}}$.\looseness-1
}
\label{fig:wvs_Z}
\end{figure}

\subsection{Causal Graph Recovery} 
Having established an interpretable latent space, we can now analyze the causal structure among these latent concepts. The ordering of $\Zb$ gives the causal ordering of latent concepts as \emph{economic hardship} $\preceq$ \emph{demographics} $\preceq$ \emph{cultural conservatism}.
 
As our model fits a fully connected DAG, we apply a post-hoc method to evaluate the strength of individual dependencies. In particular, we are interested in quantifying how much a parent variable $Z_j$ impacts a child variable $Z_\ell$ beyond the effect of $Z_\ell$’s other parents.   For this purpose, we adopt the notion of \emph{causal influence} proposed by \citet{janzing2013quantifying}, which measures the change in the joint distribution when the direct influence of $Z_j$ on $Z_\ell$ is removed:
\begin{equation}
    \operatorname{CI}^{j\to \ell}_{p}:=D_\textsc{kl}\big(p(\Zb)~\big\|~ p^{j\to \ell}(\Zb)\big), \quad \mbox{where} \quad p^{j\to \ell}\big(z_\ell~|~\pa_\ell\setminus z_j\big)
    =
    \int p\big(z_\ell~|~\pa_\ell \big) p(z_j) d z_j
    \vspace{-0.25em}
\end{equation}
\looseness-1
and $p^{j\to \ell}(\Zb)$ is the interventional distribution arising from replacing $p\big(z_\ell~|~\pa_\ell\big)$ by $p^{j\to \ell}\big(z_\ell~|~\pa_\ell\setminus z_j\big)$.

\Cref{fig:wvs_Z} shows the estimated dependencies in the baseline environment, where no interventions are applied. We see that economic hardship ($Z_\mathrm{hard}$)
and demographic structure ($Z_\mathrm{demo}$) both exert substantial influence on cultural conservatism ($Z_\mathrm{cons}$), while the effect of economic hardship on demographic structure ($Z_\mathrm{hard} \to Z_\mathrm{demo}$) is negligible. This lack of influence on $Z_\mathrm{demo}$ is consistent and expected: demographic attributes are commonly modeled as root nodes in a causal graph \ag{cite?}, and are therefore not assumed to be influenced by downstream socioeconomic conditions. 

This baseline does not provide a complete picture, since environment-specific interventions may strengthen or weaken these dependencies. We therefore compute the same causal-influence measure separately within each country, as shown in \Cref{fig:wvs_CI-E}. Although the overall patterns remain similar to the baseline, the influence of \emph{economic hardship} on \emph{cultural conservatism} and the influence of \emph{demographics} on \emph{cultural conservatism} vary substantially between environments.

\begin{figure}[tbp]
\centering
\includegraphics[width=\textwidth,trim={4.25cm 1.88cm 2.25cm 0},clip]{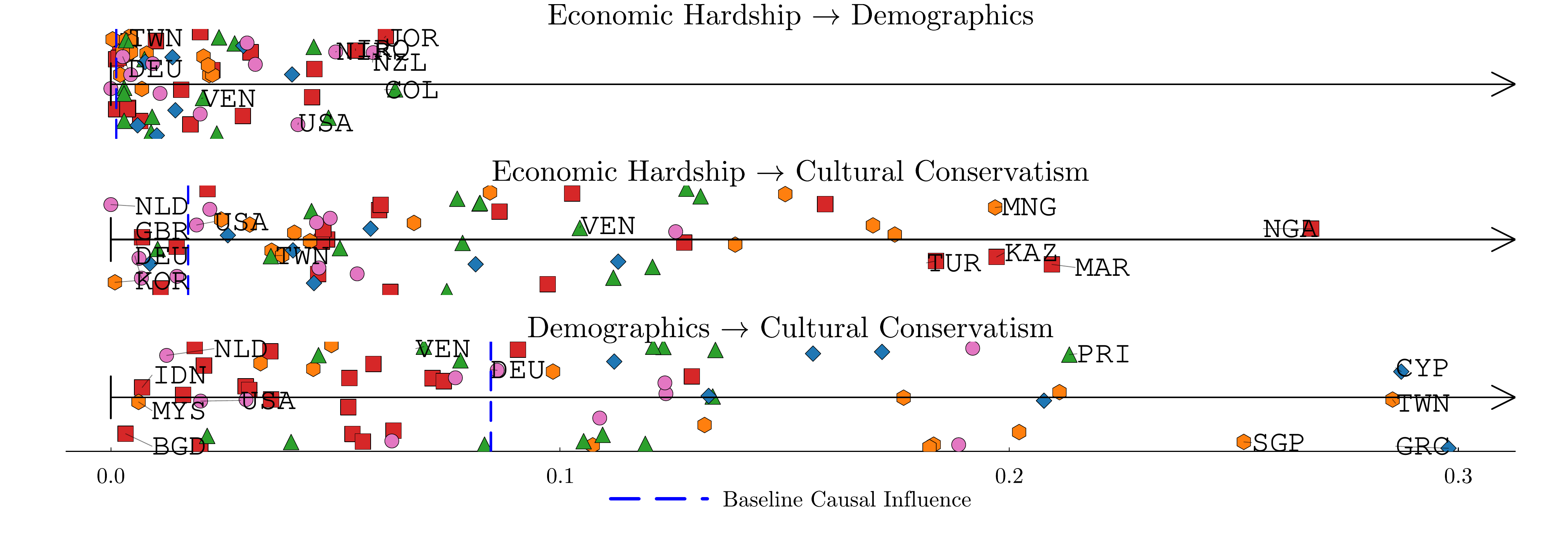}
\caption{
\textbf{Estimated causal influence of latent concepts across environments (countries).}
Each panel shows the environment-specific strength of a single dependency. The scale is shared across panels, revealing that influences of $Z_{\mathrm{demo}} \to Z_{\mathrm{cons}}$ and $Z_{\mathrm{hard}} \to Z_{\mathrm{cons}}$ vary substantially across countries, while the influence of $Z_{\mathrm{hard}} \to Z_{\mathrm{demo}}$ is relatively small across all environments.
}
\label{fig:wvs_CI-E}
\end{figure}

\subsection{Interventions}
In addition to the causal structure among latent concepts, the model also allows us to examine the form of the interventions across environments. 
\Cref{fig:wvs_ID} summarizes both the inferred presence and size of these interventions across countries. We report the posterior mean intervention matrix $\Ical$, where each entry gives the posterior probability that a given latent variable is perturbed in a given environment. This reveals mild sparsity in the inferred intervention structure. Beyond this binary intervention structure, we also examine the expected magnitude of the corresponding perturbations. For latent variable $Z_\ell$ in environment $e$, we measure this magnitude as the KL divergence between the baseline and perturbed conditional mechanisms, averaged over parent configurations:
\begin{equation}
\mathbb{E}\left[\norm{\Deltab^e_\ell}_\mathrm{KL}\right]
:=
\mathbb{E}_{\pa_\ell \sim p^e(\pa_\ell)}
\left[
D_{\mathrm{KL}}\!\left(
p\left(Z_\ell \mid \pa_\ell\right)
\,\middle\|\,
p^e\left(Z_\ell \mid \pa_\ell\right)
\right)
\right].
\end{equation}
This quantity summarizes how strongly an intervention perturbs its associated latent mechanism.  We find that many interventions have relatively small effects, even when the corresponding entries of~$\Ical$ are nonzero, suggesting that the effective sparsity is higher than what is implied by $\Ical$ alone.
\begin{figure}[tbp]
\centering
\includegraphics[width = \linewidth,trim={0 0 2cm 0},clip]{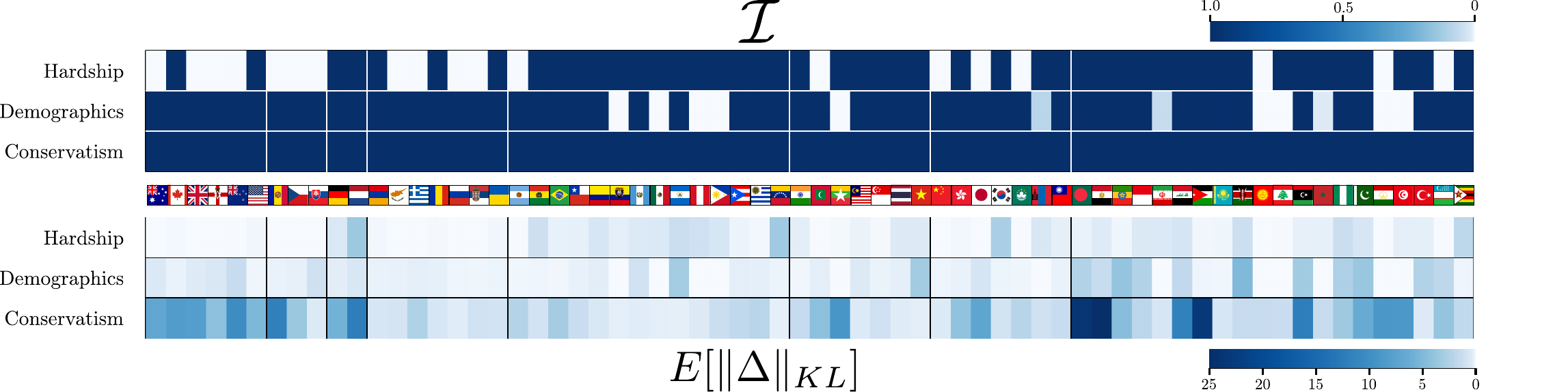} 
\caption{
\textbf{Intervention indicators and magnitudes across environments.}
(Top) Posterior mean intervention matrix $\Ical$, indicating the probability that each latent variable is intervened on in each environment. 
(Bottom) Expected magnitude of the intervention, as measured by the induced shift in $p^e(Z_\ell)$, summarizing the sizes of each environment-specific perturbation.}
\label{fig:wvs_ID}
\end{figure}

We also examine the mechanisms underlying these interventions. Specifically, $\deltab_{\mathrm{cons}}^e(Z_{\mathrm{hard}}, Z_{\mathrm{demo}})$, which describes the direction of the perturbation: that is, how the conditional distribution of $Z_{\mathrm{cons}}$ is shifted relative to the baseline mechanism. Since each parent configuration $(Z_{\mathrm{hard}}, Z_{\mathrm{demo}})$ has its own intervention vector, this analysis allows us to assess whether an intervention acts consistently or varies with the parent state.

\Cref{fig:wvs_Delta_line} presents these vectors projected onto a two-dimensional basis spanning the intervention space. We find that many Western countries, such as the United States, exhibit negative shifts across all levels of economic hardship and demographic structure. Formally, $p^e(Z_{\mathrm{cons}} \mid Z_\mathrm{hard}, Z_\mathrm{demo})$ places more probability mass on lower values of $Z_{\mathrm{cons}}$ than $p(Z_{\mathrm{cons}} \mid Z_\mathrm{hard}, Z_\mathrm{demo})$ for all $(Z_\mathrm{hard}, Z_\mathrm{demo})$. This suggests that, in these countries, the intervention primarily acts as a consistent downward shift in the level of \emph{conservatism}. In contrast, some countries such as Taiwan display more heterogeneous interventions, with both direction and magnitude varying substantially across parent configurations.
\begin{figure}[tbp]
\centering
\includegraphics[width=\textwidth,trim={0 1.75cm 0 0},clip]{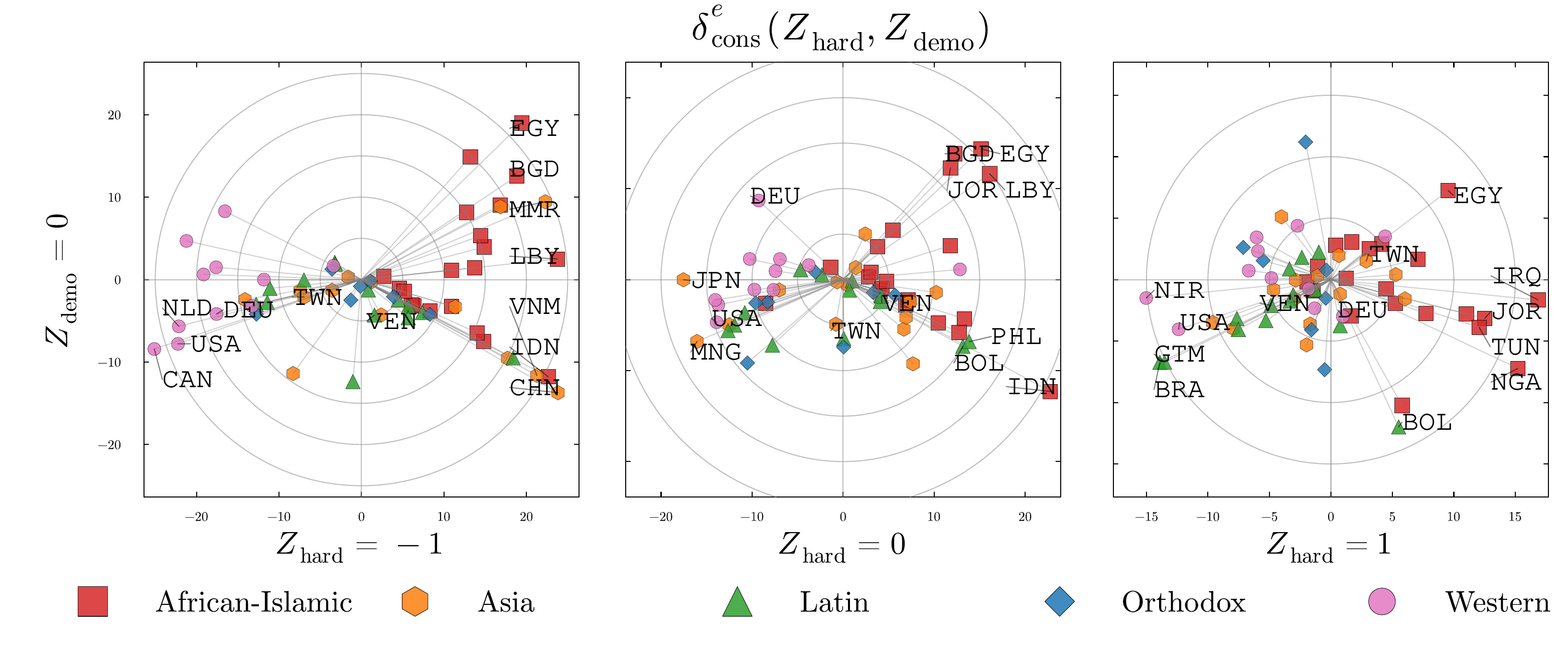}
\caption{\textbf{Intervention directions across parent configurations and environments.}
Posterior mean intervention vectors $\deltab_\mathrm{cons}^e(\pa)$ for each country, shown for $Z_\mathrm{demo}=0$ and all configurations of $Z_\mathrm{hard}$. Each stem represents an environment; its length indicates intervention magnitude and its orientation indicates direction.
 Vectors are expressed in an orthogonal basis for the two-dimensional intervention subspace; the first direction captures shifts in the mean of $Z_{\mathrm{cons}}$ and the second captures changes in dispersion. See \cref{app:metrics_intervention_fig} for the full grid over parent configurations.}\looseness-1
\label{fig:wvs_Delta_line}
\end{figure}

\section{Case Study II: US Political Opinion} 
\label{sec:case_study_us}
As a second empirical test case, we apply our model to American political opinion data, using both real and LLM-generated survey responses. In this setting, we examine how opinions on specific topics and policies interact with one another and with partisan identity. We first analyze real survey data, using geographic and urbanicity-based environments to study the structure naturally present in observed political attitudes. We then construct LLM-generated survey data based on the same survey design, but with controlled interventions, allowing us to investigate a richer causal structure.

\subsection{Real Data}
\label{sec:case_study_us_real}
We begin with a survey dataset of American political opinion collected by the polling firm PredictWise during the period 2017--2022. This dataset contains 75.5k item-level responses from 7.8k U.S.\ respondents collected between 2017 and 2022. Each respondent answers a battery of questions focused on one of five political topic pairs (e.g., elites \& racism or trade \& regulation). The surveys were designed with 7 questions per issue topic and an additional 3 general questions about party and vote choice preferences for a total of $D=17$ questions per survey.
\begin{figure}[tbp]
\centering
\begin{subfigure}{\textwidth}
    \centering
    \includegraphics[width=\textwidth]{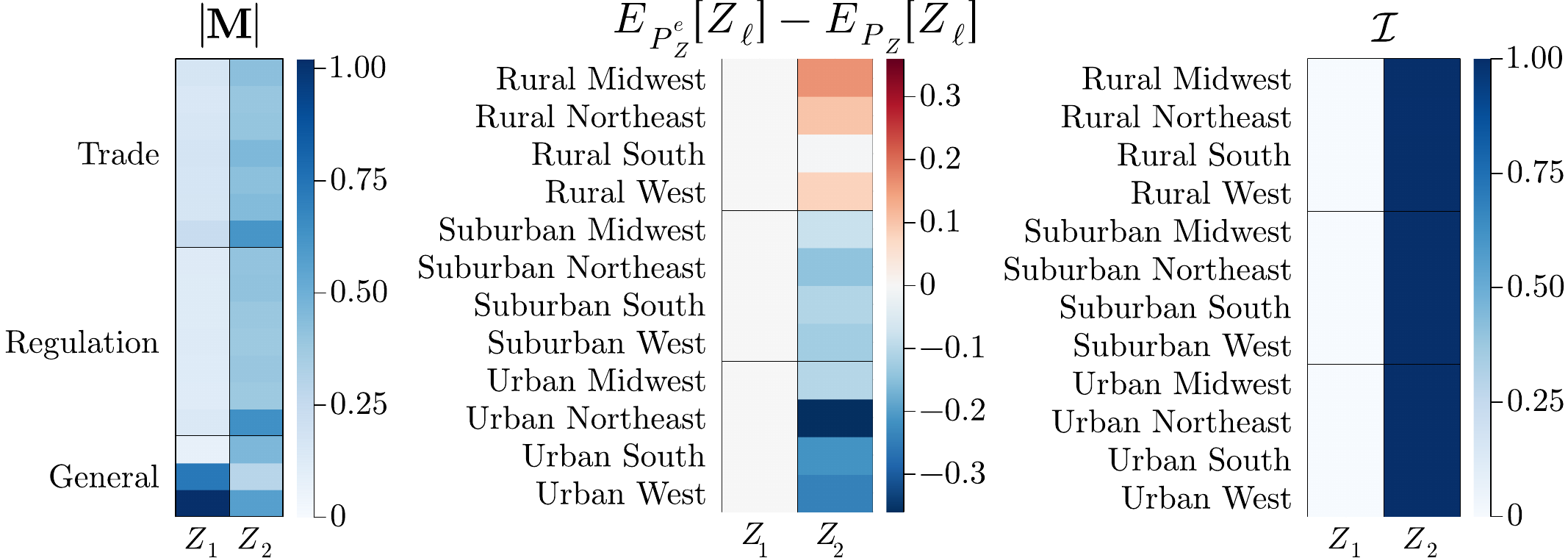}
    \caption{\textbf{Trade and Regulation}.}
    \label{fig:pollfish_TR}
\end{subfigure}
\begin{subfigure}{\textwidth}
    \centering
    \includegraphics[width=\textwidth]{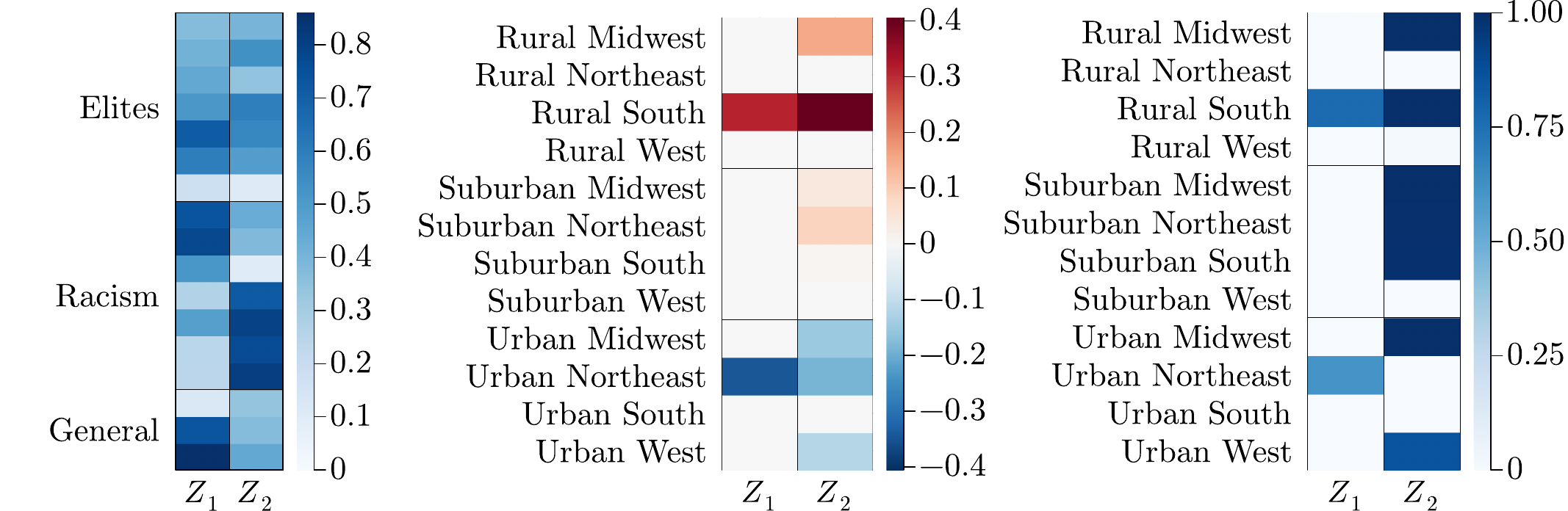}
    \caption{\textbf{Elites and Racism}.}
    \label{fig:pollfish_ER}
\end{subfigure}
\caption{
\textbf{Posterior summaries for the US political opinion surveys.}
Each panel shows the learned measurement model ($\Mb$), intervention probabilities ($\Ical$), and environment-specific deviations from baseline for each latent variable ($\mathbb{E}_{\Zb  \sim p^e(\Zb)}[Z_\ell] - \mathbb{E}_{\Zb  \sim p(\Zb)}[Z_\ell]$). 
The \emph{Trade and Regulation} results (\subref{fig:pollfish_TR}) are largely one-dimensional, separating rural--urban environments along a left--right political gradient. The \emph{Elites and Racism} results (\subref{fig:pollfish_ER}) show a similar rural--urban gradient, but with less consistent variation and more pronounced deviations in the rural South and urban Northeast.
}
\label{fig:pollfish}
\end{figure}
\subsubsection{Modeling}
For this analysis, we look at two of these surveys separately, one on Trade and Regulation and another on Elites and Racism. We partition the data into 12 environments based on the Cartesian product of the 4 Census Regions (West, Midwest, South, and Northeast) and 3 urbanicity levels (Rural, Suburban, and Urban) and fit our model with $L = 2$ with two latent variables, each with cardinality $K =5$.

\subsubsection{Results}
Compared to the WVS analysis, the differences between environments are less nuanced here. The model effectively identifies a one-dimensional latent structure that explains most of the variation across environments.  This is consistent with widely documented trends in U.S.\ politics, where both congressional voting patterns~\citep{poole2000congress} and issue opinions among the public~\citep{green2002partisan,mason2018one} tend to collapse along a single partisan axis.

Although the structure learned from this dataset is rather limited, the patterns that are recovered align with well-documented political trends.  For the \emph{Trade and Regulation} survey~(\cref{fig:pollfish_TR}), responses exhibit a clear urbanicity gradient: respondents in rural environments are shifted toward more right-leaning latent positions, those in urban environments toward more left-leaning positions, and suburban respondents lie in between. For the \emph{Elites and Racism} survey~(\cref{fig:pollfish_ER}), the structure is more intricate. Although an urban--rural divide remains, the largest shifts are concentrated in particular regions: the rural South and the urban Northeast. Overall, these results
reflect the largely one-dimensional nature of U.S.\ public opinion and provide a baseline for how the model behaves in low-signal regimes; it recovers simple but well-known trends with more heterogeneous data required to uncover non-trivial structure.\looseness-1

\subsection{LLM-Generated ``Interventional'' Survey Data}
\label{sec:LLM_data}
To further test our methodology, we next consider a synthetic setting using LLM-generated survey responses. We generate data that mirrors the structure of the original questionnaire while replacing natural environments with explicit interventions designed to satisfy the sparse-mechanism-shift assumption. Unlike the preceding analyses, which study causal structure directly from real survey responses, this experiment studies the causal structure implicit in the LLM’s model of political opinions. As such, our goal is not to claim that these synthetic responses reveal the true causal structure of human political opinions. Rather, we use them as a controlled benchmark with more nuanced latent causal structure than is available in the real survey responses. 

This setting is useful for two reasons. First, because the data are generated through a controlled intervention pipeline, we have access to a form of ground truth about the causal structure against which we can compare our results. Second, the synthetic design produces well-structured cross-environment variation, allowing us to evaluate our method in a richer setting with five latent causal variables and a non-trivial causal graph. This is substantially more complex than the preceding experiments, which involved only two or three latent variables.

\subsubsection{Data generation Process}
\looseness-1 To produce synthetic data with known interventions, we employ a four-stage LLM-driven pipeline. Since LLMs have been shown to generate representative survey responses \citep{argyle2023out}, they provide a natural way to simulate realistic responses under controlled interventions on latent political concepts. Our pipeline uses this capability to construct synthetic datasets with specified interventions: starting from demographic samples, it extrapolates detailed respondent profiles, applies specific interventions, and finally generates survey responses consistent with the intervened profiles.
\begin{enumerate}[itemsep=0.em, topsep=0.5em]
    \item[(i)]\textbf{Sample demographic data from real survey respondents (PredictWise Data).}
    Provides a realistic distribution of demographic variables such as age, education, income,~etc.
    \item[(ii)] \textbf{Expand demographic data into detailed respondent profiles (LLM).}
    Expands the raw demographic data into detailed profile including their beliefs, opinions and reasons behind them. This increases variability leading to more diverse responses representative of differences between individuals not captured by the demographic data from step (i) alone.\looseness-1
    \item[(iii)] \textbf{Apply causal interventions to respondent profiles (LLM).}
     Modifies the profile with a specified intervention and updates causally downstream beliefs and reasoning accordingly.
    \item[(iv)] \textbf{Generate full survey responses based on intervened profiles (LLM).}
\end{enumerate}
We found that, compared to a one-step approach of directly prompting the LLM to “answer as [X],” this multi-stage approach reduced causal leakage and mode collapse, in which responses under a given intervention tended to concentrate on a small number of nearly identical answer patterns. An illustration of the full pipeline is shown in \cref{fig:llm_3step}; further details, including the exact prompts used and samples from each step, can be found in the repository with our code and datasets.

\begin{figure}[tbp]
\centering
    \begin{tikzpicture}
        \tikzstyle{box} = [rectangle, rounded corners, 
        minimum width=3cm, 
        minimum height=1cm,
        text centered, 
        draw=black, 
        fill=black!10]

        \node[draw=none] (start) {};
        \node[box, right=2cm of start, align=center] (box1) {Demographic Data\\{\tiny male, >54, married,...}};
        \node[box, right=2cm of box1, align=center] (box2) {Detailed Profiles\\{\tiny John Smith, Former Dental Technician ...}};

        \node[box,fill=green!10, below= 0.75cm of box2,xshift=1cm, align=center] (box3c) {Interventional Profiles\\{\tiny John Smith, Pro-Immigration, ...}};

        \node[box,fill=red!10, left= 1cm of box3c, align=center] (box3b) {Interventional Profiles\\{\tiny John Smith, Republican, ...}};

        \node[box,fill=blue!10, left= 1cm of box3b, align=center] (box3a) {Interventional Profiles\\{\tiny John Smith, Democrat, ...}};

        \draw[->] (start) -- (box1) node[midway, above] {(i)}
        node[midway, below] {\shortstack{PredictWise\\Data}};

        \draw[->] (box1) -- (box2) node[midway, above] {(ii)} node[midway, below] {LLM};
        \draw[->] (box2) -- (box3a) node[midway, left] {};
        \draw[->] (box2) -- (box3b) node[midway, left] {};
        \draw[->] (box2) -- (box3c) node[midway, left] {(iii)}
        node[midway, right] {LLM};

        \node[box,fill=red!10, below= 0.75cm of box3b, align=center] (box4b) {Survey Responses\\{\tiny X1:10, \ X2:85, \ X3:55, ...}};
        \node[box,fill=blue!10, below= 0.75cm of box3a, align=center] (box4a) {Survey Responses\\{\tiny X1:10, \ X2:25, \ X3:60, ...}};
        \node[box,fill=green!10, below= 0.75cm of box3c, align=center] (box4c) {Survey Responses\\{\tiny X1:90, \ X2:85, \ X3:55, ...}};

        \draw[->] (box3a) -- (box4a) node[midway, left] {};
        \draw[->] (box3b) -- (box4b) node[midway, left] {};
        \draw[->] (box3c) -- (box4c) node[midway, left] {(iv)} node[midway, right] {LLM};
    \end{tikzpicture}
    \caption{
\textbf{LLM sampling scheme for generating interventional survey data.}
We (i)~sample demographic covariates from real survey respondents, (ii)~use an LLM to expand them into detailed respondent profiles, (iii)~apply specified interventions to create multiple intervened versions of each profile, and finally (iv)~generate survey responses based on the intervened profiles. This pipeline preserves heterogeneity of the final survey responses while introducing controlled interventions.
}
    \label{fig:llm_3step}
\end{figure}
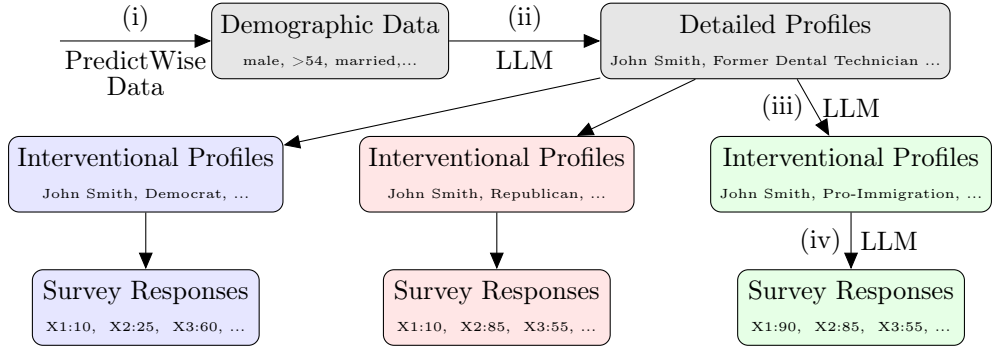

With this approach, we generate $|\Ecal|=11$ sets of synthetic survey results: one baseline with no interventions, and five interventional pairs in which partisanship and views on four issue topics are shifted in either direction (e.g., democrat and republican). For each synthetic respondent, we generate a total of $D=27$ responses, divided among the five question topics, with exact question wording taken from the original PredictWise surveys:
\begin{itemize}
    \item \textbf{General}: Three questions about political party affiliation, approval of Donald Trump, and intended vote choice for the House of Representatives.
    \item \textbf{Regulation}: Six questions about whether government regulations are necessary to protect the public interest, including regulations of pollution, workplace safety, banks, and product safety.
    \item \textbf{Trade}: Six questions about support for free trade and its associated trade-offs, such as cheaper goods, job displacement, quality control, stress on the borders, and increased competition.
    \item \textbf{Immigration}: Six questions about whether recent immigrants strengthen the country, labor market, cultural life, and national security, as well as whether the US--Mexico border is secure.
    \item \textbf{Poverty}: Six questions about whether poor people have sufficient government support to attain education, shelter, healthcare, job training, and food.
\end{itemize}
The five intervention targets used to generate the interventional profiles are matched to the latent political topics underlying these questions: general partisanship, regulation, trade, immigration, and poverty. Thus, the observed variables $\{X_1,\dots,X_{27}\}$ form blocks associated with $\{Z_1,\dots,Z_5\}$, suggesting a sparse $\Zb \to \Xb$ relationship. This sparsity, however, is not imposed as a hard constraint on the generated responses; interventions on one latent topic may also affect questions in other blocks, either through downstream causal effects or through mixing of latent topics within the LLM.

Because the LLM data were generated under known interventions, we can compute a ``ground-truth'' average treatment effect (ATE) for each latent topic. This quantity is computed directly from the raw synthetic data using the known intervention structure, which is not provided to our model, and therefore serves as a reference for the causal structure recovered by our model.
Specifically, for each question $d  \in \{1 \ldots 27\}$ and each intervention target $\ell \in \{1 \ldots 5\}$,
\begin{equation}
\label{eq:ate}
\text{ATE}(d, \ell) = 
\EE_{do(Z_\ell \to + 0.5)}  \left[X_d \right]
- 
\EE_{do(Z_\ell \to - 0.5)}  \left[X_d \right],
\end{equation}
where $do(Z_\ell = \pm 0.5)$ denotes a hard intervention setting latent topic $\ell$ to the left-leaning $(+0.5)$ or right-leaning $(-0.5)$ position.\footnote{The values $\pm 0.5$ arise from the binary encoding of the intervention targets with unit spacing, as defined in~\cref{sec:model_measurement}.} For example, $do(Z_{\text{Party}} = -0.5)$ denotes the environment in which all profiles have been intervened on to be Republican. These ATEs capture the total effect of each intervention on each observed question, including both direct effects and indirect effects mediated through other latent topics.

The resulting ATE structure is shown in \cref{fig:llm_ATE}. The heatmap is nearly block lower triangular, indicating that LLM sampling scheme largely respected a causal ordering rather than forming arbitrary bidirectional dependencies. Off-diagonal blocks indicate the total effect that an intervention on one latent topic has on questions associated with a different latent topic, capturing the combined direct and indirect causal effects among latent concepts. The resulting graph shows a strong \texttt{Regulation $\to$ Poverty} effect as well as \texttt{Partisanship $\to\{$Regulation,Poverty,Immigration}$\}$ effects, providing a reference for the causal effects we expect our model to recover.

\begin{figure}[tbp]
\includegraphics[width = 0.6 \textwidth,trim={0 1cm 0 1cm},clip]{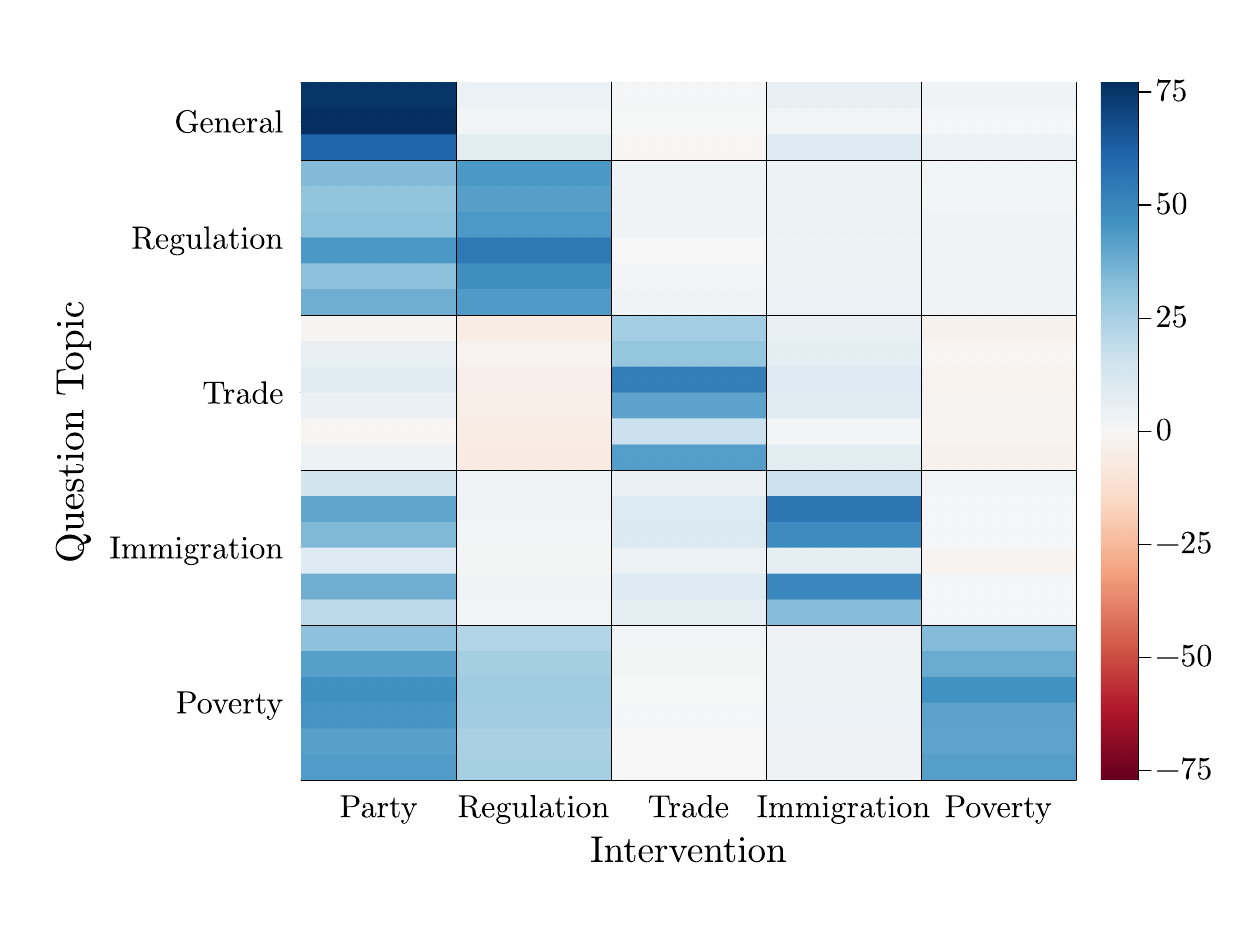}
\centering
\raisebox{0.2  \height}{%
\resizebox{0.38\textwidth}{!}{
\begin{tikzpicture}
\tikzstyle{box} = [rectangle, rounded corners, 
                    minimum width=3cm, 
                    minimum height=1cm,
                    text centered, 
                    draw=black]
\node (G) at (90:2cm) [box,draw,minimum size=8mm] {Party};
\node (T) at (+18:2cm) [box,draw,minimum size=8mm] {Trade};
\node (R) at (-126:2cm) [box,draw,minimum size=8mm] {Regulation};
\node (P) at (-54:2cm) [box,draw,minimum size=8mm] {Poverty};
\node (I) at (-198:2cm) [box,draw,minimum size=8mm] {Immigration};

\newcommand{\drawArrow}[4][]{
  \pgfmathsetmacro{\opacity}{#4/50} %
  \draw[->,line width=\opacity mm, draw opacity=\opacity, fill opacity=\opacity, bend left=10, #1] (#2) to (#3);
}

\drawArrow{G}{T}{4.098333333333336};
\drawArrow{G}{R}{35.78833333333334};
\drawArrow{G}{P}{43.049};
\drawArrow{G}{I}{27.139666666666663};

\drawArrow{T}{G}{0.4406666666666652};
\drawArrow{T}{R}{2.2013333333333236};
\drawArrow{T}{P}{0.9866666666666646};
\drawArrow{T}{I}{8.359333333333332};

\drawArrow{R}{G}{5.652000000000001};
\drawArrow{R}{T}{5.021000000000001};
\drawArrow{R}{P}{25.168666666666674};
\drawArrow{R}{I}{1.9819999999999993};

\drawArrow{P}{G}{2.608000000000004};
\drawArrow{P}{T}{2.0900000000000034};
\drawArrow{P}{R}{1.8386666666666684};
\drawArrow{P}{I}{0.5309999999999988};

\drawArrow{I}{G}{7.130000000000003};
\drawArrow{I}{T}{7.387666666666661};
\drawArrow{I}{R}{4.224666666666678};
\drawArrow{I}{P}{3.8543333333333436};
\end{tikzpicture}
}
}
\caption{\textbf{``Ground-truth'' effects computed from LLM-generated interventional pairs.} Each cell shows the average treatment effect for a question--intervention pair, as defined in~\eqref{eq:ate}. For example, the bottom--left cells show the mean difference in immigration responses between the Democrat and Republican interventional profiles. The graph is generated by weighting edges according to the average question--level effect within each topic block: $\frac{1}{|T|} \sum_{d\in T} \text{ATE}(d, \ell)$.}
\label{fig:llm_ATE}
\end{figure}
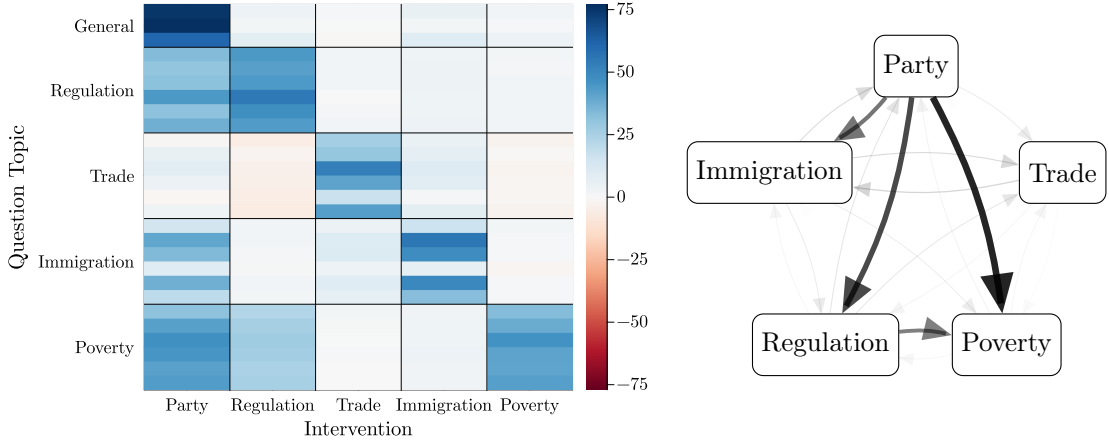
\subsubsection{Modeling}
While the construction of \cref{fig:llm_ATE} relied on knowledge of the question groupings and intervention targets, our model is not given access to this information. Instead, it observes only 11 unlabeled environments, each containing $N^e = 1000$ synthetic survey responses. Each environment consists of a stochastic intervention with an 85--15 split (e.g., the ``pro-trade'' environment consists of 85\% pro-trade and 15\% anti-trade interventions).
That is, each environment intervenes on a single $\ell \in [L]$ with $do(Z_\ell=N_\ell)$ where $N_\ell$ is a Categorical with either $p(N_\ell=+0.5)=0.85$ or $p(N_\ell=-0.5)=0.85$. We fit our model with $L=5$ binary latents.

\subsubsection{Interpretation of Inferred Causal Concepts} 
\begin{figure}[tbp]
\centering
\includegraphics[width = 0.6  \textwidth]{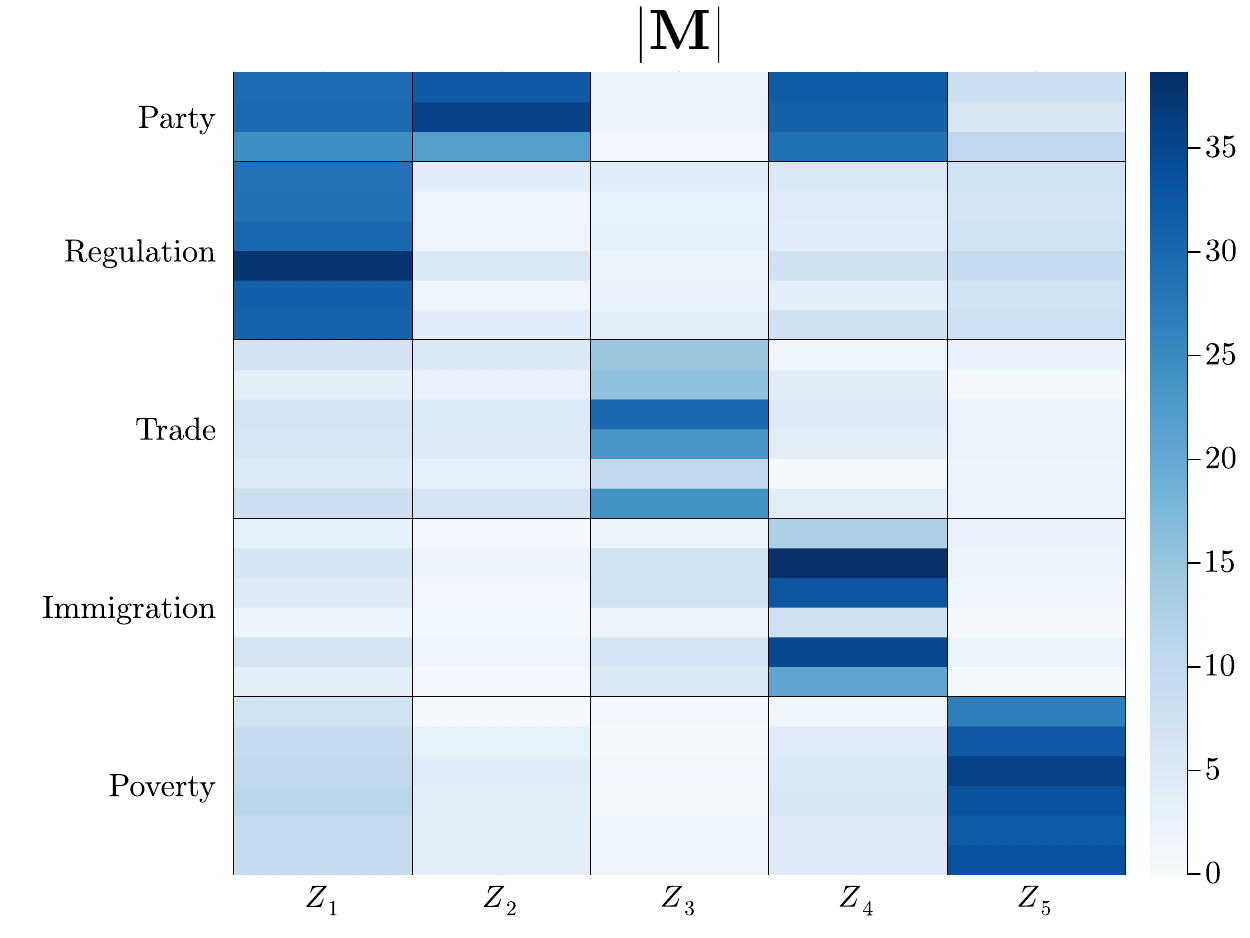}
\caption{\textbf{Estimated measurement mapping for LLM-generated survey responses.}
Heatmap of the posterior mean absolute measurement matrix, $|\Mb|$. The learned block
structure is interpretable as $Z_1,Z_2,Z_3,Z_4,Z_5$ as corresponding to \texttt{Regulation}, \texttt{Party}, \texttt{Trade}, \texttt{Immigration}, and \texttt{Poverty}, respectively, with some \texttt{Party} information also contained in $Z_1$ and $Z_4$. Question topics in this figure and \cref{fig:llm_ATE} are ordered post hoc to facilitate comparison of the learned structure.
}
\label{fig:llm_M}
\end{figure}
As with the WVS data, we begin by interpreting the learned latent variables through the measurement model. \Cref{fig:llm_M} shows a clear block structure in the posterior mean of $|\Mb|$. Since the LLM-generated survey was constructed from known issue groups, we can compare the inferred blocks to the underlying survey design. The resulting associations map each latent variable primarily onto one issue dimension, with some leakage of party-related information into multiple variables:
{\small
\[
Z_1 \mapsto \text{Regulation(\& Party)}, \quad
Z_2 \mapsto \text{Party}, \quad
Z_3 \mapsto \text{Trade}, \quad
Z_4 \mapsto \text{Immigration(\& Party)}, \quad
Z_5 \mapsto \text{Poverty}.
\]
}
\subsubsection{Causal Graph Recovery}

\begin{figure}[tbp]
\includegraphics[width = 0.58 \textwidth]{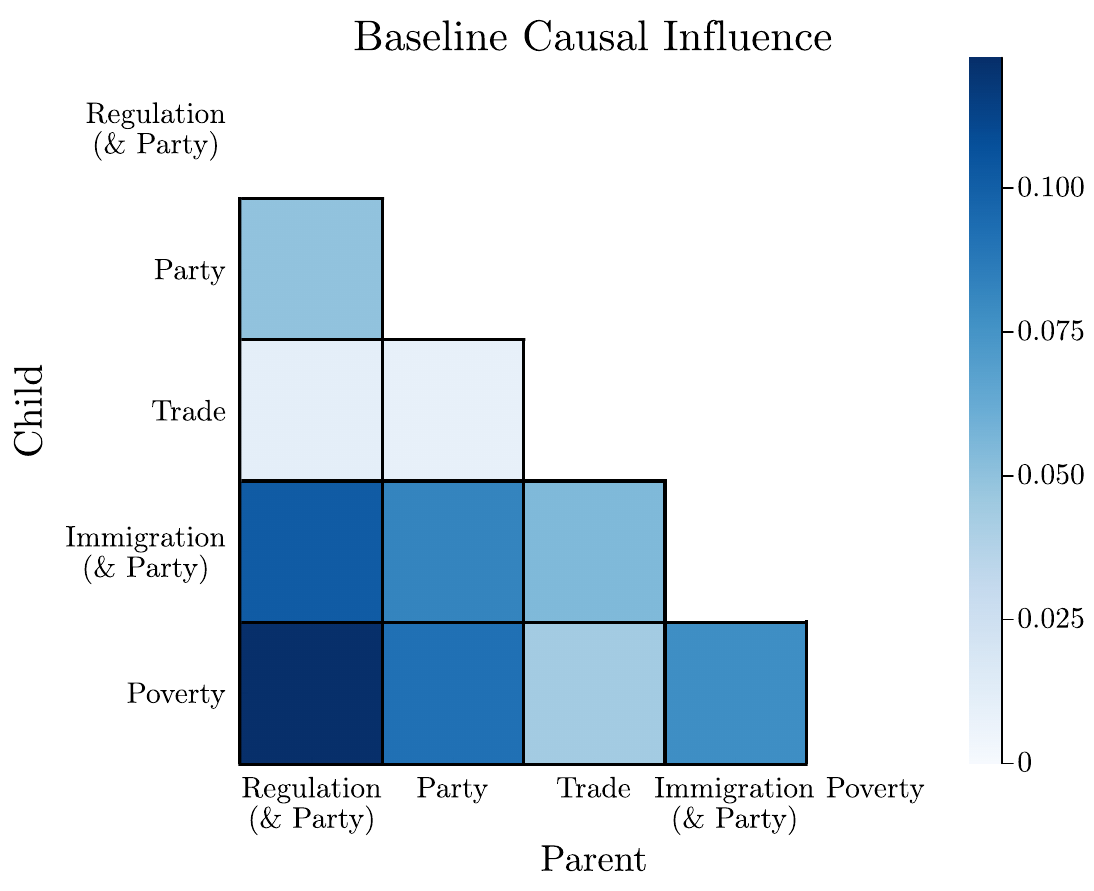}
\raisebox{0.25\height}{%
\resizebox{0.4\textwidth}{!}{
\begin{tikzpicture}
\tikzstyle{box} = [rectangle, rounded corners, 
                    minimum width=3cm, 
                    minimum height=1cm,
                    text centered, 
                    draw=black]
\newcommand{\drawArrow}[4][]{
  \pgfmathsetmacro{\opacity}{#4/.15 } %
  \draw[->,line width=\opacity mm, draw opacity=\opacity, fill opacity=\opacity, bend left=10, #1] (#2) to (#3);
}
\node (z1) at (-126.0:3cm) [box,draw,minimum size=8mm, align=center] {Regulation \\ (\& Party)};
\node (z2) at (90.0:3cm) [box,draw,minimum size=8mm, align=center] {Party};
\node (z3) at (18.0:3cm) [box,draw,minimum size=8mm, align=center] {Trade};
\node (z4) at (-198.0:3cm) [box,draw,minimum size=8mm, align=center] {Immigration \\ (\& Party)};
\node (z5) at (-54.0:3cm) [box,draw,minimum size=8mm, align=center] {Poverty };
\drawArrow{z1}{z2}{0.049626931197355346};
\drawArrow{z1}{z3}{0.01125343464878658};
\drawArrow{z2}{z3}{0.009453086957755654};
\drawArrow{z1}{z4}{0.1020866804326021};
\drawArrow{z2}{z4}{0.08256695750673616};
\drawArrow{z3}{z4}{0.05501078332859461};
\drawArrow{z1}{z5}{0.1227274758908745};
\drawArrow{z2}{z5}{0.09208431660134206};
\drawArrow{z3}{z5}{0.04392101551954549};
\drawArrow{z4}{z5}{0.07793042074274692};
\end{tikzpicture}
}}
\caption{
\textbf{Causal graph recovered from LLM responses.}
Estimated causal influence among latent concepts in the baseline environment, with latent variables labeled according to the block structure identified in \cref{fig:llm_M}. The largest recovered dependency is from $\text{Regulation} \to \text{Poverty}$, with weaker relationships among party, regulation, immigration, and poverty, and very weak effects involving trade.\looseness-1}
\label{fig:llm_CI}
\end{figure}
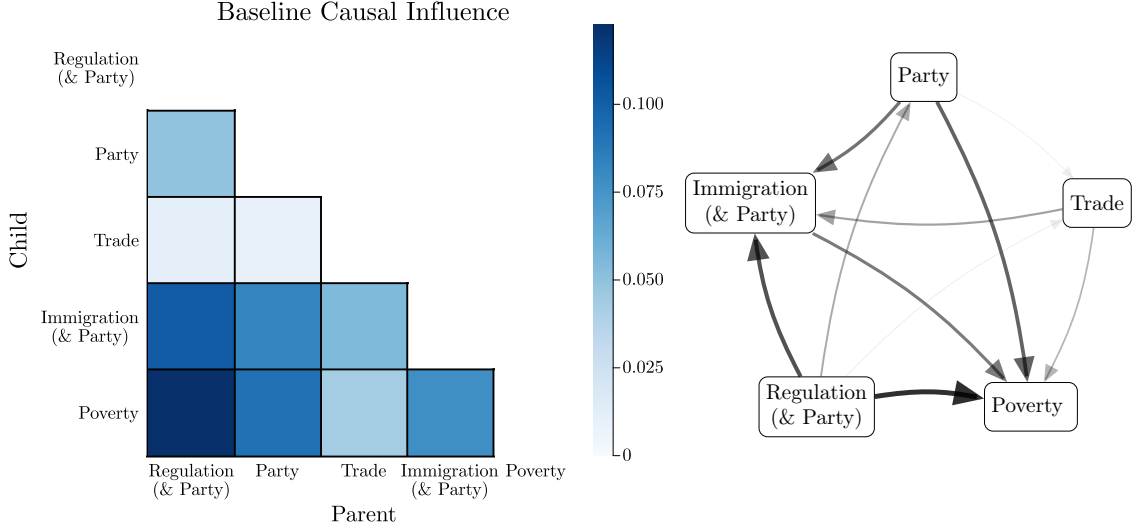

The learned causal order is consistent with the ``ground-truth'' DAG~(\cref{fig:llm_ATE}); however, because this ``ground-truth'' causal graph is relatively sparse, there are multiple such orders that would also be consistent. Thus, the individual dependencies provide a more nuanced and informative comparison. The estimated causal graph is shown in \cref{fig:llm_CI}, with latent variables labeled according to their semantic interpretations from \cref{fig:llm_M}. Several dependencies emerge, the strongest being the directed relationship from \texttt{Regulation} to \texttt{Poverty}. In contrast, \texttt{Trade} appears largely independent of the other issue dimensions. The results suggest that \texttt{Party} is upstream of \texttt{Poverty}; however, its relationships with \texttt{Regulation} and \texttt{Immigration} are harder to interpret because the learned representation does not isolate partisanship in a single latent variable. Instead, the model distributes partisanship information across multiple latent variables, with party-related signal also mixed into the \texttt{Regulation} and \texttt{Immigration} latents. 

Although this entanglement makes causal relationships involving \texttt{Party}, \texttt{Regulation}, and \texttt{Immigration} difficult to interpret, the recovered graph broadly aligns with the ``ground-truth'' treatment effects from the LLM-generated interventions. In particular, both summaries show weak effects involving \texttt{Trade}, strong effects of \texttt{Party} and \texttt{Regulation} on \texttt{Poverty}, and additional causal relationships among \texttt{Party}, \texttt{Regulation}, and \texttt{Immigration}.

Overall, these results confirm that our model can correctly recover both the latent representation and causal structure in a complex semi-synthetic setting with five latent variables and LLM-generated survey responses. This complements the simpler synthetic experiments in \cref{sec:synthetic_experiments} and provides additional support for the prior empirical case studies in Sections~\cref{sec:case_study_wvs} and~\ref{sec:case_study_us_real}, where ground-truth comparisons are unavailable.

\section{Discussion}
\label{sec:discussion}

Our work connects to several strands of literature in causal representation learning, psychometrics, and Bayesian causal inference.

\paragraph{Causal Representation Learning.}
Our work is most closely related to various other approaches to %
causal representation learning (CRL). %
As detailed in~\cref{sec:introduction}, our work  
differs from existing approaches in that we focus on estimation and interpretation instead of identifiability, adopt a fully Bayesian framework, and perform joint inference over latent variables and their causal relationships.

\citet{kivva2021learning} and \citet{zhang2026discrete} %
also study learning discrete causal representations, albeit from a single observational environment. These works impose additional structural conditions on the measurement model $p(\Xb \mid \Zb)$ to obtain identifiability. In particular, \citet{kivva2021learning} show that, in the setting of a \textit{single} environment, restrictive conditions on $p(\Xb \mid \Zb)$ ensure identifiability and allow recovery of both the number $L$ of latent variables and their cardinalities $|\Zcal_\ell|$. More recently, \citet{zhang2026discrete} propose a discrete causal representation learning framework that models a directed acyclic graph among discrete latent variables, together with a sparse bipartite graph linking latent variables to observed measurements. 
Both of these works are applicable to similar structures of observed and latent variables. The primary difference to our approach is that they rely on restrictions of the measurement model to learn from a single domain, whereas we instead exploit assumptions on sparse latent interventions to learn from multi-domain data.

Applications of CRL to survey settings have also been explored in recent work. In particular,  \citet{huang2022latenthierarchicalcausalstructure} and \citet{dai2025latent} leverage rank constraints to learn latent causal structure from personality and social survey data.

\paragraph{Bayesian Item Response Models.}
A long line of research in psychometrics studies latent variable models for survey and test data (see \citet{reise2010bifactor} for an introduction). Most notably, Item Response Theory (IRT) \citep{lord1980applications},  views responses as being driven by unobserved latent traits (e.g., ability, attitudes, preferences). These latent characteristics are not directly observable, but are instead inferred from individuals’ patterns of responses to survey or test items. IRT provides a family of probabilistic models that formalize the relationship between individuals, survey items, and the latent traits of interest. Each item is associated with parameters that determine how responses vary with the latent trait, while each individual is represented by a position on an unobserved continuum corresponding to that trait.
A substantial body of work develops Bayesian formulations of IRT  \citep{albert1993bayesian, jiang2019gibbs, li2025sparse}. In our setting, we assume discrete instead of continuous latent traits %
 and impose causal structure among them, but the underlying perspective and methods used are closely aligned.

\paragraph{Bayesian Causal Discovery and Experimental Design.}
Bayesian approaches for inferring the causal structure among a set of \textit{observed} variables from purely observational data go back at least to the work of~\citet{heckerman1995learning,heckerman2006bayesian} and \citet{friedman2003being}, typically for categorical variables with conjugate Dirichlet-multinomial models.
More recent work has generalized this approach  %
to continuous variables and nonlinear models with differentiable graph parametrization combined with variational inference~\citep{lorch2021dibs}, as well as to multi-domain data with unknown interventions~\citep{hagele2023bacadi}.
The posterior over structures and model parameters can, in turn, be used for active learning or Bayesian optimal experimental design, i.e., to select interventions for future data collection in categorical~\citep{murphy2001active,tong2001active}, linear~\citep{ness2018bayesian,agrawal19b} and nonlinear~\citep{toth2022active} models.
For the local causal discovery task of finding the parents of a target variable, \citet{wu2025bayesian} propose a Bayesian approach for invariant causal prediction~\citep{peters2016causal} from multi-domain data.

\paragraph{Limitations and Future Work.}
Our approach has several limitations that suggest natural directions for future work.
Firstly, our approach assumes a complete graph over the latent variables. While flexible, this leads to an exponential blowup in the size of the parameter space as the number and cardinality of the latent variables grows. As a result, inference becomes computationally prohibitive quickly as the number of latent variables grows. Developing more parsimonious parameterization---such as by explicitly exploiting the ordinal structure of the latent variables, restricting the graph structure by limiting node degree or introducing parameter sharing across similar parent configurations--could help mitigate this issue, but doing so in a way that preserves flexibility remains an open challenge.

\looseness-1 Second, the measurement model linking latent variables to observed responses is relatively simple in our current framework. This choice provides flexibility and ease of interpretability, but incorporating richer measurement models that are more closely tailored to the data format, such as ordinal responses or hierarchical item groupings, could be more principled and improve empirical performance. One possible direction for ordinal survey data is the use of unified skew-normal distributions, which \cite{durante2019conjugate} recently showed enable conjugate Bayesian inference for probit models.

Third, in our case studies, we do not actually have access to experimental data. We model heterogeneity across domains as soft interventions that modify a subset of the underlying mechanisms. To verify causal claims, truly randomized studies remain the gold standard.

\section{Conclusion}
\label{sec:conclusion}
We have introduced a Bayesian framework for learning discrete causal concepts from multi-domain data. Our approach encodes common assumptions from the causal representation learning literature through model priors and constraints. We develop an efficient sequential Monte Carlo-based inference scheme and evaluate our model on political survey data, including both real-world datasets and data generated by large language models. In these settings, the model is able to recover interpretable latent structure that aligns with known political trends.
A central contribution of this work is to move beyond idealized synthetic settings and develop a Bayesian causal latent variable model that can be effectively fit to real-world data. We view this work as a step toward bridging the gap between theoretical developments in causal representation learning and real-world settings with model misspecification and finite samples.

\section*{Acknowledgements}
The authors thank David M.\ Blei for several insightful discussions, Yuli Slavutsky for feedback on the manuscript, and Tobias Konitzer and David Rothschild for access to and helpful discussions about the PredictWise survey dataset. JvK is supported by The Branco Weiss Fellowship---Society in Science.

\section*{Author Contributions}
JvK conceived of the project, proposed the initial model, and brought the co-authors together. AG then led all major aspects of the project with advice and feedback from JvK and AS: (i)~iterating on and refining the generative model, particularly the logit parametrization and novel KL-shift prior; (ii)~developing the SMCS posterior inference procedure; (iii)~implementing and maintaining the main codebase; and (iv)~conducting empirical case studies. MS contributed early work on variational inference and conducted empirical comparisons to the SMCS procedure. AG and JvK wrote the manuscript with feedback from AS. 

\appendix

\changelinkcolor{black}{}
\parttoc
\changelinkcolor{BrickRed}{}

\section{KL Shift Prior}
\label{app:KL}

This appendix provides additional details on the motivation for and the implementation of the KL-truncation shift prior.

The following Lemma shows that the truncation threshold~\(\varepsilon\) in~\cref{eq:truncated_normal} can be interpreted as a finite-sample
signal-strength condition.

\begin{lemma}[A finite-sample detectability lower bound]
\label{lem:kl_min_detectability}
Let \(p_0(\cdot \mid \pa_\ell)\) denote the
shared mechanism for \(Z_\ell\), and let
\(p_\delta(\cdot \mid \pa_\ell)\) denote the corresponding shifted
mechanism under an intervention. Suppose that environment \(e \in \Ecal \) contains
\(N_e\) independent samples from one of these two mechanisms. Let \(T\) be any test of 
\[
H_0: p=p_0
\qquad\text{versus}\qquad
H_1: p=p_\delta.
\] 
Then, for any \(\eta \in [0,1]\),
\[
p_0(T) + p_\delta(1-T) \leq \eta
\quad \implies  \quad \operatorname{KL}\left(p_0 \,\|\, p_\delta\right) \geq \frac{2(1-\eta)^2}{N_e}.
\]
\end{lemma}

\begin{proof}

The error for any simple hypothesis test  can be lower bounded as
\[
p_0(T) + p_\delta(1-T)
\;\geq\;
1 - \operatorname{TV}\left(p_0^{N_e},p_\delta^{N_e}\right)
\;\geq\;
1 - \sqrt{\frac{N_e}{2}
\operatorname{KL}\left(p_0 \,\|\, p_\delta\right)} .
\]
The first inequality follows from standard testing lower bounds \citep[e.g.,][Theorem 15.1.1]{lehmann2022testing} and the second from Pinsker's inequality \citep[e.g.,][Lemma 11.6.1]{cover2006information}.\looseness-1

Now suppose that
\(
p_0(T)+p_\delta(1-T)\leq \eta.
\)
Combining this with the above lower bound gives
\[
1 - \sqrt{\frac{N_e}{2}
\operatorname{KL}\left(p_0 \,\|\, p_\delta\right)}
\leq \eta,
\]
and rearranging yields
\[
\operatorname{KL}\left(p_0 \,\|\, p_\delta\right)
\geq
\frac{2(1-\eta)^2}{N_e}.
\]
 \end{proof}

Thus, without a lower bound on the KL divergence induced by an
intervention, it can be impossible to accurately test for interventions with finite samples, making reliable recovery of intervention targets impossible.

\paragraph{Second-order expansion of the KL divergence.}
The quadratic approximation in \cref{eq:hessian_approximation} follows from a local expansion of the KL divergence.
Let $p = \softmax(\thetab)$ and consider
\[
\mathrm{KL}\!\left(\softmax(\thetab)\,\middle\|\,\softmax(\thetab+\deltab)\right).
\]
A second-order Taylor expansion in $\deltab$ around $\deltab=\bm 0$ gives
\[
\mathrm{KL}\!\left(\softmax(\thetab)\,\middle\|\,\softmax(\thetab+\deltab)\right)
\;\approx\; \tfrac{1}{2}\deltab^\top H_{\thetab}\,\deltab,
\]
where
\[
H_{\thetab} = \mathrm{diag}(p) - p p^\top
\]
is the Fisher information matrix of the categorical distribution parameterized by $\thetab$.

\paragraph{Why approximate and why not evaluate at $\thetab$?}
In principle, the most accurate truncation uses $H_{\thetab}$ evaluated at the current $\thetab$, or even the exact KL. However, this is not computationally feasible.

To see this, consider the simplified model
\[
p(a) = \mathcal{N}(0,1), \qquad p(b \mid a) \propto \mathcal{N}(0,1) \cdot \mathbbm{1}\!\left\{|b|>|a|\right\},
\]
since the truncation region for $b$ depends on $a$, 
\[
p(a \mid b) \propto p(a)\,p(b \mid a) 
= \phi(a)\,\frac{\phi(b) \cdot  \mathbbm{1}\!\left\{|b|>|a|\right\}}{Z(a)},
\]
where $Z(a)$ is the normalization constant of the truncated Gaussian. Because $Z(a)$ depends on $a$, the conditional $p(a \mid b)$ is no longer a tractable conjugate update.

The same issue occurs with ($\thetab,\deltab$): if the truncation depends on $\thetab$ through $H_{\thetab}$, then the normalization constant of the truncated prior over $\deltab$ depends on $\thetab$. This breaks conjugacy and prevents efficient Gibbs updates for $\thetab$.
To retain tractable conjugate updates, we instead fix the Hessian at $\thetab=\bm 0$, yielding the quadratic constraint used in the main text:
\[
\tfrac{1}{2}\deltab^\top H_{\bm 0}\,\deltab \geq \varepsilon.
\qquad 
H_{\bm 0} = \frac{1}{K}\left(I - \frac{1}{K}\bm 1 \bm 1^\top\right).
\]

\paragraph{Alternative: joint truncation.}
One alternative is to jointly truncate $(\thetab,\deltab)$ as
\begin{align*}
 p(\thetab,\deltab) 
 &\propto \phi(\thetab)\,\phi(\deltab)\,\mathbbm{1}\{f(\thetab,\deltab) \geq \varepsilon\}.
\end{align*}
However, this requires rejection sampling from across all environments, which quickly becomes computationally infeasible.

\begin{figure}[tbp]
\centering
\includegraphics[width=\textwidth]{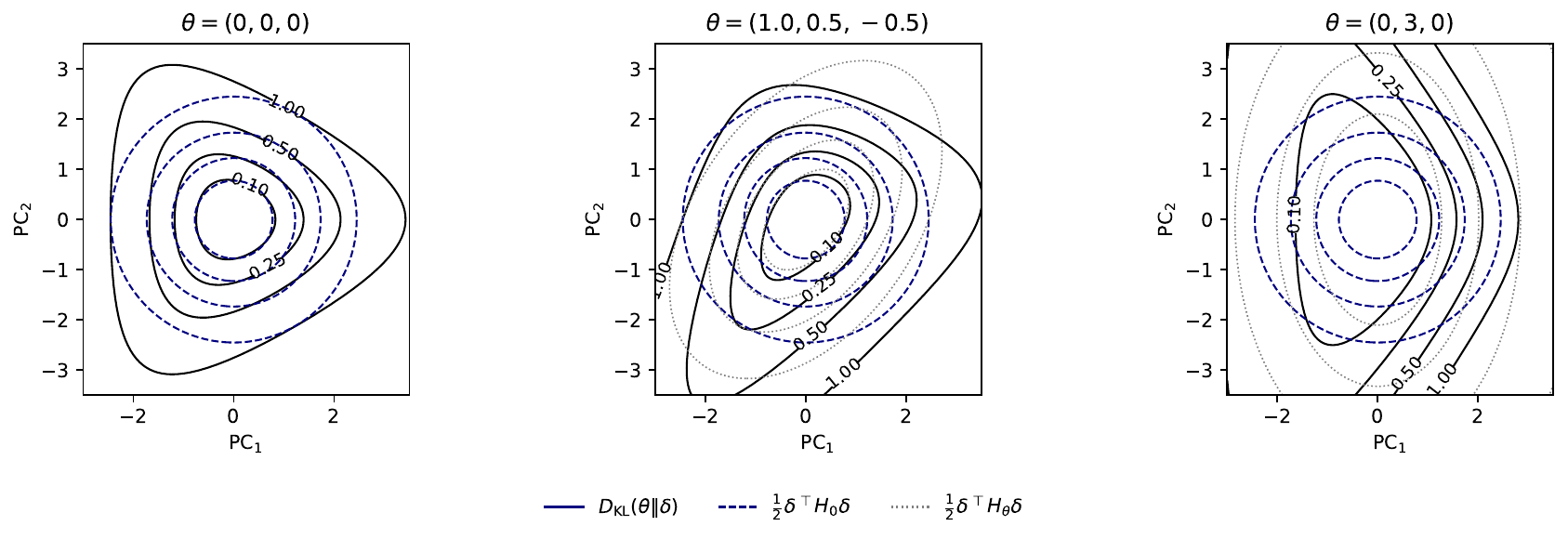}
\caption{\looseness-1 
\textbf{Quadratic KL constraint on $\deltab$ for various $\thetab$.}
We compare the exact KL divergence with second-order approximations based on $H_{\bm 0}$ and $H_{\thetab}$. 
The $H_{\bm 0}$ approximation deteriorates when $\thetab$ is far from zero (i.e., highly concentrated distributions), but this regime is unlikely under our prior.}
\label{fig:kl_normal_comparison}
\end{figure}

\paragraph{Approximation quality.}
The approximation is best at $\thetab=\bm 0$ and degrades as $\thetab$ becomes more extreme (i.e., when the categorical distribution is highly concentrated). Under our prior, however, $\thetab$~typically remains close to zero, so this regime is unlikely.
Empirically, we found this approximation reliably excludes negligible shifts even when $\thetab$ deviates moderately from zero (see \cref{fig:kl_normal_comparison}).

\section{Inference}
\label{app:inference}

This appendix provides full details of the inference procedures, including complete conditional distributions and algorithmic details.\subsection{Conjugate Updates}
\label{app:inference_conj}
\subsubsection{immediate conjugacy}
The following conditional distributions are immediate from standard conjugacy:
\begingroup
\allowdisplaybreaks
\begin{align*}
p(c_\ell \mid -) &= \operatorname{Beta}\left(\alpha_c + \sum_{e\in \Ecal} \Ical_\ell^e,\beta_c+|\Ecal|-\sum_{e\in \Ecal }\Ical_\ell^e\right) 
\\ 
\\
p(\Ical_\ell^e = \Ical^* \mid -) &\propto
c_\ell^{\Ical^*}(1-c_\ell)^{1-\Ical^*}
\prod_{j\in [N_e]}
p\left(
z^{e,j}_{\ell}|   \thetab_\ell\big(\pa_\ell^{e,j}\big) , \deltab_\ell^e\big(\pa_\ell^{e,j}\big) ,\Ical_\ell^e = \Ical^* 
\right) 
\\ 
\\
&= 
c_\ell^{ \Ical^* }(1-c_\ell)^{1- \Ical^* }
\prod_{j\in [N_e]}
\frac{\exp \left(  \thetab_\ell\big(\pa_\ell^{e,j}\big)_{z^{e,j}_{\ell}} + \Ical^*  \cdot  \deltab_\ell^e\big(\pa_\ell^{e,j}\big)_{z^{e,j}_{\ell}} \right)}
    {\sum_{z' \in [K]} \exp \left(  \thetab_\ell\big(\pa_\ell^{e,j}\big)_{z'} + \Ical^*  \cdot  \deltab_\ell^e\big(\pa_\ell^{e,j}\big)_{z'} \right)}
\\ 
\\ 
p(\zb^{e,j}=\zb^* \mid - ) 
&\propto 
p(\xb^{e,j} \mid \Mb,\sigmab,\zb^{e,j}=\zb^*)
\prod_{\ell \in [L]}
p\left(
z^*_\ell \mid   \thetab_\ell\big(\zb^*_{1:\ell-1}\big) , \deltab_\ell^e\big(\zb^*_{1:\ell-1}\big) ,\Ical_\ell^e 
\right) 
\\ &\propto \phi_D\left(\xb^{e,j}  ; \; \mb_0 +\textstyle\sum_{\ell \in [L]} \mb_\ell z^*_{\ell},
\mathrm{diag}\big(\sigmab^2\big)\right)
\prod_{\ell \in [L]}
\frac{\exp \left(  \thetab_\ell\big(\zb^*_{1:\ell-1} \big)_{z^*_{\ell}} +\Ical_\ell^e  \cdot  \deltab_\ell^e\big(\zb^*_{1:\ell-1} \big)_{z^*_{\ell}} \right)}
    {\sum_{z' \in [K]} \exp \left(  \thetab_\ell\big(\zb^*_{1:\ell-1}\big)_{z'} +\Ical_\ell^e  \cdot  \deltab_\ell^e\big(\zb^*_{1:\ell-1}\big)_{z'} \right)}
\end{align*}
\subsubsection{measurement model}
With the measurement model, for each dimension $d \in [D]$, we pool observations across all $(e,j)$. 
Let $N = \sum_{e \in \Ecal} N_e$. Define the stacked response vector $\Xb_d \in \mathbb{R}^N$ and design matrix $\Zb \in \mathbb{R}^{N \times (L+1)}$ with entries indexed by $(e,j)$:
\[
(\Xb_d)_{(e,j)} = x^{e,j}_d,
\qquad
\Zb_{(e,j),:} = \begin{bmatrix} 1 & z^{e,j}_1 & \cdots & z^{e,j}_L \end{bmatrix},
\quad e \in \Ecal,\; j \in [N_e].
\]

The mean $\Mb$ admits Gaussian--Gaussian conjugate updates:
\begin{align*}
\Xb_d \mid \mb_d, \sigma_d^2 
\sim \mathcal{N}_N\big(\Zb \mb_d, \sigma_d^2 I\big) 
\implies&\;
(\Zb^\mathsf{T} \Zb)^{-1}\Zb^\mathsf{T} \Xb_d
\sim \mathcal{N}_{L+1}\!\left(
\mb_d,\;
\sigma_d^2 (\Zb^\mathsf{T} \Zb)^{-1}
\right) \\
\implies&\;
\mb_d \mid -
\sim \mathcal{N}_{L+1}\!\left(\mub_d, \Sigmab_d\right),
\\
&\Sigmab_d
= \left(\Ib+ \frac{\Zb^\mathsf{T} \Zb}{\sigma_d^2} \right)^{-1},
\\
&\mub_d
= \Sigmab_d \left( \frac{\Zb^\mathsf{T} \Xb_d}{\sigma_d^2} \right).
\end{align*}

The variance $\sigma_d^2$ admits an inverse-gamma--Gaussian conjugate update:
\begin{align*}
\sigma_d^2 \mid - 
&\sim \operatorname{Inv\text{-}Gamma}\!\left(
\alpha_\sigma + \frac{N}{2},\;
\beta_\sigma + \frac{1}{2} \|\Xb_d - \Zb \mb_d\|_2^2
\right).
\end{align*}
Note: Here we write $\mb_d \in \RR^{L+1}$ for the $d$-th row of $\Mb$:\ 
\(
\Mb = \begin{bmatrix}
    \mb_1^\top
    \\
    \vdots
    \\
    \mb_D^\top
\end{bmatrix}
\)

\subsubsection{Pólya--Gamma Augmentation:} 
The updates for $\Thetab,\Deltab$ do not follow immediately from conjugacy and require data augmentation~\citep{polson2013bayesian,chen2013polya}. 

Fix a node $\ell \in [L]$ and parent configuration $\pa_\ell \in \mathcal{PA}_\ell$. For the remainder of this section, we suppress this indexing:
\[
\thetab := \thetab_\ell(\pa_\ell), \qquad \deltab^e := \deltab_\ell^e(\pa_\ell).
\]

For each environment $e \in \Ecal$, define the following counts and associated logits:
\begin{align*}
C^e_k 
&= \sum_{j \in [N_e]} \mathbbm{1}\{ z^{e,j}_\ell = k \;\land\; \pa_\ell^{e,j} = \pa_\ell \}, \\
N^e 
&= \sum_{k \in [K]} C^e_k,
\\
\Gammab &= \Ab \Bb.
\end{align*}

Here $\Ab \in \mathbb{R}^{E \times (E+1)}$ and $\Bb \in \mathbb{R}^{(E+1)\times K}$ are given by
\[
\Ab =
\begin{bmatrix}
1 & \Ical_\ell^1 & 0 & \cdots & 0 \\
1 & 0 & \Ical_\ell^2 & \cdots & 0 \\
1 & 0 & 0 & \ddots & 0 \\
\vdots & \vdots & \vdots & \ddots & \Ical_\ell^E
\end{bmatrix},
\qquad
\Bb = 
\begin{bmatrix}
\thetab^\top \\
(\deltab^1)^\top \\
\vdots \\
(\deltab^E)^\top
\end{bmatrix}
\]

Now, we fix a category index $k \in [K]$ and define the conditional update
for the $k$th column $\Bb_{:,k}$ given $\Bb_{:,\lnot k}$.
The aggregated counts satisfy
\begin{align*}
C^e_k 
&\sim \operatorname{Binomial}\!\left(
N^e,\;
\frac{\exp(\Gamma_{e,k})}{\sum_{k'} \exp(\Gamma_{e,k'})}
\right)
=
\operatorname{Binomial}\!\left(
N^e,\;
\frac{\exp(\rho^e)}{1 + \exp(\rho^e)}
\right),
\end{align*}
where
\[
\rho^e = \Gamma_{e,k} - \zeta^e,
\qquad
\zeta^e = \log\!\left(\sum_{k' \neq k} \exp(\Gamma_{e,k'})\right).
\]

This is the form required for Pólya--Gamma augmentation. We introduce auxiliary variables
\[
\omega^e \sim \mathcal{PG}(N^e, \rho^e), \qquad e \in \Ecal,
\]
and define $\Omegab = \mathrm{diag}(\omega^1,\dots,\omega^E)$ and $\bm{\kappa} \in \mathbb{R}^E$ with
\[
\kappa^e = C^e_k - \frac{N^e}{2}.
\]

Including $\Omegab$, the complete conditional (ignoring truncation of $\deltab$) is:
\begin{align*}
\Bb_{:,k} \mid - 
&\sim \mathcal{N}_{E+1}\!\left(\mub_{\mathrm{PG}}, \Sigmab_{\mathrm{PG}}\right), \\
\Sigmab_{\mathrm{PG}} 
&= \left( \Ab^\mathsf{T} \Omegab \Ab + \mathbf{\Lambda}_0 \right)^{-1}, \\
\mub_{\mathrm{PG}} 
&= \Sigmab_{\mathrm{PG}}  \left(
\Ab^\mathsf{T} (\bm{\kappa} + \Omegab \bm{\zeta}) + \mathbf{\Lambda}_0 \mub_0
\right),
\end{align*}
where $\bm{\zeta} = \begin{bmatrix}
    \zeta^1&\dots&\zeta^E
\end{bmatrix}^\mathsf{T}$, and $(\mub_0,\mathbf{\Lambda}_0)$ denote the prior mean and precision induced by $\Sigmab_\Theta,\Sigmab_\Delta$ and conditioning on $\Bb_{:,\lnot k}$:
\[
\mathbf{\Lambda}_0 = \operatorname{diag}\!\left(
(\Sigmab_\Theta^{-1})_{kk},
(\Sigmab_\Delta^{-1})_{kk},
\dots,
(\Sigmab_\Delta^{-1})_{kk}
\right),
\qquad
\mub_0 =
\begin{bmatrix}
\mu_{0,\Theta} &
\mu_{0,\Delta}^1 &
\hdots &
\mu_{0,\Delta}^E
\end{bmatrix}^\top,
\]
where
\[
\mu_{0,\Theta}
=
-\frac{
(\Sigmab_\Theta^{-1})_{k,\lnot k}\,
\thetab_{\lnot k}
}{
(\Sigmab_\Theta^{-1})_{kk}
},
\qquad
\mu_{0,\Delta}^e
=
-\frac{
(\Sigmab_\Delta^{-1})_{k,\lnot k}\,
\deltab^e_{\lnot k}
}{
(\Sigmab_\Delta^{-1})_{kk}
}.
\]

The update for $\theta_k$ uses exactly the Gaussian conditional above. 
For $\deltab^e_k$, there is one additional truncation step. Since the coordinates $\{\deltab^e\}_{e \in \Ecal}$ are conditionally independent given $\thetab$ and the auxiliary variables, each environment $e$ can be updated separately after sampling $\theta_k$. To do this, we partition $\mub_{\mathrm{PG}}$ and $\Sigmab_{\mathrm{PG}}$ as
\[
\mub_{\mathrm{PG}}=
\begin{bmatrix}
\mu_\theta \\
\mub_\Delta
\end{bmatrix},
\qquad
\Sigmab_{\mathrm{PG}}=
\begin{bmatrix}
\Sigma_{\theta\theta} & \Sigmab_{\theta\Delta} \\
\Sigmab_{\Delta\theta} & \Sigmab_{\Delta\Delta}
\end{bmatrix},
\]
where $\mu_\theta \in \mathbb R$, $\mub_\Delta \in \mathbb R^E$, $\Sigma_{\theta\theta}\in\mathbb R$, $\Sigmab_{\theta\Delta}\in\mathbb R^{1\times E}$, $\Sigmab_{\Delta\theta}\in\mathbb R^{E\times 1}$, and $\Sigmab_{\Delta\Delta}\in\mathbb R^{E\times E}$.

We first sample
\[
\theta_k \sim \mathcal N(\mu_\theta,\Sigma_{\theta\theta}).
\]
Conditioning on this sampled value, the remaining coordinates satisfy
\[
p(\delta^e_k\mid \theta_k,-)
\propto
\mathcal N\!\left((\mu_{\mathrm{cond}})_e,(\Sigma_{\mathrm{cond}})_{ee}\right)
\mathbbm{1}\!\left\{
\delta^e_k\in(-\infty,l^e_k] \cup [u^e_k,\infty)
\right\},
\]
where
\begin{align*}
\mub_{\mathrm{cond}}
&=
\mub_\Delta
+
\Sigmab_{\Delta\theta}\Sigma_{\theta\theta}^{-1}
\bigl(\theta_k-\mu_\theta\bigr), \\
\Sigmab_{\mathrm{cond}}
&=
\Sigmab_{\Delta\Delta}
-
\Sigmab_{\Delta\theta}\Sigma_{\theta\theta}^{-1}\Sigmab_{\theta\Delta}.
\end{align*}
The truncation bounds $(l^e_k,u^e_k)$ are obtained as the roots of the quadratic constraint with respect to $\delta^e_k$

\[
\frac{1}{2}H_{kk}(\delta^e_k)^2
+
\delta^e_k \,\Hb_{k,\lnot k}\deltab^e_{\lnot k}
+
\frac{1}{2}(\deltab^e_{\lnot k})^\top
\Hb_{\lnot k,\lnot k}
\deltab^e_{\lnot k}
= \varepsilon ,
\]
where 
\[
\deltab^e
=
\begin{pmatrix}
\delta^e_k \\
\deltab^e_{\lnot k}
\end{pmatrix},
\qquad
\Hb =
\begin{pmatrix}
H_{kk} & \Hb_{k,\lnot k} \\
\Hb_{\lnot k,k} & \Hb_{\lnot k,\lnot k}
\end{pmatrix}.
\]

\subsection{Stratified resampling}
\label{app:inference_resampling}
For the resampling step of SMCS, we use \emph{stratified resampling} \citep{Kitagawa1996Filter}.\footnote{The simplest resampling method is \emph{multinomial resampling}, which independently re-draws $S$ particles with replacement from the current set, where each particle is selected with probability equal to its importance weight. This can be done using i.i.d.\ draws $u_s \sim \mathrm{Uniform}(0,1)$ with the same inverse-CDF construction underlying stratified resampling. This is a valid approach but typically results in higher variance.} Stratified resampling re-draws $S$ particles with replacement from the current set, where the expected number of times each particle is selected is proportional to its importance weight. Specifically, let cumulative weights be
\[
C_s = \sum_{s=1}^S w_k^{(s)},
\]
and draw
\[
u_s \sim \mathrm{Uniform}\left(\frac{s-1}{S}, \frac{s}{S}\right), \quad s=1,\ldots,S.
\]
The ancestor indices are then defined as
\[
a_s = \min \{ j : C_j \geq u_s \}.
\]
Finally, particles are replaced by their selected ancestors, and all weights are reset:
\[
(\Phib,\Psib,\Zb)_k^{(s)} \leftarrow (\Phib,\Psib,\Zb)_k^{(a_s)}, \qquad w_k^{(s)} \leftarrow \frac{1}{S}.
\]
\subsection{Parallel Tempering}
Several methods have been developed to improve Monte Carlo exploration of multimodal posteriors, including mode-jumping strategies~\citep{pompe2020framework} and tempering-based approaches~\citep{geyer1991markov,marinari1992simulated,neal1996sampling}. 
We initially considered parallel tempering \citep{geyer1991markov}, an MCMC-based approach that leverages the same family of tempered posteriors introduced in Section~\ref{sec:SMCS}, but found SMCS performed better in our implementation and model.

Parallel tempering runs $K$ Markov chains in parallel at different temperatures $1 = T_1 < T_2 < \cdots < T_K$, where each chain targets the tempered distribution $p_{T_k}(\Phib,\Psib,\Zb \mid \Xb)$. As in SMCS, higher temperatures correspond to flattened posteriors that facilitate exploration across modes, while the lowest-temperature chain ($T_1 = 1$) targets the true posterior of interest.

 While SMCS maintains a population of particles and progresses sequentially through the temperature schedule, parallel tempering instead maintains chains at fixed temperatures and evolves them sequentially to generate samples. Information is exchanged across temperatures via metropolis swap moves between adjacent chains. Specifically, after every $m$ Gibbs steps, swap proposals are made between adjacent temperature levels $T_k$ and $T_{k+1}$. A proposed swap of states $(\text{State}_k, \text{State}_{k+1})$ is accepted with probability
\begin{equation}
\begin{aligned}
    p_{\mathrm{swap}}
    &=
    \min\!\left\{
        1,\,
        \frac{
            p_{T_k}(\text{State}_{k+1})
            p_{T_{k+1}}(\text{State}_k)
        }{
            p_{T_k}(\text{State}_k)
            p_{T_{k+1}}(\text{State}_{k+1})
        }
    \right\}
    \\
    &=
    \min\!\left\{
        1,\,
            \big[
                p(\Zb_{k+1} \mid \Phib_{k+1})\,
                p(\Xb \mid \Zb_{k+1}, \Psib_{k+1})
            \big]^{1/T_k-1/T_{k+1}}
            \big[
                p(\Zb_k \mid \Phib_k)\,
                p(\Xb \mid \Zb_k, \Psib_k)
            \big]^{1/T_{k+1}-1/T_k}
    \right\}.
\end{aligned}
\label{eq:swap-acceptance}
\end{equation}
This mechanism allows high-temperature chains to escape local modes and transfer information back to lower temperatures, thereby improving overall posterior exploration. The intermediate temperatures serve as bridges between the cold and hot chains, and greatly improve acceptance rates.

\begin{algorithm}[tbp]
\caption{Parallel Tempering}
\label{alg:parallel_tempering}
\begin{algorithmic}[1]
\Require Temperatures $1 = T_1 < T_2 < \cdots < T_K$, iterations $S$, swap interval $m$
\State Initialize states $\{(\Phib,\Psib,\Zb)^{(0)}_k\}_{k\in[K]}$

\For{$s = 1$ {\bf to} $S$}
    \For{$k = 1$ {\bf to} $K$} 
        \State $(\Phib,\Psib,\Zb)^{(s)}_k \leftarrow \textsc{GibbsUpdate}((\Phib,\Psib,\Zb)^{(s-1)}_k, T_k)$ 
    \EndFor

    \If{$s \bmod m = 0$}
        \For{$k = 1$ {\bf to} $K-1$}
            \State Swap $((\Phib,\Psib,\Zb)^{(s)}_k, (\Phib,\Psib,\Zb)^{(s)}_{k+1})$ with probability $p_{\mathrm{swap}}$
        \EndFor
    \EndIf
\EndFor

\State {\bf return} Samples from cold chain $\{(\Phib,\Psib,\Zb)_1^{(s)}\}_{s=1}^S$ 
\end{algorithmic}
\end{algorithm}

\subsection{Variational Inference}
\label{app:VI}

As an alternative inference method, we also considered an approach based on variational inference~\citep{blei2017variational,wainwright2008graphical}.
Rather than maintaining a population of particles or Markov chains, VI approximates the exact posterior by optimizing a tractable surrogate distribution over the continuous latent variables with respect to the evidence lower bound (ELBO).
In our implementation,
the discrete causal variables $\Zb^{\Ecal}$ are \emph{not} approximated variationally; instead, they are marginalized exactly by enumerating all $K^L$ joint configurations, which avoids high-variance score-function gradient estimates for the discrete variables.

\paragraph{Variational family.}
The approximate posterior is defined over the continuous latent variables $\Wb := (\Mb,\,\Thetab,\,\bm{c},\,\Ical^{\Ecal},\,\Deltab^{\Ecal})$ and factorizes according to a mean-field assumption as
\begin{equation}
\label{eq:variational_family}
    q_{\phib}\!\left(\Wb\right)
    =
    q_{\phib}\!\left(\Mb\right)
    \prod_{\ell=1}^{L}
    \Bigg[
    q_{\phib}\!\left(c_\ell\right)
    \prod_{\pa_\ell\in\PAcal_\ell}
    q_{\phib}\!\left(\bm\theta_\ell(\pa_\ell)\right)
    \prod_{e=1}^{E}
    q_{\phib}\!\left(\Ical_\ell^e\right)\,
    q_{\phib}\!\left(\bm\delta^e_\ell(\pa_\ell)\right)
    \Bigg].
\end{equation}
The individual factors are specified as follows.

\begin{itemize}

\item \emph{Measurement matrix.}
The matrix $\Mb\in\RR^{D\times(L+1)}$ is treated as a point estimate rather than a full variational distribution, i.e.,
\begin{equation}
    q_{\phib}\!\left(\Mb\right) = \text{Dirac-Delta}\!\left(\Mb_{\mathrm{loc}}\right),
\end{equation}
where 
$\Mb_{\mathrm{loc}}\in\RR^{D\times(L+1)}$ is a learnable parameter.
This corresponds to maximum a posteriori (MAP) estimation of $\Mb$ within the variational framework.

\item \emph{Base mechanism parameters.}
Each parameter vector $\bm\theta_\ell(\pa_\ell)$ is approximated by an independent multivariate Gaussian with learnable mean and %
covariance:
\begin{equation}
    q_{\phib}\!\left(\bm\theta_\ell(\pa_\ell)\right)
    = \Ncal_K\Big(\bm\mu^\theta_\ell(\pa_\ell),\;
    \Lb^\theta_\ell(\pa_\ell)\,{\Lb^\theta_\ell(\pa_\ell)}^\top\Big),
\end{equation}
where $\bm\mu^\theta_\ell(\pa_\ell)\in\RR^K$ is the variational mean and $\Lb^\theta_\ell(\pa_\ell)$ is a learnable lower-triangular matrix with positive diagonal.

\item \emph{Intervention probability.}
The scalar $c_\ell$ is approximated by a Beta distribution with learnable concentration parameters:
\begin{equation}
    q_{\phib}(c_\ell)
    = \mathrm{Beta}\Big(\mathrm{softplus}(\alpha_\ell),\;\mathrm{softplus}(\beta_\ell)\Big),
\end{equation}
where $\alpha_\ell,\beta_\ell\in\RR$ are passed through the softplus function to ensure positivity.

\item \emph{Intervention indicators.}
Since $\Ical^e_\ell\in\{0,1\}$ is discrete, we replace the Bernoulli factor with a Relaxed Bernoulli~\citep{maddison2016concrete} at a fixed temperature $\tau>0$:
\begin{equation}
    q_{\phib}\!\left(\Ical^e_\ell\right)
    = \mathrm{RelaxedBernoulli} \Big(\tau,\;\sigma\!\left(\xi^e_\ell\right)\Big),
\end{equation}
where $\xi^e_\ell\in\RR$ is a learnable logit and $\sigma(\cdot)$ denotes the sigmoid function.
This continuous relaxation permits reparameterization gradients while approaching binary behaviour for small~$\tau$.
To recover binary predictions at evaluation time, the straight-through estimator~\citep{bengio2013estimating} is applied: the forward pass uses the value of the relaxed sample rounded to $\{0,1\}$ by thresholding at $0.5$, while gradients are passed through the continuous relaxation.

\item \emph{Intervention shifts.}
Each shift vector $\bm\delta^e_\ell(\pa_\ell)$ is approximated by a multivariate Gaussian:
\begin{equation}
    q_{\phib}\!\left(\bm\delta^e_\ell(\pa_\ell)\right)
    = \Ncal_K\Big(\bm\mu^\delta_\ell(e,\pa_\ell),\;
    \Lb^\delta_\ell(e,\pa_\ell)\,{\Lb^\delta_\ell(e,\pa_\ell)}^\top\Big).
\end{equation}
To enforce the orthogonality constraint from~\eqref{eq:covariance_delta} (i.e., $\sum_k \delta_k = 0$), each raw sample $\bm\delta_{\mathrm{raw}}$ is projected onto the sum-zero subspace after sampling:
\begin{equation}
\label{eq:vi_delta_projection}
    \bm\delta = \bm\delta_{\mathrm{raw}} - \frac{1}{K}\!\left(\bm1_K^\top \bm\delta_{\mathrm{raw}}\right)\bm1_K.
\end{equation}
Note that, unlike in the SMCS and Gibbs samplers, the KL-truncation of the shift prior (\cref{sec:shift_priors}) is not applied in this VI implementation; a plain isotropic Gaussian $\Ncal_K(\bm0,\,\sigma^2_\Delta\,\Ib)$ is used as the prior on the raw shifts.

\end{itemize}

The observation noise $\bm\sigma^2$ is treated as a fixed hyperparameter and is not estimated variationally.
The variational parameters $\phib$ thus comprise
\begin{equation}
\label{eq:variational_parameters}
    \phib = \Big(\Mb_{\mathrm{loc}},\;
    \big(\bm\mu^\theta_\ell(\pa_\ell),\,\Lb^\theta_\ell(\pa_\ell)\big)_{\ell,\pa_\ell},\;
    (\alpha_\ell,\beta_\ell)_{\ell},\;
    \big(\xi^e_\ell\big)_{\ell,e},\;
    \big(\bm\mu^\delta_\ell(e,\pa_\ell),\,\Lb^\delta_\ell(e,\pa_\ell)\big)_{\ell,e,\pa_\ell}
    \Big).
\end{equation}

\paragraph{ELBO and exact marginalization of the discrete variables.}
Collecting all continuous latents in $\Wb$, the ELBO is
\begin{align}
\label{eq:vi_elbo}
    \Lcal(\phib)
    &:= \EE_{q_{\phib}(\Wb)}\!\left[
        \log \sum_{\zb\in\Zcal^L}
        p\!\left(\Wb,\Zb=\zb,\Xb^\Ecal\middle|\bm\sigma^2\right)
        - \log q_{\phib}(\Wb)
    \right]
    \;\leq\;
    \log p\!\left(\Xb^\Ecal\mid\bm\sigma^2\right).
\end{align}
Because no variational factor is placed on $\Zb$, the discrete latents are summed out exactly.
For each assignment $\zb\in\Zcal^L$, the joint log-likelihood is computed analytically, and the sum over all $K^L$ configurations is folded into the ELBO via the log-sum-exp operation.
This eliminates the need for a score-function estimator for the discrete variables entirely.
In practice, this exact enumeration is carried out by Pyro's \texttt{TraceEnum\_ELBO} objective with parallel enumeration of latent states.

\paragraph{Optimization.}
All variational parameters $\phib$ are optimized jointly using the Adam optimizer~\citep{kingma2015adam} with a parameter-specific learning rate schedule.
The measurement matrix $\Mb_{\mathrm{loc}}$ and the mechanism means $\bm\mu^\theta_\ell$ are updated with learning rate $\eta_\theta = 10^{-3}$, while the intervention logits $\xi^e_\ell$ and the shift means $\bm\mu^\delta_\ell$ use a higher learning rate $\eta_\delta = 10^{-2}$ to encourage faster adaptation to environment-specific effects.
The ELBO is normalized by the total number of observations before computing gradients.
Training proceeds for $2{,}000$ gradient steps with intervention temperature $\tau = 0.4$.

The measurement matrix $\Mb_{\mathrm{loc}}$ is initialized from $\Ncal(0,1)$; all remaining means are initialized from small-variance Gaussians near zero.
The intervention logits $\xi^e_\ell$ are initialized to $-3$, placing the prior probability of intervention at approximately $5\%$ and encouraging sparsity at the start of training.

\paragraph{Soft KL penalty.}
To partially compensate for the absence of the hard KL truncation, we also evaluated a variant that adds a differentiable Lagrangian penalty to the ELBO,
\begin{equation}
\label{eq:vi_soft_penalty}
    \log f_\lambda(\bm\delta^e_\ell(\pa_\ell))
    = -\lambda \cdot \mathrm{ReLU}\!\left(b - \frac{1}{2K}\left\|\bm\delta^e_\ell(\pa_\ell)\right\|^2\right),
\end{equation}
with threshold $b=0.1$ (matching the paper's default) and penalty weight $\lambda=1$.
This penalizes configurations where the shift is below the KL threshold while remaining fully differentiable.
However, the penalty term is negligible in scale relative to the ELBO (which is on the order of $10^6$ per dataset), so the optimizer effectively ignores it. Empirically, intervention recovery accuracy and LPPD were indistinguishable from the unpenalized VI baseline across all 10 datasets. This confirms that a soft Lagrangian relaxation alone is insufficient to replicate the effect of the hard KL truncation, and thus the adaptations used to allow for gradient-based optimization create a performance gap between VI and sampling based approaches, beyond any differences that may arise from the inference methods themselves.\looseness-1

\section{Experiment Processing and Metrics}
\label{app:metrics}
This appendix provides details on processing steps and exact metrics computed for each figure.
\subsection{Label Switching}
The model permits a nuisance symmetry of sign flips in $\Zb$, which we have to account for in our analysis. Specifically, for any $\sbb \in \{-1,1\}^L$, the transformation
\[
\zb' = \sbb \odot \zb , \quad  \Mb' = \begin{bmatrix}
    \mb_0 & s_1 \mb_1 &\hdots & s_L \mb_L 
\end{bmatrix},
\]
leaves the posterior over all other quantities unchanged and is fundamentally the same representation. Unlike permutations of the latent ordering, this is a true nuisance symmetry and leaves the full posterior invariant, not just the measurement model. To manage this, we align the signs of $\zb$ individually for each sample using the coordinate of $\Mb$ with the largest signal,
\[
s_\ell = \mathrm{sign}\left(M_{ \arg\max_{d\in [D]} |M_{d,\ell}|,\ell}\right), \qquad
z_\ell' = s_\ell z_\ell.
\]

When comparing to ground truth $\Mb^*$, we perform the additional alignment
\[
s_\ell =
\mathrm{sign}\left(M^*_{ \arg\max_{d\in [D]} |M_{d,\ell}|,\ell}\right)\cdot
\mathrm{sign}\left(M_{ \arg\max_{d\in [D]} |M_{d,\ell}|,\ell}\right),
\qquad
z_\ell' = s_\ell z_\ell.
\]
This relabeling is standard practice~\citep{stephens2000labels}, and we use aligned $\zb'$ instead of raw $\zb$ for all metrics and graphs.

        \subsection{Full Intervention Figure}
        \label{app:metrics_intervention_fig}
        \Cref{fig:wvs_Delta_all} contains the full parent grid referenced in \cref{fig:wvs_Delta_line}, as well as similar plots for $\delta_1,\delta_2$
\begin{figure}[tbp]
\centering
\includegraphics[scale = 0.25]{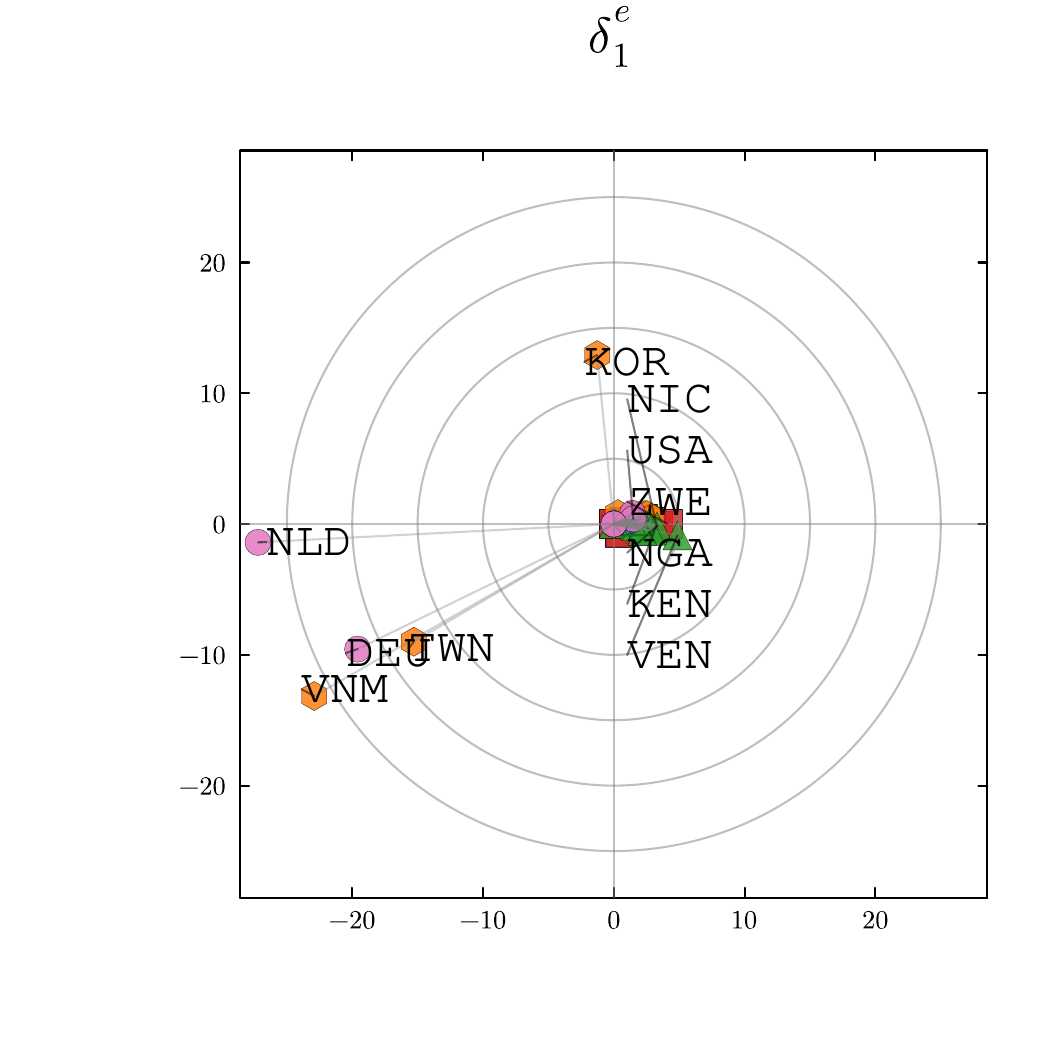}
\includegraphics[scale = 0.2]{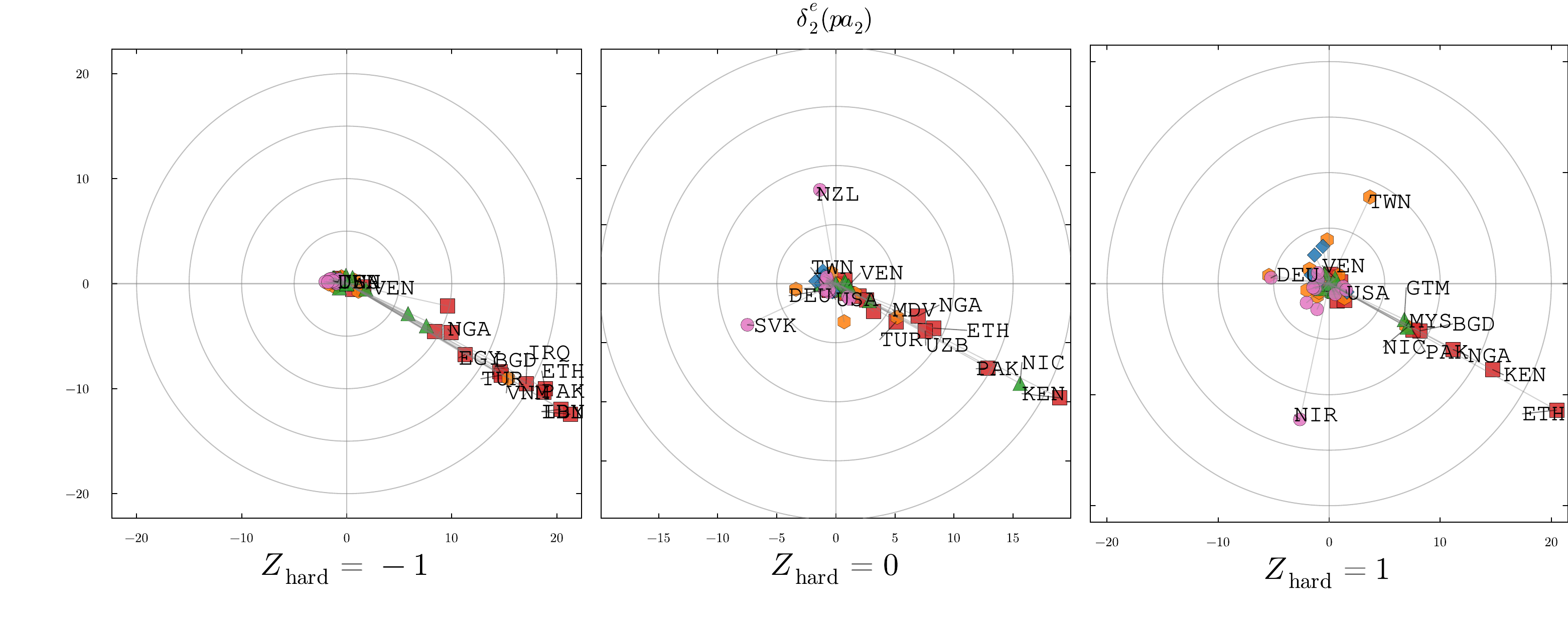}
\includegraphics[scale = 0.25]{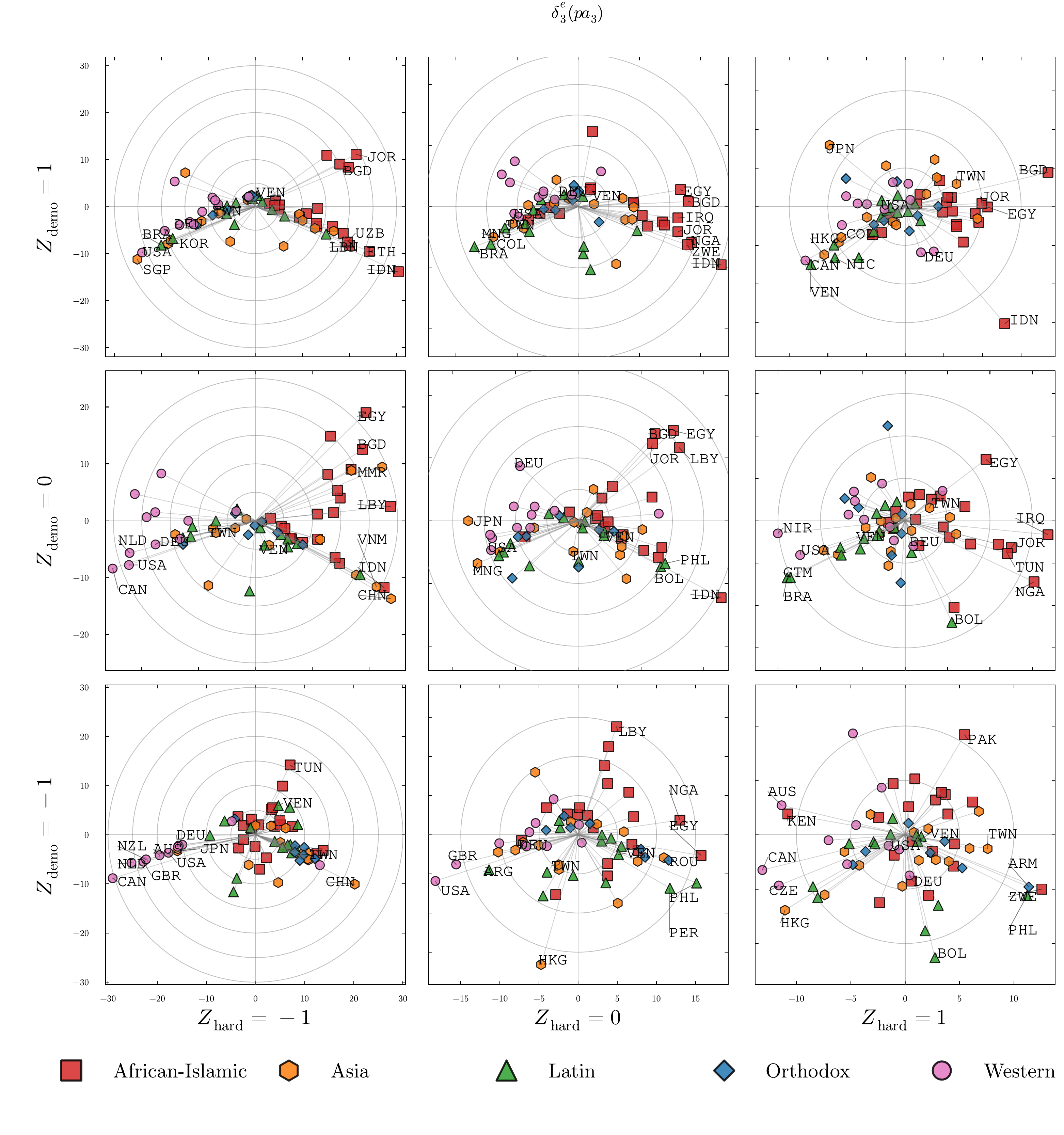}
\caption{
\textbf{Intervention directions across parent configurations and environments.}
Posterior mean intervention vectors $\delta_\ell^e(\pa_\ell)$ for each country, shown across all parent configurations. Vectors are expressed in an orthogonal basis for the two-dimensional intervention subspace, where the first basis direction captures shifts in the mean of $Z_{\ell}$ and the second captures changes in dispersion.\looseness-1
}
\label{fig:wvs_Delta_all}
\end{figure}

\subsection{Metrics} 
This reports the exact sample statics used in figures and tables throughout the experimental results sections. 
Throughout,  $\{(\Phib, \Psib,\Zb)^{(s)}\}_{s=1}^S$ is the set of posterior samples with weights. Weights are uniformly $1/S$ for Gibbs and PT.

\subsubsection{Measurement Model}
We report the posterior mean of the magnitude of the measurement matrix. For the WVS dataset, we additionally normalize each variable by its range of observed values. 
This normalization is necessary because \textsc{Q262 Age} has a substantially larger range than other variables; it would otherwise dominate the scale, obscuring patterns in the remaining variables.

\begin{align*}
\text{normalized }|\Mb|_{d,\ell }
&:= 
\sum_{s=1}^S w^{(s)} \cdot  \frac{|\Mb^{(s)} _{d,l}|}{\max_{e,j}x^{e,j}_d - \min_{e,j}x^{e,j}_d }
&&\text{\cref{fig:wvs_M}}
\\
|\Mb|_{d,l}
&:= 
\sum_{s=1}^S w^{(s)} \cdot  |\Mb^{(s)} _{d,l}|
&&
\text{\cref{fig:pollfish_ER,fig:pollfish_TR,fig:llm_M}}
\\
\Mb_{d,l}
&:= 
\sum_{s=1}^S w^{(s)} \cdot  \Mb^{(s)} _{d,l}
&&
\text{\cref{tab:inference_comparison}}
\\
\Mb_{\pi^*} &:= \sum_{s=1}^S  w^{(s)} \cdot \Mb^{(s)}_{\pi^{*(s)}} 
&&\text{\cref{tab:inference_comparison}}
\\
\pi^{*(s)} &:= \arg\min_\pi \|{\Mb^{(s)}_\pi} - \Mb^*\|_{1,1}
\end{align*}

\subsubsection{Causal Latent Variables}
We report the marginal posterior mean of each causal latent variable. We also define several distributions over $\Zb$ that are used by multiple metrics below.
\begin{align*}
\EE_{p^e_\Zb}[Z_\ell] 
&:= 
\sum_{s=1}^S w^{(s)} \cdot \left ( 
\sum_{z_\ell \in \Zcal_\ell } z_\ell \cdot p^{e,(s)}(z_\ell)\right)
&&\text{\cref{fig:wvs_Z,fig:pollfish_ER,fig:pollfish_TR}}
\\
p^{e,(s)}(z_\ell|\pa_\ell)
&=
\softmax\Big(\bm\theta^{(s)}_{\ell}\left(\pa^{(s)}_\ell\right) + \Ical^{e,(s)}_\ell  \bm\delta^{e,(s)}_\ell\left(\pa^{(s)}_\ell\right   ) \Big)_{z_\ell }
&&\text{Conditional}
\\
p^{e,(s)}(\zb)
&=
\prod_{\ell=1}^L
p^{e,(s)}_{Z_\ell}(z_\ell|\zb_{1:\ell-1})
&&\text{Joint}
\\
p^{e,(s)}(z_\ell^*)
&=
\sum_{\zb \in \Zcal } \mathbbm{1}(z_\ell = z_\ell^*) \cdot p^{e,(s)}(\zb)
&&\text{Marginal}
\\
\Zcal &= \prod_{\ell = 1}^L \Zcal_\ell \qquad \Zcal_\ell = \left\{1-\frac{K+1}{2},\, 2-\frac{K+1}{2},\, \dots,\, K-\frac{K+1}{2}\right\}.
\end{align*}

\subsubsection{Causal Graph}
We quantify causal dependence using the method suggested by \citet{janzing2013quantifying}.
\begin{align*}
 \text{Causal Influence}
&:= 
\sum_{s=1}^S w^{(s)} \cdot 
D_\textsc{kl}\big(p^{e,(s)}(\Zb)~\big\|~ p^{j\to \ell,e,(s)}(\Zb)\big)
&&\text{\cref{fig:wvs_Z,fig:wvs_CI-E,fig:llm_CI}}
\\
p^{j\to \ell,e,(s)}(\zb)
&= 
\left(
\sum_{z_j^* \in \Zcal_j}
p^{e,(s)}(z_j^*)
\cdot 
p^{e,(s)}(z_\ell|
\zb_{1:j-1},
z^*_{j},
\zb_{j+1:\ell-1})
\right)
\cdot 
\prod_{\ell'\in [L]\setminus\ell}
p^{e,(s)}(z_{\ell'}|\zb_{1:\ell'-1})
\end{align*}
\subsubsection{Interventions}
We report posterior summaries of intervention indicators and shifts.
\begin{align*}
\Ical
&:= 
\sum_{s=1}^S w^{(s)} \cdot  \Ical^{(s)}
&&
\parbox{15em}{
\cref{tab:inference_comparison} \\
\cref{fig:wvs_ID,fig:pollfish_ER,fig:pollfish_TR}}
\\
\EE[\|\Delta\|_{KL}]^{e}_\ell
&:= 
\sum_{s=1}^S w^{(s)} 
\cdot  
\left( 
\sum_{\zb \in \Zcal }D_\textsc{kl}\left(
p^{0,(s)}(Z_\ell \cdot|\zb_{1:\ell-1})
\|
p^{e,(s)}(Z_\ell|\zb_{1:\ell-1})
\right) \cdot p^{e,(s)}(\zb)
\right)
&&\text{\cref{fig:wvs_ID}}
\\
\delta_{\ell}^e(\pa_\ell)
&:= 
\sum_{s=1}^S w^{(s)} \Ical^{e,(s)}_\ell\delta_{\ell}^{e,(s)}\left(\pa^{(s)}_\ell\right)
&&\text{\cref{fig:wvs_Delta_line}}
  \end{align*}
\subsubsection{Miscellaneous}
\begin{align*}
\EE_{e} \left[X_d \right]
&:=
\frac{1}{N^{e}} \sum_{j \in [N^e]}  \mkern-5mu x^{e,j}_d 
&&\text{\Cref{eq:ate}}
\\
\mathrm{LPPD}
&:=
\frac{1}{\sum_{e\in\Ecal}N^e}
\sum_{e \in \mathcal{E}}
\sum_{j \in [N_e]}
\log \!\left(
\sum_{s=1}^S
w^{(s)}\,
\sum_{\zb\in\Zcal}
p^{e,(s)}(\zb)
\phi_D\left(x^{e,j} ; \mb_0^{(s)} +\textstyle\sum_{\ell \in [L]} \mb_\ell^{(s)} z_\ell,
\mathrm{diag}\big(\sigmab^2\big)^{(s)}\right)
\right)
  &&\text{\cref{tab:inference_comparison}}                                                                   
        \end{align*}

\endgroup

{
\setlength{\bibsep}{4pt plus 2pt minus 2pt}
\small
\bibliography{ref}
}
\end{document}